\documentclass[10pt,journal,compsoc]{IEEEtran}
\usepackage{amssymb}
\usepackage{amsmath}
\usepackage{amsfonts}
\usepackage{amsmath}
\usepackage{mathabx}
\usepackage{amssymb}
\usepackage{amsopn}
\usepackage{bm}
\usepackage{relsize}
\usepackage{exscale}
\usepackage{booktabs}
\usepackage{graphicx}
\usepackage{epstopdf}
\usepackage{epsfig}
\usepackage{ragged2e}
\usepackage{url}
\usepackage{float}
\usepackage{stfloats}
\usepackage{subfigure}
\usepackage{algorithmic}
\usepackage{algorithm}
\usepackage{threeparttable}
\usepackage{lineno}
\usepackage{multirow}
\usepackage{amsthm}
\newtheorem{definition}{Definition}
\newtheorem{theorem}{Theorem}
\newtheorem{hypothesis}{Hypothesis}

\ifCLASSOPTIONcompsoc
  \usepackage[nocompress]{cite}
\else
  \usepackage{cite}
\fi
\ifCLASSINFOpdf
\else
\fi
\begin{document}
\title{Multi-label Causal Variable Discovery: \\ Learning Common Causal Variables and Label-specific Causal Variables}

\author{Xingyu~Wu,
        Bingbing~Jiang,
        Yan~Zhong,
        and~Huanhuan~Chen*,~\IEEEmembership{Senior Member,~IEEE}% <-this % stops a space
\IEEEcompsocitemizethanks{\IEEEcompsocthanksitem X. Wu and H. Chen (*corresponding author) are with the School of Computer Science and Technology, University of Science and Technology of China, Hefei 230027, China (e-mail: xingyuwu@mail.ustc.edu.cn; hchen@ustc.edu.cn).% <-this % stops a space
\IEEEcompsocthanksitem B. Jiang is with the School of Information Science and Engineering, Hangzhou Normal University, Hangzhou 311121, China (e-mail: jiangbb@hznu.edu.cn).% <-this % stops a space
\IEEEcompsocthanksitem Y. Zhong is with the School of Data Science, University of Science and Technology of China, Hefei 230027, China (e-mail: zhongyan@mail.ustc.edu.cn).}
}

\IEEEtitleabstractindextext{%
\begin{abstract}
Causal variables in Markov boundary (MB) have been widely applied in extensive single-label tasks. While few researches focus on the causal variable discovery in multi-label data due to the complex causal relationships. Since some variables in multi-label scenario might contain causal information about multiple labels, this paper investigates the problem of multi-label causal variable discovery as well as the distinguishing between common causal variables shared by multiple labels and label-specific causal variables associated with some single labels. Considering the multiple MBs under the non-positive joint probability distribution, we explore the relationships between common causal variables and equivalent information phenomenon, and find that the solutions are influenced by equivalent information following different mechanisms with or without existence of label causality. Analyzing these mechanisms, we provide the theoretical property of common causal variables, based on which the discovery and distinguishing algorithm is designed to identify these two types of variables. Similar to single-label problem, causal variables for multiple labels also have extensive application prospects. To demonstrate this, we apply the proposed causal mechanism to multi-label feature selection and present an interpretable algorithm, which is proved to achieve the minimal redundancy and the maximum relevance. Extensive experiments demonstrate the efficacy of these contributions.
\end{abstract}

}

\maketitle

\IEEEdisplaynontitleabstractindextext

% \fi
\IEEEpeerreviewmaketitle

\section{Introduction}

\IEEEPARstart{C}{ausal} variable set, also known as Markov boundary (MB), contains critical causal information about a given target. As shown in Figure \ref{pic_example_mb}, MB consists of the \emph{direct causes}, \emph{direct effects}, and \emph{other direct causes of the direct effects} of the target \cite{aliferis2010locala,aliferis2010localb,pearl1988}. Causal variables have potential ability to imply the underlying causal mechanism around the target \cite{aliferis2010locala,aliferis2010localb}, and are widely applied to real-world tasks. For example, causal variable discovery is the first step in causal learning and Bayesian network (BN) structure learning, where the skeleton of the Bayesian network without orientation is constructed by MB \cite{gao2015local,pellet2009finding,ram2009markov}. Another important application is feature selection \cite{guyon2007causal,cai2011bassum}, since all other features are independent of the class attribute conditioned on its MB \cite{guyon2007causal}. Some studies \cite{tsamardinos2003towards,pellet2008using,masegosa2012bayesian,statnikov2013algorithms} have proved that causal variable set is the theoretically optimal subset for learning and inference tasks. Due to the practical benefits, extensive algorithms are proposed to search MB in single-label data. Some of these algorithms \cite{iamb,mmmb,hiton,pcmb,stmb,ccmb,wu2020tolerant} learn the MB set based on Unique MB Hypothesis\footnote{A basic assumption of MB discovery problem, supposing that each target has a unique MB set (Refer to Theorem \ref{unique_mb} in \textbf{Section \ref{relatedwork}} for details).} \cite{aliferis2010locala}, which is always violated in real-world data. Other algorithms \cite{pcmb,statnikov2013algorithms,liu2016swamping} relax the hypothesis to detect multiple MBs of a target, whereas it is still not tractable to find all of the possible MBs due to the unpredictable number of MBs.

\begin{figure}[t]
  \centering
  \includegraphics[height=1.60in, width=3.20in]{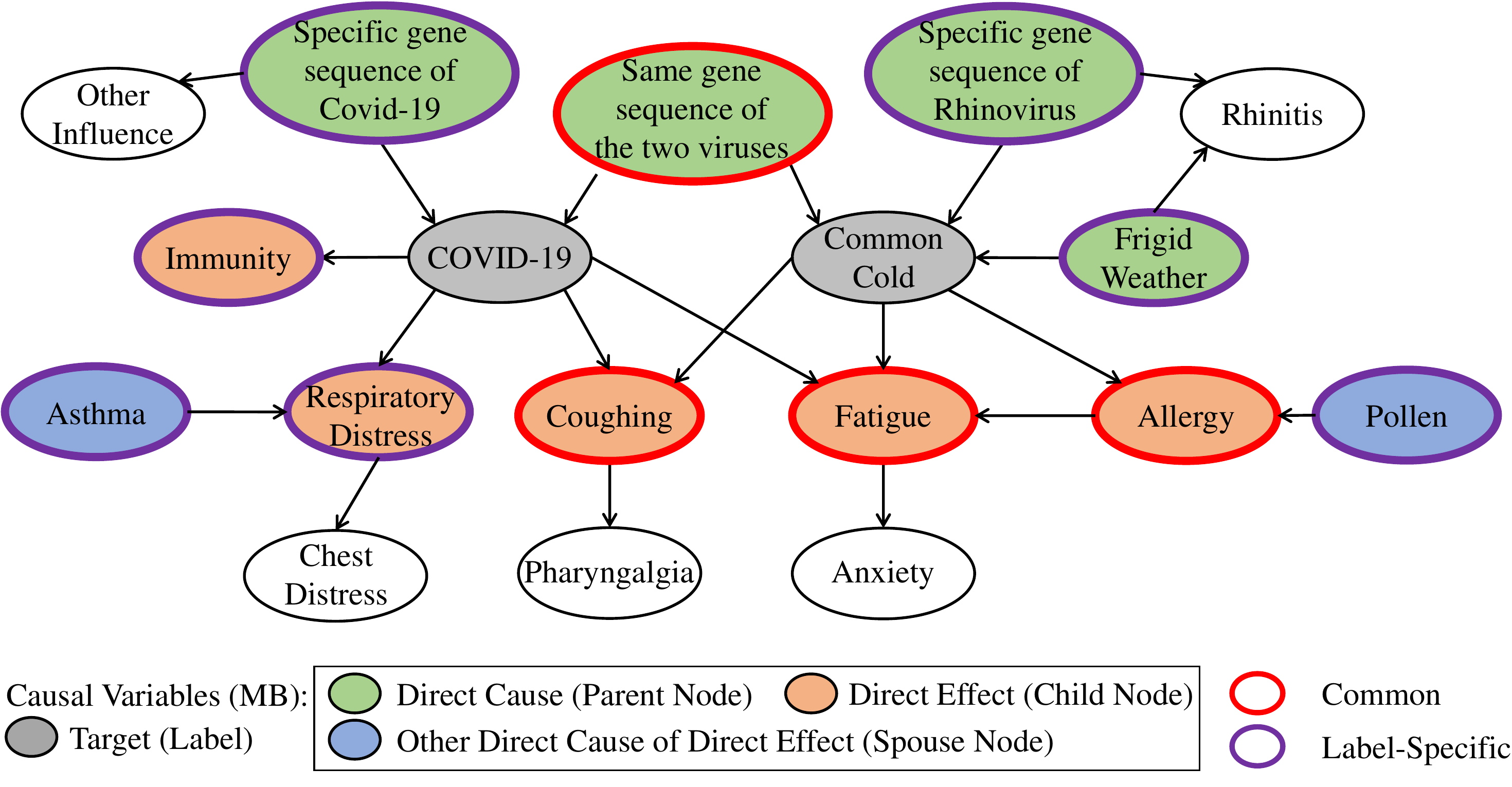}
  \caption{Examples to illustrate the basic concepts in this paper. (1) Markov boundary (\textbf{MB}) of a target (e.g. Common Cold) contains \textbf{direct causes} (Frigid Weather, Specific gene sequence of Rhinovirus, and Same gene sequence of the two viruses), \textbf{direct effects} (Coughing, Fatigue, and Allergy) and \textbf{other direct causes} (Pollen) \textbf{of the direct effects} (Allergy) of the target. (2) For multiple labels, \textbf{common} causal variables simultaneously influence multiple labels while \textbf{label-specific} causal variables influence a single label. For example, Coughing is the common causal variable of COVID-19 and Common Cold, while Respiratory Distress is the label-specific causal variable of COVID-19.}\label{pic_example_mb}
\end{figure}

However, few researches consider the multi-label causal variable discovery despite the ubiquity of multi-label data. Causal variable discovery occurs when the joint probability distribution of several labels conditioned on other features is analyzed, such as common causes and effects discovery of multiple targets, dimensionality reduction for multi-label learning and inference, and etc. The core of these tasks is to identify causal variables of multiple labels, which will be discussed in this paper. Contrary to single-label data, multi-label data involve extra relationships between labels, leading to two types of causal variables. As shown in Figure \ref{pic_example_mb}, some causal variables simultaneously contain the causal information about several labels, which are called \textbf{common causal variables} in the following, and correspondingly, others in the MB set are called \textbf{label-specific causal variables}. Both of them can facilitate the comprehension of the underlying causal mechanism, yet their focuses are different. Intuitively, label-specific causal variables reflect the differences among the local causal structures around different labels, which would assist the prediction or inference tasks on their corresponding labels \cite{huang2019specific}. While the common causal variables represent the causal connections between these labels, which are naturally capable of providing the information about multiple labels with minimal number of variables. Hence, they have extensive application prospects in dimension reduction, such as multi-label feature selection \cite{lin2015multi}. In order to drive a biased model\footnote{A biased model means that different causal variables in the model have different effects on it. For example, in a prediction model for a certain label, label-specific variables of this label have a greater impact on predicting results.} where different types of causal variables can contribute to the model with varying degrees, it is necessary to study the multi-label causal variable discovery problem with the following goals:
\begin{itemize}
\item To discover all of the causal variables for a label set;
\item To distinguish between common causal variables and label-specific causal variables.
\end{itemize}
%Based on these goals, the study of multi-label causal variable discovery could facilitate both causal (multi-target causal discovery, especially the common causes and effects discovery) and non-causal (multi-label feature selection and biased multi-label learning model) tasks.

Due to more complex causal relationships in multi-label data, the existing single-label methods cannot be applied to multi-label data directly.
%It seems that the causal variables around multiple labels will emerge spontaneously when the MB of each label is obtained, then the distinguishing process can be transformed to a set of single-label problems, i.e., to find the intersection of MBs of different labels. However, \emph{nothing is ever what it appears to be} since
Different from single-label scenario, Unique MB Hypothesis should be relaxed in multi-label causal variable discovery since the multiple MBs lead to uncertain solutions. For example, in meteorology, \emph{the dropping in sea level pressure} ($F_1$), \emph{the convergence of the winds near the surface} ($F_2$), and \emph{the divergence of the winds at the top of the atmosphere} ($F_3$), are spatially adjacent \cite{Yu2015Tornado}, and just any one of them can be equivalently used to predict the \emph{tornado} ($T_1$). And only $F_1$ need be used to predict the \emph{extreme precipitations} ($T_2$). Thus, there exist three equivalent MBs of $T_1$ including $F_1$ or $F_2$ or $F_3$ respectively, making $F_1$ become a common causal variable of $T_1$ and $T_2$. However, as mentioned before, it is scarcely possible for existing methods to mine all MBs due to the low statistical reliability \cite{statnikov2013algorithms}, making some common causal variables undetected. In the example above, if we search common causal variables of $T_1$ and $T_2$ through the intersection of their MBs but only find the MB of $T_1$ including $F_2$ or $F_3$, then $F_1$ can not be identified. Besides this, directly finding all MBs may suffer from high time complexity and low accuracy due to the numerous conditional independence tests, especially with large conditioning sets \cite{ccmb}. Fortunately, the non-unique MBs coexist with equivalent information\footnote{A Phenomenon that two variable sets contain equivalent information about a target (Refer to Definition 3 in \textbf{Section \ref{relatedwork}} for details).} \cite{statnikov2013algorithms}, which are easier and more efficient to detect. Furthermore, due to the label relationships in multi-label data, causality between labels should be examined.

In this paper, we discover that common causal variables are determined by equivalent information following different mechanisms with or without existence of label causal relationships. Explicitly, we prove that if any label does not contain the causal information about another label, then the equivalent information about labels might induce new common causal variables, while the equivalent information about non-label variables does not. For this purpose, we introduce a Label-causality Hypothesis to simplify the discussion, which supposes that a label is not included by the MB of another label. Based on this hypothesis, the discussion is divided into two parts, satisfying and violating the hypothesis. We start from the simple case satisfying the hypothesis, and provide the general characteristics of common causal variables. Afterwards, we relax the hypothesis and find that some unidentified common causal variables are influenced by equivalent information about non-label variables. The above theoretical analyses are provided in \textbf{Section \ref{secccf}}, based on which we subsequently develop a common and label-specific causal variable discovery (CLCD) algorithm to achieve the following benefits:
\begin{enumerate}%itemize
\item Practicability: CLCD can search most of the causal variables and simultaneously distinguish the two types of causal variables;
\item Robustness: It is always effective in the case satisfying or violating the Label-causality Hypothesis and Unique MB Hypothesis;
\item Generality: CLCD can be directly extended to facilitate some real-world applications.%, such as causal discovery and feature selection.
\end{enumerate}

To demonstrate the generality of CLCD, we apply it to multi-label feature selection and propose a novel CLCD-FS algorithm in \textbf{Section \ref{sec_fs}}. Through learning the features containing causal information about multiple labels, CLCD-FS possesses three superiorities over traditional algorithms:
\begin{enumerate}
\item Interpretability: CLCD-FS can explain which labels a selected feature influences.
\item Practicability: Under the premise of ensuring the relatively higher accuracy, CLCD-FS automatically predetermines the number of selected features via mining the underlying causal mechanism.
\item Theoretical Reliability: We will prove that CLCD-FS achieves the maximum relevance and minimum redundancy in Section \ref{sec_fsalg2}.
\end{enumerate}
As a causality-based multi-label feature selection algorithm, it is necessary to state the main difference between CLCD and MB-MCF, another causality-based method presented in our conference paper \cite{wu2020multi}: (1) From the aspect of discussed issues: Wu \emph{et. al} \cite{wu2020multi} study \emph{\textbf{multi-label feature selection problem}} and designs the MB-MCF based on empirical knowledge without reliable theoretical guarantee. While in this paper, we discuss the \emph{\textbf{multi-label causal variable discovery problem}} and present a complete theoretical framework, which provides the theoretical guarantee for CLCD-FS. (2) From the aspect of the proposed algorithms: \emph{\textbf{CLCD-FS selects more relevant features than MB-MCF}} since it additionally considers the spouse variables, which could enhance the predictive power of direct effects \cite{guyon2007causal}. Moreover, due to the analyses of label causality, CLCD-FS better shields its negative influence on feature selection process, so that more relevant features can be identified. Based on CLCD, CLCD-FS can discover most of common causal variables, which help \emph{\textbf{CLCD-FS remove more redundant features than MB-MCF}}. Experiments in \textbf{Section \ref{exsec}} validate its superiority against traditional algorithms and MB-MCF.

To the best of our knowledge, this is the first study discussing multi-label causal variable discovery, which can not only be applied to multi-label feature selection, but also to plenty of tasks as elaborated in \textbf{Section \ref{futurework}}, such as multi-target causal discovery, especially the discovery of common causes and effects. The biased models for multi-label learning and causal inference are also worth exploring.

\section{Synopsis of Theories and Methods Motivating Present Research}
\label{relatedwork}

In this section, we introduce some basic definitions and theories motivating the present research. Additionally, some classic and state-of-the-art methods related to this research are also introduced. Conventionally, literatures in causal learning use \emph{target} to denote the variable being studied, and so we use \emph{target} when explaining the causal theory and \emph{label} as target when analyzing the multi-label data. In this paper, common upper-case letters denote random variables and upper-case bold letters denote random variable sets. Specifically, $\textbf{\emph{U}}$ represents the set of all non-label variables (or features), $T$ and $\textbf{\emph{T}}$ represent the label (target) and label set (target set) in the single-label and multi-label scenario, respectively. Hollow upper-case letter $\mathbb{G}$ denotes a directed acyclic graph (DAG), and $\mathbb{P}$ denotes the joint probability distribution over $\textbf{\emph{U}}\cup\textbf{\emph{T}}$.

\subsection{Basic Properties of Probability Distribution}

\begin{definition}
\label{conditional_independece}
(Conditional Independence) Variable sets $\textbf{X}$ and $\textbf{Y}$ are conditionally independent given a variable set $\textbf{Z}$ if $\mathbb{P}(\textbf{X},\textbf{Y}|\textbf{Z})=\mathbb{P}(\textbf{X}|\textbf{Z})\mathbb{P}(\textbf{Y}|\textbf{Z})$, denoted as $\textbf{X}\perp \textbf{Y}|\textbf{Z}$. Inversely, $\textbf{X}\notperp \textbf{Y}|\textbf{Z}$ denotes the conditional dependence relationships.
\end{definition}

Some important basic properties of joint probability distribution will be used to prove the theorems in this paper.

\begin{theorem}
\label{common_property}
\cite{pearl1988,pena2007towards} Let variable sets $\textbf{A},\textbf{B},\textbf{C},\textbf{Z}\subset\textbf{U}\cup\textbf{\emph{T}}$, six properties hold in any joint probability distribution $\mathbb{P}$ over $\textbf{U}\cup\textbf{\emph{T}}$:

(1) Self-conditioning: $\textbf{A}\perp \textbf{Z}|\textbf{Z}$.

(2) Symmetry: $\textbf{A}\perp \textbf{B}|\textbf{Z}\Leftrightarrow\textbf{B}\perp \textbf{A}|\textbf{Z}$.

(3) Decomposition: $\textbf{A}\perp \textbf{B}\cup\textbf{C}|\textbf{Z}\Rightarrow\textbf{A}\perp \textbf{B}|\textbf{Z}$ and $\textbf{A}\perp \textbf{C}|\textbf{Z}$.

(4) Weak union: $\textbf{A}\perp \textbf{B}\cup\textbf{C}|\textbf{Z}\Rightarrow\textbf{A}\perp \textbf{B}|\textbf{Z}\cup\textbf{C}$.

(5) Contraction: $\textbf{A}\perp \textbf{B}|\textbf{Z}\cup\textbf{C}$ and $\textbf{A}\perp \textbf{C}|\textbf{Z}\Rightarrow\textbf{A}\perp \textbf{B}\cup\textbf{C}|\textbf{Z}$.

(6) Intersection: If $\mathbb{P}$ is strictly positive, then: $\textbf{A}\perp \textbf{B}|\textbf{Z}\cup\textbf{C}$ and $\textbf{A}\perp \textbf{C}|\textbf{Z}\cup\textbf{B} \Rightarrow \textbf{A}\perp \textbf{B}\cup\textbf{C}|\textbf{Z}$.
\end{theorem}

Using mutual information \cite{cover2012elements} to measure the conditional independence relationship, we have $I(\textbf{\emph{X}},\textbf{\emph{Y}}|\textbf{\emph{Z}})=0$ if $\textbf{\emph{X}}\perp \textbf{\emph{Y}}|\textbf{\emph{Z}}$.

\subsection{Markov Blanket and Markov Boundary (MB)}

\begin{definition}
\label{markov_blanket}
(Markov Blanket and Markov Boundary) \cite{pearl1988} The Markov blanket $\textbf{Mb}$ of target $T$ is a subset of $\textbf{U}$ satisfying the condition: $\forall X \in \textbf{U}-\textbf{Mb}, X\perp T|\textbf{Mb}$ in the joint probability distribution $\mathbb{P}$. Markov boundary $\textbf{MB}$ of $T$ is the minimum Markov blanket of $T$ satisfying: $\forall \textbf{Z} \subset \textbf{MB}$, $\textbf{Z}$ is not a Markov blanket of $T$.
\end{definition}

In this paper, Markov boundary is abbreviated as MB. According to Definition \ref{markov_blanket}, the mutual information $I(T,\textbf{\emph{U}})=I(T,\textbf{\emph{MB}})$, and thus, MB set carries all of the predictive information about the corresponding target. %Due to the prominent property, MB has been widely used in feature selection task \cite{ccmb}, and the unique MB has been proved as the optimal solution for feature selection \cite{tsamardinos2003towards,masegosa2012bayesian}.
%\begin{figure}[t]
%  \centering
%  \includegraphics[height=0.90in, width=1.80in]{pic_bn_mb}
%  \caption{An example of Markov boundary. The MB of $T$ includes parents $\{A,B\}$, children $\{G,H\}$ and spouse $\{F\}$.}\label{pic_bn_mb}
%\end{figure}
From the perspective of causality, MB provides a complete picture of the local causal structure around the target \cite{ccmb}, which could be intuitively understood in causal Bayesian network \cite{pearl1988}. An example of MB with Directed Acyclic Graph (DAG) is shown in Fig. \ref{pic_example_mb}. Statnikov \emph{et al.} \cite{statnikov2013algorithms} proved that, in a faithful \cite{pearl1988} Bayesian network, MB of a variable includes its parents (direct causes), children (direct effects) and spouses (other direct causes of direct effects). Thus, MB of target Common Cold contains direct causes (Frigid Weather, Specific gene sequence of Rhinovirus, and Same gene sequence of the two viruses), direct effects (Coughing, Fatigue, and Allergy) and other direct causes of the direct effects (Pollen). And the remaining variables are independent of Common Cold conditioned on its MB. The variables in an MB are called causal variables.

Extensive researches assume that the target has a unique MB, and propose many effective methods, which can be broadly classified into two types according to the review \cite{yu2019causality}, i.e., simultaneous MB learning algorithms and divide-and-conquer MB learning algorithms. Some early proposed methods, such as IAMB \cite{iamb} and its variants \cite{fast_iamb,iamb}, are simultaneous MB learning algorithms. These algorithms do not distinguish between the parent-child variables and spouse variables and learn them simultaneously. Thus, they are time-efficient but require the number of samples to be exponential to the size of the MB, which means that insufficient samples will result in the performance
degradation \cite{pcmb}. Divide-and-conquer MB learning algorithms are proposed to further improve the MB discovery accuracy with a reasonable time cost, which firstly search the parent-child variables and then spouse variables of a target. Classical methods include MMMB \cite{mmmb} and PCMB \cite{pcmb}, and the state-of-the-art STMB \cite{stmb}, TLMB \cite{wu2020tolerant} and CCMB \cite{ccmb}. Most of these algorithms are efficient to seek an approximate MB set with a reasonable time cost. The above-mentioned algorithms assume that the probability distribution is strictly positive according to Theorem \ref{unique_mb}:

\begin{theorem}
\label{unique_mb}
\cite{neapolitan2004learning} If the joint probability distribution $\mathbb{P}$ satisfies Intersection property, then a target has a unique MB.
\end{theorem}

Under certain assumption, these algorithms have good performances and also have been widely applied in causal feature selection. However, in real-world applications, the unique MB assumption is always violated, leading to multiple equivalent MBs for a target. For example, if a variable is completely determined by another (e.g., binary variable $X$ and $Y$ satisfies $P(X=1|Y=1)=1$, $P(X=0|Y=0)=1$), then the Intersection property is violated when taking one of the two as the target. Some relevant theories and algorithms about multiple MBs are reviewed next.

\subsection{Multiple MBs and Equivalent Information}

\begin{figure}[t]
  \centering
  \includegraphics[height=0.90in, width=1.40in]{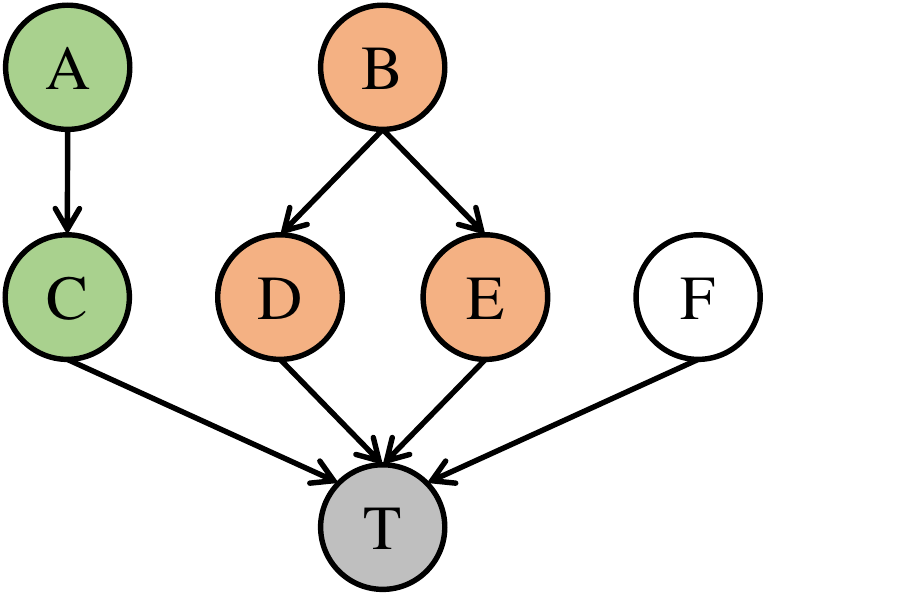}
  \includegraphics[height=1.80in, width=3.40in]{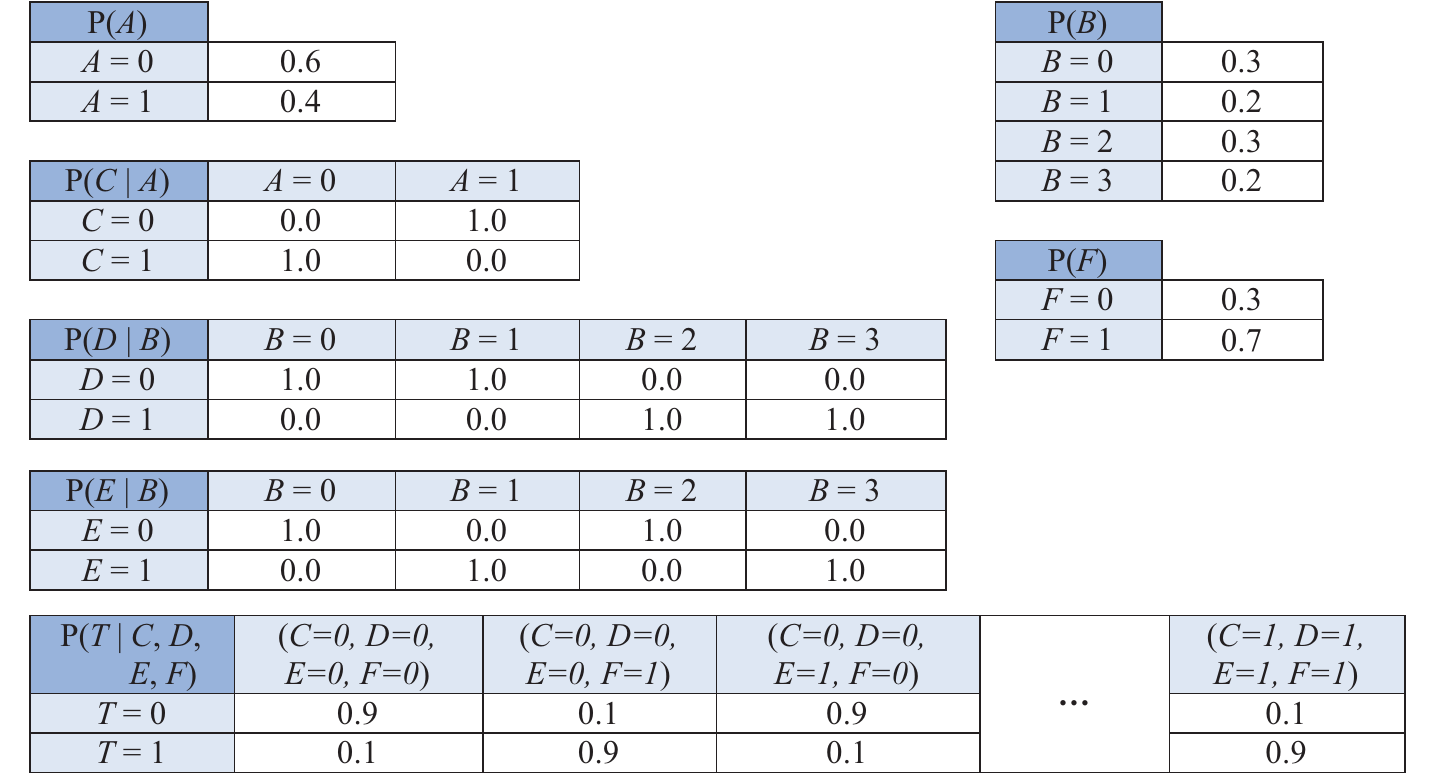}
  \caption{An example of Equivalent information. The response variable is $T$. All variables take values $\{0,1\}$ except for $B$ that takes values $\{0,1,2,3\}$. Variables $A$ and $C$ contain equivalent information about $T$ and are highlighted with the same color. Likewise, two variables $\{D,E\}$ jointly and a single variable $B$ contain equivalent information about $T$ and thus are also highlighted with the same color.}\label{pic_EI}
\end{figure}

When the joint probability distribution $\mathbb{P}$ dose not satisfy the Intersection property, there exists a phenomenon, called equivalent information.

\begin{definition}
\label{equivalent_information}
(Equivalent information) \cite{statnikov2013algorithms} Variable subsets $\textbf{X}$ and $\textbf{Y}$ contain equivalent information about target variable $T$ conditioned on $\textbf{Z}$ if and only if $T\notperp \textbf{X}|\textbf{Z}$, $T\notperp \textbf{Y}|\textbf{Z}$, $T\perp \textbf{X}|\textbf{Y}\cup\textbf{Z}$, $T\perp \textbf{Y}|\textbf{X}\cup\textbf{Z}$.
\end{definition}

Fig. \ref{pic_EI} provides an example of equivalent information. According to the probability distribution, we can conclude that $A\notperp T$, $C\notperp T$, $A\perp T|C$ and $C\perp T|A$. Thus, $A$ and $C$ contain equivalent information about $T$. The same analysis can be conducted on $\{D,E\}$ and $B$. Moreover, it can be seen from the probability distribution in Fig. \ref{pic_EI} that, the MB of $T$ could be: $\{A,B,F\}, \{C,B,F\}, \{A,D,E,F\}$, or $\{C,D,E,F\}$, which verifies the coexistence of the equivalent information phenomenon and multiple MBs. The phenomenon is formally stated as Theorem \ref{not_unique_mb}.

\begin{theorem}
\label{not_unique_mb}
\cite{liu2018markov} The intersection property holds if and only if no information equivalence occurs.
\end{theorem}

According to Theorem \ref{not_unique_mb}, if the target has multiple MBs, there must exist equivalent information. Some algorithms are proposed to detect multiple MBs, such as KIAMB \cite{pcmb}, TIE* \cite{statnikov2013algorithms} and WLCMB \cite{liu2016swamping}. KIAMB is a stochastic extension of IAMB, which tries to get all MBs by running a large number of times. Thus, there is no guarantee for the outputs of KIAMB. TIE* is a multiple MB discovery framework and WLCMB is essentially an instance of TIE* framework. They are relatively more efficient than KIAMB. However, it is still not tractable to find all of the possible MBs since the unpredictable number of MBs makes the process time-consuming.

\subsection{How Does MB Assist Feature Selection?}

An important application of MB is causality-based feature selection \cite{guyon2007causal,yu2019causality}. We first explain the main difference between traditional and causality-based feature selection methods. Traditional feature selection methods search the relevant features with consideration of the correlation between features and the class attribute (target) \cite{zhang2017self,pang2018efficient,zhu2016robust,zhang2020unsupervised}. However, correlations capture only the co-occurrence of features and the target, while cause-effect relationships imply the underlying mechanism of the occurrence of the target and they are persistent across different environments \cite{guyon2007causal,yu2019causality}. MB discovery algorithms aims to learn the MB of the class attribute, which represents the local causal-effect relationships around the class attribute without learning an entire Bayesian network. Thus, MB discovery algorithms have potential abilities to select more robust and interpretable features than traditional feature selection algorithms even under the distribution shift \cite{aliferis2010locala}. Qualitatively, MB can be interpreted in terms of Kohavi-John feature relevance according to \cite{tsamardinos2003towards}: irrelevant feature is disconnected from target in the DAG, and weakly relevant features connect with target but not belong to MB, and strongly relevant features are included in the MB. More specifically, Pellet \emph{et al.} have proved that MB is the optimal solution for feature selection problem under the faithfulness condition \cite{pellet2008using}. With the theoretical guarantee, MB has been widely used to select causal features in high-dimensional data, which are not only predictive but also causally informative.

\section{Multi-label Causal Variable Discovery}
\label{secccf}

Compared with single-label problem, the additional information introduced to multi-label problem is the relationship between labels, which is crucial for the analysis of multi-label data \cite{zhang2013review}. Due to the label relationships, causal variables in multi-label data include two types, i.e., aforementioned common causal variables and label-specific causal variables. In the following, we formally define these variables in Section \ref{secccf_def}, and discuss their properties with and without consideration of relationships between labels in Section \ref{under_hypo} and Section \ref{not_under_hypo} respectively. Based on the theoretical property, a discovery and distinguishing algorithm is proposed in Section \ref{secccf_alg}.

\subsection{Definition and Hypothesis: Common Causal Variables, Label-specific Causal Variables, and Label-causality Hypothesis}
\label{secccf_def}

%In this section, we first analyze the property of common causal variables since these two types of variables can be naturally divided if the common causal variables of all target combinations are located.
As a concept derived from causal variables, we formally give the definitions of common causal variables and label-specific causal variables based on MB.
\begin{definition}
\label{common_causal_feature}
For variable set $\textbf{U}$ and label set $\textbf{T}=\{T_1,T_2,...,T_k\}$, variable $X$ is a common causal variable of multiple labels in $\textbf{T}$ if for $\forall T\in\textbf{T}$, there exists an MB set of $T$ including $X$. Variable $X$ is a label-specific causal variable if for only one $T\in\textbf{T}$, there exists an MB set of $T$ including $X$.
\end{definition}

\begin{figure}[t]
  \centering
  \includegraphics[height=0.80in, width=2.24in]{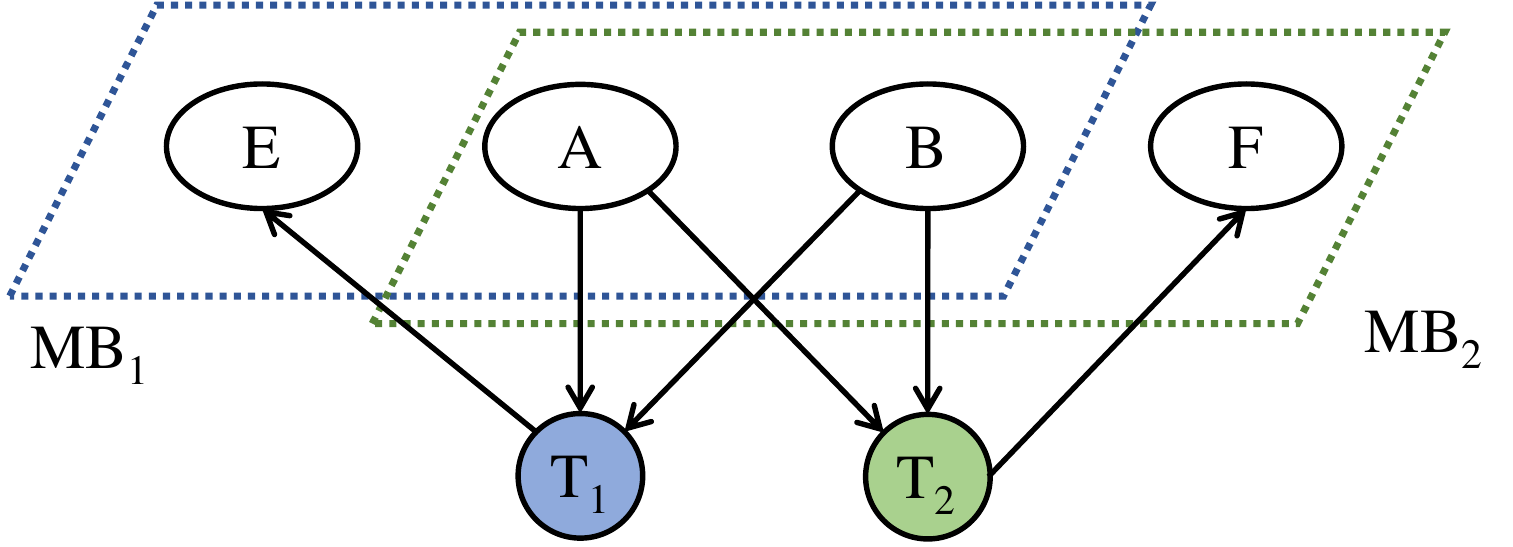}
  \caption{An example of common causal variables and label-specific causal variables without multiple MBs.}\label{pic_ccv_tscv_withoutei}
\end{figure}

\begin{figure}[t]
  \centering
  \includegraphics[height=1.20in, width=2.72in]{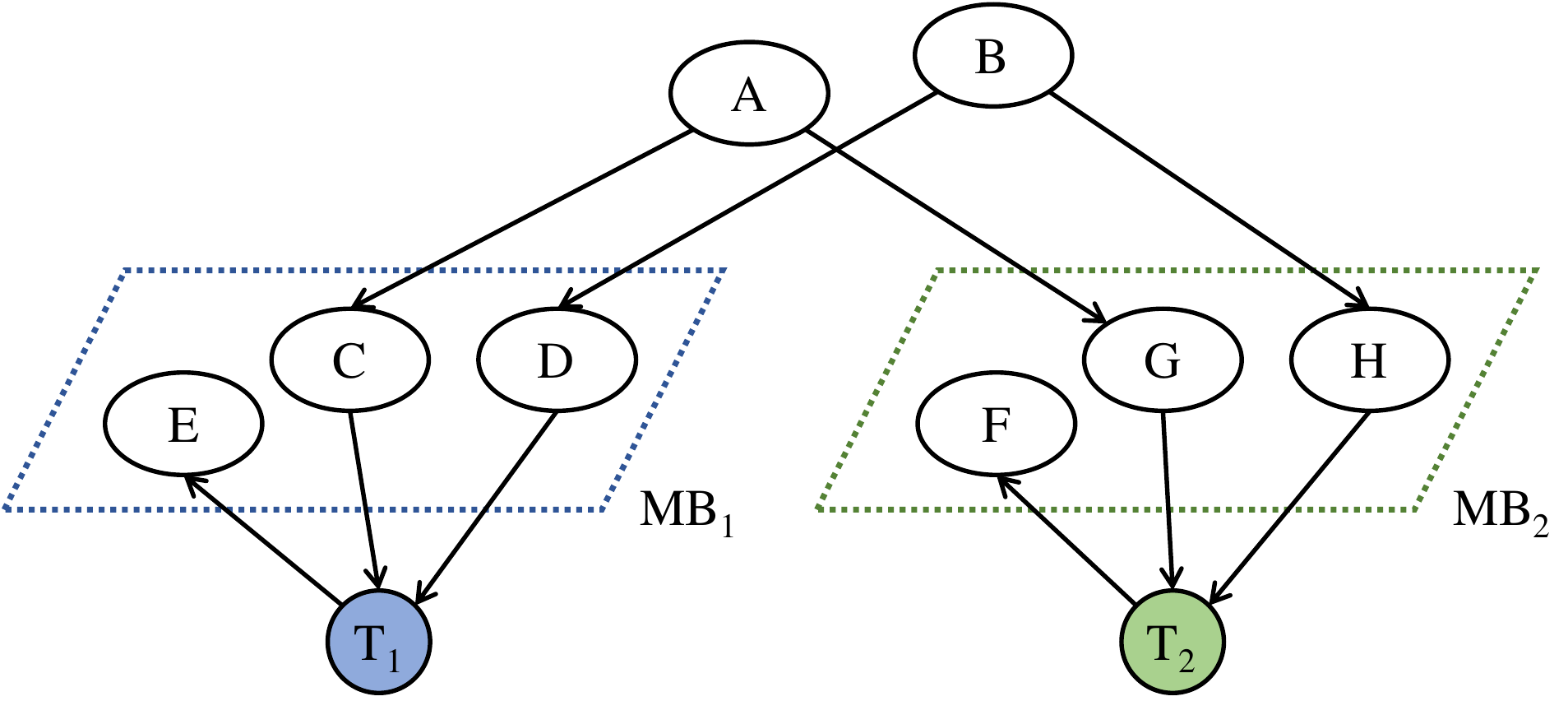}
  \caption{An example of common causal variables and label-specific causal variables with multiple MBs.}\label{pic_ccv_tscv_withei}
\end{figure}

Fig. \ref{pic_ccv_tscv_withoutei} and Fig. \ref{pic_ccv_tscv_withei} provide examples without and with multiple MBs respectively. In Fig. \ref{pic_ccv_tscv_withoutei}, the MB of $T_1$ is $\{A,B,E\}$, and the MB of $T_2$ is $\{A,B,F\}$. Neither of $T_1$ and $T_2$ has multiple MBs, and the intersection of their MB sets is non-empty. Thus, $A$ and $B$ are common causal variables of $T_1$ and $T_2$. $E$ is a label-specific causal variable of $T_1$, and $F$ is a label-specific causal variable of $T_2$. Further, we present a case with multiple MBs in Fig. \ref{pic_ccv_tscv_withei}. Assume that $\{A,B,E\}$ and $\{C,D,E\}$ are MBs of $T_1$, $\{A,B,F\}$ and $\{G,H,F\}$ are MBs of $T_2$\footnote{Note that the case in this assumption is possible under the DAG in Fig. \ref{pic_ccv_tscv_withei}, which can be understood from the probability table in Fig. \ref{pic_EI}}. Then, $A$ and $B$ are common causal variables of $T_1$ and $T_2$, and others are label-specific. By Definition \ref{common_causal_feature}, learning all of multiple MB sets is the premise of distinguishing these two types of variables. Thus, it is still difficult to use as a criterion for identification. Given the constant causal information, we now present another definition of common causal variables in perspective of information theory. Afterwards we prove the equivalence of these two definitions.

\begin{definition}
\label{common_causality_set}
Let $\textbf{U}$ denote variable set and $\textbf{T}=\{T_1,T_2,...,T_k\}$ denote label set. Given an MB set $\textbf{MB}_i$ and an MB subset $\textbf{Z}_i\subset\textbf{MB}_i$ for $T_i\in\textbf{T}$, if using $\textbf{Z}$ to replace $\textbf{Z}_i$ so that
\begin{equation}
\label{common_causality_eq}
I(\textbf{MB}_i-\textbf{Z}_i \cup \textbf{Z},T_i)=I(\textbf{MB}_i,T_i),
\end{equation}
and any subset of $\textbf{Z}$ does not satisfy Eq. (\ref{common_causality_eq}), then all variables in $\textbf{Z}$ are common causal variables of labels in $\textbf{T}$.
\end{definition}
%For feature set $\textbf{U}$ and target set $\textbf{T}=\{T_1,T_2,...,T_k\}$, $\textbf{MB}_i \subset \textbf{U}$ (for $i \in \{1,2,...,k\}$) is a MB set of target $T_i$. The common causal feature set (denoted as $\textbf{Z}$) of $\textbf{T}$ is a minimal subset of $\textbf{U}$ satisfying the condition: for all $T_i\in \textbf{T}$, using $\textbf{Z}$ to replace a subset of $\textbf{MB}_i$, all these MB sets have no information loss about these targets, which can be formalized as:
%\begin{equation}
%\label{common_causality_eq}
%\begin{split}
%& \forall i \in \{1,2,...,k\}, \exists \textbf{Z}_i\subset \textbf{MB}_i, s.t. \\
%& I(\bigcup_{i=1}^k (\textbf{MB}_i-\textbf{Z}_i) \cup \textbf{Z},\textbf{T})=I(\bigcup_{i=1}^k \textbf{MB}_i,\textbf{T}).
%\end{split}
%\end{equation}

\begin{theorem}
\label{definition_equivalent}
Definition \ref{common_causal_feature} and Definition \ref{common_causality_set} are equivalent.
\end{theorem}
\begin{proof}
Definition \ref{common_causal_feature} $\Rightarrow$ Definition \ref{common_causality_set}: Assume $X$ is a variable satisfying Definition \ref{common_causal_feature}. Let $\textbf{\emph{Z}}=\{X\}$, then there exists $\textbf{\emph{MB}}_i$ for $\forall T_i$ such that $\textbf{\emph{Z}}\subset\textbf{\emph{MB}}_i$. Let $\textbf{\emph{Z}}_i=\textbf{\emph{Z}}$ for $\forall T_i$, then we have $I(\textbf{\emph{MB}}_i-\textbf{\emph{Z}}_i \cup \textbf{\emph{Z}},T_i)=I(\textbf{\emph{MB}}_i,T_i)$.

Definition \ref{common_causal_feature} $\Leftarrow$ Definition \ref{common_causality_set}: Assume $\textbf{\emph{Z}}$ is a subset satisfying Definition \ref{common_causality_set}. Suppose $\exists X \in \textbf{\emph{Z}}$ is not a common causal variable, i.e., $\exists T_i$ such that $X\notin\textbf{\emph{MB}}_i$ for $\forall \textbf{\emph{MB}}_i$ of $T_i$. Then, by the chain rule of mutual information \cite{cover2012elements}:
\begin{equation}
\label{eq_definition_equivalent}
\begin{split}
& I(\textbf{\emph{MB}}_i-\textbf{\emph{Z}}_i \cup \textbf{\emph{Z}}-\{X\},T_i)=\\
&I(\textbf{\emph{MB}}_i-\textbf{\emph{Z}}_i \cup \textbf{\emph{Z}},T_i) - I(X,T_i|\textbf{\emph{MB}}_i-\textbf{\emph{Z}}_i \cup \textbf{\emph{Z}}-\{X\}).
\end{split}
\end{equation}
According to Eq. (\ref{common_causality_eq}), $\textbf{\emph{MB}}_i-\textbf{\emph{Z}}_i \cup \textbf{\emph{Z}}$ is a Markov blanket of $T_i$. Since $X\notin\textbf{\emph{MB}}_i$ for $\forall \textbf{\emph{MB}}_i$, $\textbf{\emph{MB}}_i-\textbf{\emph{Z}}_i \cup \textbf{\emph{Z}}-\{X\}$ is a Markov blanket of $T_i$. Thus,
\begin{equation}
\label{eq_definition_equivalent_sup}
I(X,T_i|\textbf{\emph{MB}}_i-\textbf{\emph{Z}}_i \cup \textbf{\emph{Z}}-\{X\})=0
\end{equation}
Substituting Eq. (\ref{eq_definition_equivalent_sup}) into Eq. (\ref{eq_definition_equivalent}) and we obtain $I(\textbf{\emph{MB}}_i-\textbf{\emph{Z}}_i \cup \textbf{\emph{Z}}-\{X\},T_i)=I(\textbf{\emph{MB}}_i,T_i)$. Thus, $\textbf{\emph{Z}}-\{X\}\subset\textbf{\emph{Z}}$ also satisfies Definition \ref{common_causality_set}, contracting the condition. Therefore, all variables in $\textbf{\emph{Z}}$ satisfy Definition \ref{common_causal_feature}. (Q.E.D.)
\end{proof}

Intuitively, Eq. (\ref{common_causality_eq}) in Definition \ref{common_causality_set} means that, the common causal variables in $\textbf{\emph{Z}}$ can be used to replace the MB subset of each label without any information loss. Thus, $\textbf{\emph{Z}}$ carries the causal information of all labels in $\textbf{\emph{T}}$, while $\textbf{\emph{Z}}_i$ only contains the causal information of label $T_i$. We can easily understand the range of common causal variables from Definition \ref{common_causal_feature}, and will identify them with the help of Definition \ref{common_causality_set}.
%Note that the common causality set can be discussed on target sets with any scale ($k\geq2$). Intuitively, Eq. (\ref{common_causality_eq}) in Definition \ref{common_causality_set} means that, the common causality set $\textbf{\emph{Z}}$ can be used to replace the MB subset of each target without information loss. We measure the mutual information between these MB sets and the target set instead of each single target in Eq. (\ref{common_causality_eq}) since the influence of target correlation must be taken into consideration. Thus, $\textbf{\emph{Z}}$ carries the common causal information of all of the targets in $\textbf{\emph{T}}$, while $\textbf{\emph{Z}}_i$ only contain the target-specific causal information of target $T_i$.

By Definition \ref{common_causality_set}, the problem discussed in this paper can be described as: for variable set $\textbf{\emph{U}}$ and label set $\textbf{\emph{T}}=\{T_1,T_2,...,T_k\}$, we need to search two types of causal variables from $\textbf{\emph{U}}$, i.e., (1) the common causal variables of $\textbf{\emph{T}}$ and all subsets of $\textbf{\emph{T}}$, and (2) the label-specific causal variables of each single label $T_i \in \textbf{\emph{T}}$.
Different from the single-label problem, a special issue must be discussed in multi-label case, i.e., the possible causality in label set. To simplify the problem, we first propose the Label-causality Hypothesis as follows:

%\textbf{Hypothesis.} (Label-causality Hypothesis) $\forall T_1, T_2 \in \textbf{\emph{T}}$, $T_1\notin \textbf{\emph{MB}}(T_2)$ and $T_2\notin \textbf{\emph{MB}}(T_1)$.

\begin{hypothesis}
\label{hypo}
(Label-causality Hypothesis) $\forall T_1, T_2 \in \textbf{T}$, $T_1\notin \textbf{MB}(T_2)$ and $T_2\notin \textbf{MB}(T_1)$.
%%Any target is not a direct cause or effect of another target.
\end{hypothesis}

Label-causality Hypothesis considers the relationships between labels, and divides the discussion according to whether a label contains the critical causal information about another label. Note that the hypothesis allows the indirect causality between labels, for which an eligible example is that a label influences another label through a non-label variable. In Sections \ref{under_hypo} and \ref{not_under_hypo}, the property of common causal variable is discussed in the cases where Label-causality Hypothesis satisfied and violated, respectively.

\subsection{Discussion under the Hypothesis}
\label{under_hypo}

Eq. (\ref{common_causality_eq}) has a constant solution of $\textbf{\emph{Z}}$. According to Definition \ref{common_causality_set}, let $\textbf{\emph{Z}}=\bigcap_{i=1}^k \textbf{\emph{MB}}_i$ and $\textbf{\emph{Z}}_i=\textbf{\emph{MB}}_i-\textbf{\emph{Z}}$, then it is readily justified that the intersection of MB sets of multiple labels is a common causal variable set of labels in $\textbf{\emph{T}}$. Therefore, if all of the labels have the unique MB, then the intersection of these MBs is the intact common causal variable set. However, the real-world applications always violate the unique MB assumption, and most labels have multiple MBs.

Directly finding the multiple MBs is time-consuming since the time complexity is exponential to the size of variable set. It will also suffer from incorrect independence tests due to the large conditioning sets in the process. Detecting whether the Intersection property is violated is a possible method according to Theorem \ref{unique_mb}, whereas it is infeasible to identify the strictly positive joint probability distribution. Another criterion for unique MB, by Theorem \ref{not_unique_mb}, is to detect the equivalent information, as mainly discussed in this section. As the multi-label problem has complex relationships, equivalent information phenomenon has diversified forms. It can be roughly classify into two types, i.e., equivalent information about labels and about labels. To narrow the discussion, we first prove that only the equivalent information about labels has influence on the identification of common causal variables.

\begin{theorem}
\label{feature_noinfluence}
For a label $T$, if $\forall \textbf{X}, \textbf{Y}, \textbf{Z} \subset \textbf{U} - \{T\}$, $\textbf{X}$ and $\textbf{Y}$ do not contain equivalent information about $T$ conditioned on $\textbf{Z}$, then $T$ has a unique MB.
\end{theorem}

\begin{proof}
Assuming that $T$ has two MB sets $\textbf{\emph{MB}}_1$ and $\textbf{\emph{MB}}_2$, then we need to find two variable sets containing equivalent information about $T$. According to Definition \ref{markov_blanket}, we have:
\begin{equation}
\label{eq1}
\begin{split}
& T \perp \textbf{\emph{U}}-\textbf{\emph{MB}}_1-\{T\}|\textbf{\emph{MB}}_1,  T \perp \textbf{\emph{U}}-\textbf{\emph{MB}}_2-\{T\}|\textbf{\emph{MB}}_2.
\end{split}
\end{equation}
According to the Decomposition property in Theorem \ref{common_property}, Eq. (\ref{eq1}) indicates that:
\begin{equation}
\label{eq4}
\begin{split}
& T \perp \textbf{\emph{MB}}_2-\textbf{\emph{MB}}_1|\textbf{\emph{MB}}_1,  T \perp \textbf{\emph{MB}}_1-\textbf{\emph{MB}}_2|\textbf{\emph{MB}}_2.
\end{split}
\end{equation}
Now we prove that $\textbf{\emph{MB}}_1-\textbf{\emph{MB}}_2$ and $\textbf{\emph{MB}}_2-\textbf{\emph{MB}}_1$ contain equivalent information about $T$ conditioned on $\textbf{\emph{MB}}_1\cap\textbf{\emph{MB}}_2$. Assume that $T\perp(\textbf{\emph{MB}}_1-\textbf{\emph{MB}}_2)|\textbf{\emph{MB}}_1\cap\textbf{\emph{MB}}_2$. Considering with Eq. (\ref{eq1}), we obtain the following relationship according to Contraction property in Theorem \ref{common_property}:
\begin{equation}
\label{eq2}
T \perp (\textbf{\emph{U}}-\textbf{\emph{MB}}_1-\{T\}) \cup (\textbf{\emph{MB}}_1-\textbf{\emph{MB}}_2) | \textbf{\emph{MB}}_1\cap\textbf{\emph{MB}}_2.
\end{equation}
Simplify the Eq. (\ref{eq2}), then:
\begin{equation}
\label{eq3}
T \perp (\textbf{\emph{U}}-\{T\}-\textbf{\emph{MB}}_1\cap\textbf{\emph{MB}}_2) | \textbf{\emph{MB}}_1\cap\textbf{\emph{MB}}_2.
\end{equation}
We can concluded from Eq. (\ref{eq3}) that $\textbf{\emph{MB}}_1\cap\textbf{\emph{MB}}_2$ is an Mb of $T$ according to Definition \ref{markov_blanket}. However, $\textbf{\emph{MB}}_1\cap\textbf{\emph{MB}}_2\subset \textbf{\emph{MB}}_1$ and $\textbf{\emph{MB}}_1\cap\textbf{\emph{MB}}_2\subset \textbf{\emph{MB}}_2$, which leads to $\textbf{\emph{MB}}_1$ and $\textbf{\emph{MB}}_2$ are Mb instead of MB, contradicting the condition. Therefore,
\begin{equation}
\label{eq3-1}
T\notperp(\textbf{\emph{MB}}_1-\textbf{\emph{MB}}_2)|\textbf{\emph{MB}}_1\cap\textbf{\emph{MB}}_2.
\end{equation}
Similarly, we can prove that
\begin{equation}
\label{eq3-2}
T\notperp(\textbf{\emph{MB}}_2-\textbf{\emph{MB}}_1)|\textbf{\emph{MB}}_1\cap\textbf{\emph{MB}}_2.
\end{equation}
Combining Eq. (\ref{eq3-1}) and Eq. (\ref{eq3-2}) with Eq. (\ref{eq4}), we can conclude that, $\textbf{\emph{MB}}_1-\textbf{\emph{MB}}_2$ and $\textbf{\emph{MB}}_2-\textbf{\emph{MB}}_1$ contain equivalent information about $T$ conditioned on $\textbf{\emph{MB}}_1\cap\textbf{\emph{MB}}_2$, contradicting the condition. Hence, $T$ has a unique MB. (Q.E.D.)
\end{proof}

Theorem \ref{feature_noinfluence} proves that multiple MBs of a label are brought by the equivalent information about the corresponding label, while equivalent information on non-label variables does not influence the uniqueness of MB, as well as common causal variables. Therefore, only equivalent information about each label needs to be considered for common causal variable identification. Theorem \ref{theo_under_hypo} is proposed below to describe this criterion.

\begin{theorem}
\label{theo_under_hypo}
Let $\textbf{MB}_i$ denote the MB set of $T_i$ ($i \in \{1,2,...,k\}$) in label set $\textbf{T}=\{T_1,T_2,...,T_k\}$. Under the Label-causality Hypothesis, $\textbf{Z}\subset \textbf{U}$ is s common causal variable set of labels in $\textbf{T}$ if and only if $\exists \textbf{Z}_i\subset \textbf{MB}_i$ such that $\textbf{Z}_i$ and $\textbf{Z}$ contain equivalent information about $T_i$ conditioned on $\textbf{MB}_i-\textbf{Z}_i$ for each $T_i\in\textbf{T}$.
\end{theorem}

\begin{proof}
To prove that $\textbf{\emph{Z}}$ is a common causal variable set for $\textbf{\emph{T}}$, we need to prove that $\textbf{\emph{Z}}$ satisfies Eq. (\ref{common_causality_eq}) in Definition \ref{common_causality_set}. By the chain rule of mutual information, we express $\textbf{\emph{MB}}_i\cup \textbf{\emph{Z}}$ as $(\textbf{\emph{MB}}_i\cup \textbf{\emph{Z}}-\textbf{\emph{Z}}_i)\cup\textbf{\emph{Z}}_i$ and obtain:
\begin{equation}
\label{tuidao10}
\begin{split}
&I(\textbf{\emph{MB}}_i\cup \textbf{\emph{Z}}-\textbf{\emph{Z}}_i,T_i)\\
&=I(\textbf{\emph{MB}}_i\cup \textbf{\emph{Z}},T_i)-I(\textbf{\emph{Z}}_i,T_i|\textbf{\emph{MB}}_i\cup \textbf{\emph{Z}}-\textbf{\emph{Z}}_i).
\end{split}
\end{equation}
Since $\textbf{\emph{Z}}$ and $\textbf{\emph{Z}}_i$ contain equivalent information about $T_i$, according to Definition \ref{equivalent_information}, we have:
%\begin{equation}
%\label{eimutual}
%I(\textbf{\emph{Z}},T_i)=I(\textbf{\emph{Z}}_i,T_i).
%\end{equation}
%And thus:
\begin{equation}
\label{tuidao100}
I(\textbf{\emph{Z}}_i,T_i|\textbf{\emph{MB}}_i\cup \textbf{\emph{Z}}-\textbf{\emph{Z}}_i)=0.
\end{equation}
Since $\textbf{\emph{Z}}\perp T_i|\textbf{\emph{MB}}_i$, we have:
\begin{equation}
\label{tuidao1000}
I(\textbf{\emph{MB}}_i\cup \textbf{\emph{Z}},T_i)=I(\textbf{\emph{MB}}_i,T_i).
\end{equation}
Substitute Eq. (\ref{tuidao100}) and Eq. (\ref{tuidao1000})into Eq. (\ref{tuidao10}), thus,
\begin{equation}
\label{tuidao1}
I(\textbf{\emph{MB}}_i\cup \textbf{\emph{Z}}-\textbf{\emph{Z}}_i,T_i)=I(\textbf{\emph{MB}}_i,T_i).
\end{equation}
According to Theorem \ref{feature_noinfluence}, all common causal variables are considered in the theorem since the Label-causality hypothesis is satisfied and thus any label is not a causal variables of another. In conclusion, Theorem \ref{theo_under_hypo} is true. (Q.E.D.)
\end{proof}

%\begin{figure}[h]
%  \centering
%  \includegraphics[height=1.3in, width=2.6in]{pic_theorem4}
%  \caption{An examples of common causality set. If $\{A,B\}$ and $\{C,D\}$ contain equivalent information about $T_1$, $\{A,B\}$ and $\{G,H\}$ contain equivalent information about $T_2$, then $\{A,B\}$ is the common causality set of $\{T_1,T_2\}$.}\label{pic_theorem4}
%\end{figure}

Theorem \ref{theo_under_hypo} proves that, equivalent information between MB subset and other variable set can be used to detect common causal variables. For example in Fig. \ref{pic_ccv_tscv_withei}, $\{A,B\}$ and $\{C,D\}$ contain equivalent information about $T_1$, $\{A,B\}$ and $\{G,H\}$ contain equivalent information about $T_2$. Assume it has been known that $\{C,D,E\}$ is an MB set of $T_1$, and $\{F,G,H\}$ is an MB set of $T_2$. According to Theorem \ref{theo_under_hypo}, $\{A,B\}$ can be detected as common causal variables of $\{T_1,T_2\}$ without mining other MB sets of $T_1$ and $T_2$. %Theorem \ref{theo_under_hypo} will help us design the algorithm in the next section.

\subsection{Relax the Label-causality Hypothesis}
\label{not_under_hypo}

When the Label-causality Hypothesis is relaxed, there might exist more common causal variables undetected. Different from the case satisfying the hypothesis, the causal structure of a label is represented with non-label variables as well as labels. Therefore, it is improper to make a difference between non-label variables and labels when mining the MBs of each labels. Furthermore, the equivalent information about both labels and non-label variables need be considered so that some common causal variables are not ignored. Theorem \ref{theo_notunder_hypo} is proposed below to describe the case where common causal variables cannot be detected.

\begin{theorem}
\label{theo_notunder_hypo}
For labels $T_1, T_2 \in \textbf{T}$, $T_1\in\textbf{MB}_2$ and $T_2\in\textbf{MB}_1$, variable subset $\textbf{Z}$ is a common causal variable set of $T_1$ and $T_2$ but might not be detected if the following statements hold: (1) $\textbf{Z}\subset \textbf{MB}_2$. $T_2$ and $\textbf{Z}$ contain equivalent information about $T_1$ conditioned on $\textbf{MB}_1-\{T_2\}$. (2) $\textbf{Z}\subset \textbf{MB}_1$ and $\textbf{Z}\subset \textbf{MB}_2$. $T_1$ and $T_2$ contain equivalent information about $\textbf{Z}$.
\end{theorem}

\begin{proof}
For (1): Since $\textbf{\emph{Z}}\subset \textbf{\emph{MB}}_2$, $\textbf{\emph{Z}}$ satisfies Eq. (\ref{common_causality_eq}) in Definition \ref{common_causality_set} for $T_2$. We prove that $\textbf{\emph{Z}}$ satisfies Eq. (\ref{common_causality_eq}) for $T_1$. According to the chain rule of mutual information, we have:
\begin{equation}
\label{theo3eqres0}
\begin{split}
&I(T_1,(\textbf{\emph{MB}}_1-\{T_2\})\cup\textbf{\emph{Z}})\\
&=I(T_1,\textbf{\emph{Z}}|\textbf{\emph{MB}}_1-\{T_2\})+I(T_1,\textbf{\emph{MB}}_1-\{T_2\}),
\end{split}
\end{equation}
Also, $\textbf{\emph{MB}}_1$ can be split into $\textbf{\emph{MB}}_1-\{T_2\}$ and $\{T_2\}$:
\begin{equation}
\label{theo3eqres01}
\begin{split}
I(T_1,\textbf{\emph{MB}}_1)=I(T_1,T_2|\textbf{\emph{MB}}_1-\{T_2\})+I(T_1,\textbf{\emph{MB}}_1-\{T_2\}).
\end{split}
\end{equation}
Since $T_2$ and $\textbf{\emph{Z}}$ contain equivalent information about $T_1$, thus:
\begin{equation}
\label{theo3eqres}
I(T_1,\textbf{\emph{Z}}|\textbf{\emph{MB}}_1-\{T_2\})=I(T_1,T_2|\textbf{\emph{MB}}_1-\{T_2\}).
\end{equation}
Then, substituting Eq. (\ref{theo3eqres}) into Eq. (\ref{theo3eqres0}) and Eq. (\ref{theo3eqres01}), we obtain:
\begin{equation}
\label{theo3eq}
I(T_1,(\textbf{\emph{MB}}_1-\{T_2\})\cup\textbf{\emph{Z}})=I(T_1,\textbf{\emph{MB}}_1).
\end{equation}
Thus, $\textbf{\emph{Z}}$ is a common causal variable set of $T_1$ and $T_2$. Since $T_1 \perp \textbf{\emph{Z}}|\textbf{\emph{MB}}_1$ and $T_2\in\textbf{\emph{MB}}_1$, if $T_2$ is selected by the MB discovery algorithm first, then $\textbf{\emph{Z}}$ will be excluded in the MB set according to the Decomposition property in Theorem 1.

For (2): It is readily justified that the variables in $\textbf{\emph{Z}}$ are common causal variables of $T_1$ and $T_2$ according to Definition \ref{common_causal_feature}. Since $T_1$ and $T_2$ contain equivalent information about $\textbf{\emph{Z}}$, then $T_1 \perp \textbf{\emph{Z}}| T_2$ and $T_2 \perp \textbf{\emph{Z}}| T_1$. If $T_1$ is selected by the MB discovery algorithm before $\textbf{\emph{Z}}$ when selecting MB of $T_2$ and $T_2$ is selected before $\textbf{\emph{Z}}$ when selecting MB of $T_1$, then $\textbf{\emph{Z}}$ can not be found. (Q.E.D.)
\end{proof}

%\begin{figure}[h]
%  \centering
%  \includegraphics[height=0.35in, width=1.1in]{pic_without_hypothesis}
%  \caption{A toy example of Theorem \ref{theo_notunder_hypo}.}\label{pic_without_hypothesis}
%\end{figure}

We use an example to illustrate Theorem \ref{theo_notunder_hypo}. If a label $T_1$ and the common causal variable $A$ contain equivalent information about another label $T_2$, then $A$ might be ignored since $A\perp T_2|T_1$, which describes the case in Theorem \ref{theo_notunder_hypo} (1). While the same risk does not exist under the case that two variables contain equivalent information about a label, since these variables are found when detecting the equivalent information according to Theorem \ref{theo_under_hypo}. By Theorem \ref{theo_notunder_hypo} (1), it is necessary to treat all of the labels and non-label variables as ordinary variables so that some common causal variables are not ignored due to the influence of labels. For Theorem \ref{theo_notunder_hypo} (2), also using the above example, when the two labels $T_1$ and $T_2$ contain equivalent information about common causal variable $A$, $A$ might be discarded since it might be excluded when searching MB sets both of $T_1$ and $T_2$. To solve this problem, we can remove a label first and continue to search the undetected variables.
%解读这个例子。然后说，相同的情况在特征与特征之间包含标签的等价信息时不会发生，因为根据定理5，他们会作为等价的特征而被找到。Theorem7(1)告诉我们，在寻找具有等价信息的特征时，我们不应该区别对待特征和标签，而是将所有标签同意作为特征来对待，这样就可以找到这些被忽略的CCS。
%解读这个例子。然后说，只要移除掉符合条件的target，然后继续寻找即可。

\subsection{Algorithm: Learn the Common Causal Variables and the Label-specific Causal Variables}
\label{secccf_alg}

Based on the property of common causal variables, we propose the Common and Label-specific Causal variables Discovery (CLCD) algorithm. For the sake of preciseness, the above analyses provide the corresponding conditioning set where the equivalent information exists. While in the design of algorithm, considering the complex conditioning set will introduce some time-consuming and unreliable processes. For comprehensive consideration of effectiveness and efficiency, we adopt a simplified strategy presented in \cite{statnikov2013algorithms}, i.e., assume that all information equivalence relations are context-independent and there is no need to consider the conditioning sets \cite{statnikov2013algorithms}. CLCD consists of three phases:

% \\ \ \ \
\begin{algorithm}[t]
        \caption{The \emph{CLCD} Algorithm.}
        \label{algccfd}
        \begin{algorithmic}[1]
            \STATE{\textbf{Input:}} Label set $\textbf{\emph{T}}$ and variable set $\textbf{\emph{U}}$; A divide-and-conquer-based MB discovery algorithm $\mathbb{A}$ with significance level $\alpha$.\\
            \{\emph{\textbf{Phase 1: Search an MB for each label.}}\}
            \STATE\textbf{for} each $T\in\textbf{\emph{T}}$ \textbf{do}
            \STATE\ \ \ $\textbf{\emph{PC}}_T, \textbf{\emph{SP}}_T, \textbf{\emph{C}}_T\leftarrow$ Mine the local causal structure of $T$ \\ \ \ \ from $\textbf{\emph{T}}\cup\textbf{\emph{U}}-\{T\}$ using $\mathbb{A}$, and record the direct causes \\ \ \ \ and effects $\textbf{\emph{PC}}_T$, other direct causes of direct effects \\ \ \ \ $\textbf{\emph{SP}}_T$ with the corresponding effects in $\textbf{\emph{C}}_T$
            \STATE\textbf{end for}\\
            \{\emph{\textbf{Phase 2: Retrieve the ignored variables.}}\}
            \STATE\textbf{for} each $T_i,T_j\in\textbf{\emph{T}}$ \textbf{do}
            \STATE\ \ \ \textbf{if} $T_i\in \textbf{\emph{PC}}_j$ \textbf{do}
            \STATE\ \ \ \ \ \ \textbf{for} each $\textbf{\emph{Z}}$ satisfying $\textbf{\emph{Z}}\notperp T_i$, $\textbf{\emph{Z}}\notperp T_j$, $\textbf{\emph{Z}}\perp T_i|T_j$ and \\ \ \ \ \ \ \ $\textbf{\emph{Z}}\perp T_j|T_i$ \textbf{do}
            %\STATE\ \ \ \ \ \ \ \ \ $\textbf{\emph{PC}}_{T_i}\leftarrow\textbf{\emph{PC}}_{T_i}\cup\textbf{\emph{Z}}$ if $\forall \textbf{\emph{S}}\subset\textbf{\emph{PC}}_j-\{T_j\}$, $\textbf{\emph{Z}}\notperp T_i|\textbf{\emph{S}}$.
            \STATE\ \ \ \ \ \ \ \ \ $\textbf{\emph{PC}}_{T_j}\leftarrow\textbf{\emph{PC}}_{T_j}\cup\textbf{\emph{Z}}$ if $\forall \textbf{\emph{S}}\subset\textbf{\emph{PC}}_j-\{T_i\}$, $\textbf{\emph{Z}}\notperp T_j|\textbf{\emph{S}}$.
            \STATE\ \ \ \ \ \ \textbf{end for}
            \STATE\ \ \ \textbf{end if}
            \STATE\textbf{end for}\\
            \{\emph{\textbf{Phase 3: Distinguish process.}}\}
            \STATE\textbf{for} $X\in\textbf{\emph{T}}\cup_{T\in\textbf{\emph{T}}} \textbf{\emph{C}}_T$ \textbf{do}
            \STATE\ \ \ \textbf{for} each $\textbf{\emph{Z}}\subset\textbf{\emph{U}}-\textbf{\emph{PC}}_X$ and $\textbf{\emph{Z}}\notperp X$ \textbf{do}
            \STATE\ \ \ \ \ \ \textbf{if} $\exists\textbf{\emph{S}}\subset \textbf{\emph{PC}}_X$ \emph{s.t.} $X\perp \textbf{\emph{Z}}|\textbf{\emph{S}}$ and $X\perp\textbf{\emph{S}}|\textbf{\emph{Z}}$ \textbf{then}
            \STATE\ \ \ \ \ \ \ \ \ $\textbf{\emph{EI}}_X=\textbf{\emph{EI}}_X\cup\{<\textbf{\emph{S}},\textbf{\emph{Z}}>\}$
            \STATE\ \ \ \ \ \ \textbf{end if}
            \STATE\ \ \ \textbf{end for}
            \STATE\textbf{end for}
            \STATE Common causal variables for any label (sub)sets $\textbf{\emph{T}}_\textbf{\emph{S}}\subset\textbf{\emph{T}}$: $\textbf{\emph{CCV}}_{\textbf{\emph{T}}_\textbf{\emph{S}}}\leftarrow\{X|X\in\textbf{\emph{Z}}$ where $\Theta_{\textbf{\emph{T}}_\textbf{\emph{S}}}(\textbf{\emph{Z}})=1\}$, and label -specific causal variables for each label $T\in\textbf{\emph{T}}$: $\textbf{\emph{TCV}}_T=\{X|X\in\textbf{\emph{MB}}_T$ and $X\notin\textbf{\emph{CCV}}_{\textbf{\emph{T}}_\textbf{\emph{S}}}$ for $\forall \textbf{\emph{T}}_\textbf{\emph{S}}$ including $T\}$.
        \end{algorithmic}
    \end{algorithm}

Phase 1: Mine the Local causal structures around these labels. Though each label could have multiple MBs, Phase 1 only needs to find one of them. According to Theorem \ref{theo_notunder_hypo} (1), CLCD equally treats labels and non-label variables and only focuses on the causal relationships between them. A divide-and-conquer-based MB discovery algorithm $\mathbb{A}$ is used so that CLCD can distinguish the direct cause and effect set (also called parent-child variable set and denoted as $\textbf{\emph{PC}}_T$) and other direct cause of direct effect set (also called spouse variable set and denoted as $\textbf{\emph{SP}}_T$) of $T$. Here, the corresponding child of each spouse in $\textbf{\emph{SP}}_T$ also needs to be recorded, which will be used in Phase 3.

Phase 2: To guarantee the accuracy when Label-causality Hypothesis is violated, Phase 2 retrieves the ignored variables whose information is equivalently included by two labels, which is the case described in the Theorem \ref{theo_notunder_hypo} (2). For each pair of labels where one is included by the MB of another (Line 6), Line 7 finds the \textbf{\emph{Z}} satisfying the condition in Theorem \ref{theo_notunder_hypo} (2), which is retrieved in Lines 8-9.

Phase 3: Find the variables containing equivalent information first and then discover the common and label-specific causal variables. Since independence tests with large-scale variable sets will be involved if we directly find the equivalent subsets from MB of each label as described in Theorem \ref{theo_under_hypo}, CLCD searches the common causal variables from parent-child set and spouse set, respectively. Thus, both of the equivalent causal variables of labels in $\textbf{\emph{T}}$ and variables in $\textbf{\emph{C}}$ are recorded to the $\textbf{\emph{EI}}$ of each corresponding variable (Lines 12-18) to make preparations for the discovery of common causal variables. According to Theorem \ref{theo_under_hypo}, the variables in subset $\textbf{\emph{Z}}$ are common causal variables if at least one of the three conditions in Eq. (\ref{aleq1}) are satisfied for each label in $\textbf{\emph{T}}$, which can be formalized as the logical operation in Eq. (\ref{aleq1}).
\begin{equation}
\label{aleq1}
\Theta_{\textbf{\emph{T}}}(\textbf{\emph{Z}}) = \bigwedge_{T\in\textbf{\emph{T}}}(\theta_1(\textbf{\emph{Z}},T) \vee \theta_2(\textbf{\emph{Z}},T) \vee \theta_3(\textbf{\emph{Z}},T))
\end{equation}
\begin{itemize}
\item $\theta_1(\textbf{\emph{Z}},T)=1$ when $\exists \textbf{\emph{Z}}_T\subset\textbf{\emph{MB}}_T$ s.t. $\textbf{\emph{Z}}=\textbf{\emph{Z}}_T$, and 0 otherwise.
\item $\theta_2(\textbf{\emph{Z}},T)=1$ when $\exists \textbf{\emph{Z}}_T\subset\textbf{\emph{PC}}_T$ s.t. $<\textbf{\emph{Z}},\textbf{\emph{Z}}_T>\in \textbf{\emph{EI}}_T$, and $0$ otherwise.
\item $\theta_3(\textbf{\emph{Z}},T)=1$ when $\exists \textbf{\emph{Z}}_T\subset\textbf{\emph{SP}}_T$ s.t. $<\textbf{\emph{Z}},\textbf{\emph{Z}}_T>\in \textbf{\emph{EI}}_C$, and $0$ otherwise, where $C$ is the common effect of $\textbf{\emph{Z}}_T$ and $T$.
\end{itemize}
Specifically, for a label $T$ and variable subset $\textbf{\emph{Z}}$, $\theta_1(\textbf{\emph{Z}},T)=1$ indicates that $\textbf{\emph{Z}}$ is a subset of the searched MB, and $\theta_2(\textbf{\emph{Z}},T)=1$ indicates that $\textbf{\emph{Z}}$ is equivalent with a subset of the searched PC set, and $\theta_3(\textbf{\emph{Z}},T)=1$ indicates that $\textbf{\emph{Z}}$ is equivalent with a subset of the searched SP set. $\textbf{\emph{Z}}$ satisfying one of the conditions contains critical causal information about this label, and thus variables in $\textbf{\emph{Z}}$ making $\Theta_{\textbf{\emph{T}}}(\textbf{\emph{Z}})=1$ are common causal variables of labels in $\textbf{\emph{T}}$. Thus, in Line 19, we obtain common causal variable set $\textbf{\emph{CCV}}_{\textbf{\emph{T}}_\textbf{\emph{S}}}$ of any label subset $\textbf{\emph{T}}_\textbf{\emph{S}}$ and label-specific causal variable set $\textbf{\emph{TCV}}_T$ for each label $T$.

\section{Applying CLCD to Multi-label Feature Selection}
\label{sec_fs}

%Actually, the proposed concept of common causal variables, is frequently used and considered, yet it has not been formally discussed before this research. Thus, many tasks will benefit from the proposed CLCD in Section \ref{secccf_alg}, such as mining task of common causes and effects of multiple targets, multi-target causal inference, and multi-label feature selection. In this section, we apply CLCD to multi-label feature selection, to facilitate the interpretability.

To demonstrate the generality of CLCD proposed in Section \ref{secccf_alg}, we apply CLCD to multi-label feature selection problem. Compared with single-label feature selection, the additional introduced label correlations construct more complex relationships in multi-label data, including feature-feature, feature-label and label-label relationships. Although the label relationships are crucial, it is unreasonable to specially treat them as a more important information. Conversely, the ideal strategy is to consider the relationships of all variables in a unified framework. Therefore, in this section, we try to use CLCD to handle these various forms of complex relationships in consideration of its ability to map the complex relationships among features and labels to a causal graph model. And the process of constructing the skeleton of causal graph on multi-label data naturally takes all types of causal relationships into consideration, which can be easily ``read'' from the causal graph. In the following, the novel CLCD-driven multi-label Feature Selection (CLCD-FS) algorithm is present in Section \ref{sec_fsalg1}. Subsequently, the relevance and redundancy of CLCD-FS are analyzed in Section \ref{sec_fsalg2}, and the time complexity of CLCD-FS is analyzed in Section \ref{sec_timeanalysis}.

\subsection{CLCD-FS Algorithm}
\label{sec_fsalg1}

Pellet and Elisseeff have proved that, MB is the optimal solution for single-label feature selection problem under the faithfulness condition \cite{pellet2008using}, and the strongly relevant features are included in its MB set in terms of Kohavi-John feature relevance \cite{tsamardinos2003towards}. Thus, on each label, the independence property of MB indicates that the variable subset contains all of the predictive information about each corresponding label, while the minimality of it can guarantee the minimal redundancy in the variable set. As previously mentioned, in multi-label data, the union of MB sets can not be used directly as the selected feature subset due to the redundancy between MBs of different labels. While CLCD can be used to identify and select the common features simultaneously containing predictive information about several labels as many as possible to minimize the redundancy in the selected feature subset, which is just what Theorem \ref{theo_under_hypo} does, i.e., replacing the MB subset $\textbf{\emph{Z}}_i$ of multiple labels with an common equivalent feature subset $\textbf{\emph{Z}}$ and keeping the information constant.

It is worth mentioning that, using CLCD to directly search the common features of several labels might import some labels into the feature subset due to the case violating the Label-causality Hypothesis. Since labels are usually undetermined and thus can neither used as a factor to infer another label nor a feature to model an predictive learner (or classifier) in most of real-world multi-label applications, it is necessary to remove them and find other predictive features. However, it does not mean that we can directly learn the local causal structure in the variable set $\textbf{\emph{U}}$ seeing that if a pair of labels are the direct cause or effect of each other, the other direct causes of the effect will be ignored when searching MB in $\textbf{\emph{U}}$. To search the substitutes, we can remove the labels in the discovered MB set and continue to search the variables containing similar causal information until no label is included in the MB. Now, we present the CLCD-driven multi-label Feature Selection algorithm (CLCD-FS) in Algorithm \ref{algccfdfs}.

% \\ \ \ \
\begin{algorithm}[t]
        \caption{The \emph{CLCD-FS} Algorithm.}
        \label{algccfdfs}
        \begin{algorithmic}[1]
            \STATE{\textbf{Input:}} Label set $\textbf{\emph{T}}$ and features set $\textbf{\emph{U}}$; A divide-and-conquer-based MB discovery algorithm $\mathbb{A}$ with significance level $\alpha$.
            \STATE CLCD (Phase 1, Phase 2)
            \STATE\textbf{for} each $T\in\textbf{\emph{T}}$ \textbf{do}
            \STATE\ \ \ \textbf{repeat}
            \STATE\ \ \ \ \ \ $\textbf{\emph{PC}}_T\leftarrow \textbf{\emph{PC}}_T-\textbf{\emph{T}}$
            \STATE\ \ \ \ \ \ $\textbf{\emph{PC}}_T\leftarrow \textbf{\emph{PC}}_T\cup \{X| X\in \bigcup_{T_i\in\textbf{\emph{PC}}_T\cap\textbf{\emph{T}}}\textbf{\emph{PC}}_i-(\textbf{\emph{PC}}_T\cap$ \\ \ \ \ \ \ \ $\textbf{\emph{T}})$ and $X\notperp T|\textbf{\emph{Z}}$ for $\forall \textbf{\emph{Z}}\subset\textbf{\emph{PC}}_T\}$
            \STATE\ \ \ \textbf{until} $\textbf{\emph{PC}}_T\cap\textbf{\emph{T}}=\varnothing$
            \STATE\textbf{end for}
            \STATE CLCD (Phase 3: Lines 12 - 18)
            \STATE\textbf{repeat}
            \STATE\ \ \ Select $\textbf{\emph{Z}}$ to $\textbf{\emph{CF}}$ where $\Theta_{\textbf{\emph{T}}_S}(\textbf{\emph{Z}})=1$ for the most large-\\ \ \ \ scale $|\textbf{\emph{T}}_S|$ ($\textbf{\emph{T}}_S\subset \textbf{\emph{T}}$).
            \STATE\ \ \ $\textbf{\emph{PC}}_T\leftarrow \textbf{\emph{PC}}_T-\textbf{\emph{Z}}_T$, $\textbf{\emph{SP}}_T\leftarrow \textbf{\emph{SP}}_T-\textbf{\emph{Z}}_T$ for each $T$.
            \STATE\textbf{until} for $\forall \textbf{\emph{Z}}$, $\Theta_{\textbf{\emph{T}}_S}(\textbf{\emph{Z}})\neq 1$ for all $|\textbf{\emph{T}}_S|>1$.
            \STATE{\textbf{Output:}} Common features $\textbf{\emph{CF}}$, and label-specific features $\textbf{\emph{PC}}_T\cup\textbf{\emph{SP}}_T$ for each $T$.
        \end{algorithmic}
    \end{algorithm}

CLCD-FS finds an MB set of each label to detect the equivalent features for each label, so it inherits the Phase 1 and Phase 2 of CLCD. We explain the additional components of Algorithm \ref{algccfdfs} below:

(1) Lines 3-8: Find the predictive common features shielded by label causality. Through removing the labels from the current $\textbf{\emph{PC}}_T$ set, the information loss about label needs to be supplied with the features in $\textbf{\emph{PC}}$ of each removed label. Thus, in Line 6, $X$ is traversed from $\bigcup_{T_i\in\textbf{\emph{PC}}_T\cap\textbf{\emph{T}}}\textbf{\emph{PC}}_i-(\textbf{\emph{PC}}_T\cap\textbf{\emph{T}})$. Since the $X$ could be a label, Lines 5-6 might be iterated several times.

(2) Lines 10-13: Search the common features and label-specific features. To minimize the redundancy as previously discussed, Line 11 finds the common feature subset $\textbf{\emph{Z}}$ containing information about as many labels as possible, which satisfies the three rules in Eq. (\ref{aleq1}). At the same time, CLCD-FS can record the relationships between selected features and each label, i.e., ``which labels does a selected feature in $\textbf{\emph{Z}}$ relate to". Then, the corresponding $\textbf{\emph{Z}}_T$ needs to be removed from the $\textbf{\emph{PC}}_T$ or $\textbf{\emph{SP}}_T$ to guarantee no redundancy about the same label. The above process is iterated until there are no features containing information about multiple labels ($|\textbf{\emph{T}}_S|>1$). The remaining features in $\textbf{\emph{PC}}_T$ and $\textbf{\emph{SP}}_T$ are label-specific features of their corresponding label $T$.

Compared with traditional multi-label feature selection algorithms, the superiority of CLCD-FS is reflected in three aspects: (1) Interpretability: CLCD-FS not only selects predictive features but also interprets which labels a select feature influences, i.e., identifies the common features and label-specific features; (2) Practicability: CLCD-FS automatically determines the number of selected features without training an additional classifier to achieve the optimal accuracy; (3) Theoretical Reliability: It can be proved that CLCD-FS achieves maximum relevance and minimal redundancy.

In the following subsection, we will give the theoretical analyses of relevance and redundancy.

\subsection{Analyses of Relevance and Redundancy}
\label{sec_fsalg2}

\subsubsection{\textbf{Relevance}}

We prove that using $\textbf{\emph{Z}}$ to replace the $\textbf{\emph{Z}}_i$ for each $T_i\in\textbf{\emph{T}}$, the obtained feature subset $\bigcup_{T_i\in\textbf{\emph{T}}} (\textbf{\emph{MB}}_i-\textbf{\emph{Z}}_i)\cup\textbf{\emph{Z}}-\textbf{\emph{T}}$\footnote{We use $\bigcup_{T_i\in\textbf{\emph{T}}} (\textbf{\emph{MB}}_i-\textbf{\emph{Z}}_i)\cup\textbf{\emph{Z}}-\textbf{\emph{T}}$ instead of $\bigcup_{T_i\in\textbf{\emph{T}}} (\textbf{\emph{MB}}_i-\textbf{\emph{Z}}_i)\cup\textbf{\emph{Z}}$ as in Theorem \ref{theo_under_hypo} since any label cannot be used to predict another label in the feature selection problem.} contains the same information as \textbf{\emph{U}} about $\textbf{\emph{T}}$. Mathematically in other words, all features excluded by $\bigcup_{T_i\in\textbf{\emph{T}}} (\textbf{\emph{MB}}_i-\textbf{\emph{Z}}_i)\cup\textbf{\emph{Z}}-\textbf{\emph{T}}$ are independent of $\textbf{\emph{T}}$ conditioned on $\bigcup_{T_i\in\textbf{\emph{T}}} (\textbf{\emph{MB}}_i-\textbf{\emph{Z}}_i)\cup\textbf{\emph{Z}}-\textbf{\emph{T}}$. It is sufficient to prove the case with $\textbf{\emph{T}}=\{T_i,T_j\}$ since any multi-label case is a direct consequence of two-label case using induction on the number of variables involved in $\textbf{\emph{T}}$. According to Eq. (\ref{tuidao1}), $\textbf{\emph{MB}}_i-\textbf{\emph{Z}}_i\cup\textbf{\emph{Z}}$ is a Markov blanket of $T_i$, denoted as $\textbf{\emph{M}}_i$. Thus,
\begin{equation}
\label{releaseeq3}
T_i\perp\textbf{\emph{U}}-\textbf{\emph{M}}_i\cup \{T_j\}|\textbf{\emph{M}}_i.
\end{equation}
Decompose the $\textbf{\emph{U}}-\textbf{\emph{M}}_i\cup \{T_j\}$ in Eq. (\ref{releaseeq3}) as:
\begin{equation}
\label{releaseeq4}
\begin{split}
\textbf{\emph{U}}-\textbf{\emph{M}}_i\cup \{T_j\}= &(\textbf{\emph{U}}-\textbf{\emph{M}}_i-\textbf{\emph{M}}_j)\cup \\
&(\textbf{\emph{M}}_j-\textbf{\emph{M}}_i-\{T_i\}\cup\{T_j\}).
\end{split}
\end{equation}
According to the Weak union property in Theorem \ref{common_property}, we have:
\begin{equation}
\label{releaseeq5}
\begin{split}
T_i\perp(\textbf{\emph{U}}-\textbf{\emph{M}}_i-\textbf{\emph{M}}_j)|(\textbf{\emph{M}}_j-\{T_i\}\cup\{T_j\}).
\end{split}
\end{equation}
Due to the symmetry between $T_i$ and $T_j$, a similar relationship will exist:
\begin{equation}
\label{releaseeq6}
\begin{split}
T_j\perp(\textbf{\emph{U}}-\textbf{\emph{M}}_i-\textbf{\emph{M}}_j)|(\textbf{\emph{M}}_i-\{T_j\}\cup\{T_i\}).
\end{split}
\end{equation}
According to Theorem \ref{common_property}, we combine the Eq. (\ref{releaseeq5}) and Eq. (\ref{releaseeq6}), and obtain:
\begin{equation}
\label{releaseeq7}
\begin{split}
&\textbf{\emph{U}}-(\textbf{\emph{M}}_i\cup\textbf{\emph{M}}_j-\{T_i,T_j\})\perp\{T_i,T_j\} \\
&|\textbf{\emph{M}}_i\cup\textbf{\emph{M}}_j-\{T_i,T_j\}.
\end{split}
\end{equation}
Thus, $I(\textbf{\emph{T}},\textbf{\emph{U}})=I(\textbf{\emph{T}},\textbf{\emph{M}}_i\cup\textbf{\emph{M}}_j-\{T_i,T_j\})$, which means that the selected feature subset of CLCD-FS contains all information about the labels and achieves the maximum relevance among the subsets of feature sets.

\subsubsection{\textbf{Redundancy}}

We continue to prove under the case $\textbf{\emph{T}}=\{T_i,T_j\}$. Assume that there exists a subset of $\textbf{\emph{S}}\subset\bigcup_{T_i\in\textbf{\emph{T}}} (\textbf{\emph{MB}}_i-\textbf{\emph{Z}}_i)\cup\textbf{\emph{Z}}-\textbf{\emph{T}}$ such that it also contain the same information as \textbf{\emph{U}} about $\textbf{\emph{T}}$, then:
\begin{equation}
\label{redundancy1}
T_i \perp \textbf{\emph{U}}-\textbf{\emph{S}}|\textbf{\emph{S}}
\end{equation}
We construct a subset of $\textbf{\emph{M}}_i$, $\textbf{\emph{A}}=\textbf{\emph{M}}_i\cap\{T_j\}\cup\textbf{\emph{S}}$, to assist the analysis. Obviously, $\textbf{\emph{M}}_i$ can be written as the union of two sets $(\textbf{\emph{M}}_i-\textbf{\emph{A}})\cup(\textbf{\emph{M}}_i\cap\textbf{\emph{A}})$. Thus, we have:
\begin{equation}
\label{redundancy2}
T_i \perp \textbf{\emph{U}}-\textbf{\emph{M}}_i-\{T_i\} | (\textbf{\emph{M}}_i-\textbf{\emph{A}})\cup(\textbf{\emph{M}}_i\cap\textbf{\emph{A}}).
\end{equation}
According to Eq. (\ref{redundancy1}), extend the $\textbf{\emph{S}}$ as a more large-scale Mb $\textbf{\emph{S}}\cup \{T_j\} \cup (\textbf{\emph{U}}-\textbf{\emph{M}}_i-\{T_i\})$, which is equivalent with $(\textbf{\emph{M}}_i\cap\textbf{\emph{A}}) \cup (\textbf{\emph{U}}-\textbf{\emph{M}}_i-\{T_i\})$. Then, we have
\begin{equation}
\label{redundancy3}
T_i \perp \textbf{\emph{M}}_i-\textbf{\emph{A}} | (\textbf{\emph{M}}_i\cap\textbf{\emph{A}}) \cup (\textbf{\emph{U}}-\textbf{\emph{M}}_i-\{T_i\}).
\end{equation}
If the Intersection property in Theorem \ref{common_property} is satisfied here, then Eq. (\ref{redundancy2}) and Eq. (\ref{redundancy3}) indicate that:
\begin{equation}
\begin{split}
\label{redundancy4}
& T_i \perp (\textbf{\emph{M}}_i-\textbf{\emph{A}}) \cup (\textbf{\emph{U}}-\textbf{\emph{M}}_i-\{T_i\}) | \textbf{\emph{M}}_i\cap\textbf{\emph{A}} \\
\Rightarrow \  & T_i \perp \textbf{\emph{U}}-(\textbf{\emph{M}}_i\cap\textbf{\emph{A}})-\{T_i\} | \textbf{\emph{M}}_i\cap\textbf{\emph{A}}.
\end{split}
\end{equation}
Thus, $\textbf{\emph{M}}_i\cap\textbf{\emph{A}}$ is an Mb of $T_i$. However, $\textbf{\emph{M}}_i$ is an MB of $T_i$, thus, $\textbf{\emph{M}}_i\cap\textbf{\emph{A}}=\textbf{\emph{M}}_i$. Also, $\textbf{\emph{M}}_i\subset\textbf{\emph{M}}_i\cap\textbf{\emph{A}}$, i.e., $(\textbf{\emph{M}}_i-\{T_j\})\cup(\textbf{\emph{M}}_i\cap\{T_j\}) \subset \textbf{\emph{S}}\cup(\textbf{\emph{M}}_i\cap\{T_j\})$. Hence, $\textbf{\emph{M}}_i-\{T_j\} \subset \textbf{\emph{S}}$. Similarly, $\textbf{\emph{M}}_j-\{T_i\} \subset \textbf{\emph{S}}$. Since $\textbf{\emph{S}}$ is a subset of $(\textbf{\emph{M}}_i-\{T_j\})\cup(\textbf{\emph{M}}_j-\{T_i\})$, the above three equations indicate that $\textbf{\emph{S}}=(\textbf{\emph{M}}_i-\{T_j\})\cup(\textbf{\emph{M}}_j-\{T_i\})$.

In conclusion, if the Intersection property is satisfied, no redundancy exists in the $\bigcup_{T_i\in\textbf{\emph{T}}} (\textbf{\emph{MB}}_i-\textbf{\emph{Z}}_i)\cup\textbf{\emph{Z}}-\textbf{\emph{T}}$. While if the Intersection property is violated for Eq. (\ref{redundancy2}) and Eq. (\ref{redundancy3}), then we can assert that $\textbf{\emph{U}}-\textbf{\emph{M}}_i-\{T_i\}$ and $\textbf{\emph{M}}_i-\textbf{\emph{A}}$ contain equivalent information about $T_i$ and there might exist redundancy in $\bigcup_{T_i\in\textbf{\emph{T}}} (\textbf{\emph{MB}}_i-\textbf{\emph{Z}}_i)\cup\textbf{\emph{Z}}-\textbf{\emph{T}}$. We give an example to explain the redundancy brought by equivalent information. Assume that feature subsets $\{A,B\}$ and $\{C,D\}$ are equally effective to predict label $T_1$ since they contain equivalent information about $T_1$, but only $\{C,D\}$ can be used to predict $T_2$. Then, in $\{A,B,C,D\}$, there exists redundancy between $\{A,B\}$ and $\{C,D\}$, which could be reduced by removing $\{A,B\}$. The proposed CLCD-FS algorithm tries to detect the features containing equivalent information, so the minimal redundancy is guaranteed in the selected feature subsets.
%Actually, the Phase 3 in CCFD just aims to complete the task according to Theorem \ref{theo_under_hypo}. In the following, we will present a multi-label feature selection algorithms based on CCFD algorithm.

\subsection{Time Complexity Analysis}
\label{sec_timeanalysis}

Finally, we provide time complexity analysis as follows. The computational cost of the causality-based algorithms is measured via the number of CI-tests. Let $|*|$ denote the scale of variable set $*$ and $p$ denote the largest scale of the parent-child set of any label. For Phase 1 in CLCD, the time complexity of the MB discovery process of any label is less than $O(2^pp|\emph{\textbf{U}}|)$, and thus the time complexity of Phase 1 is $O(2^pp|\textbf{\emph{U}}||\textbf{\emph{T}}|)$. For Phase 2, there are less than ${\rm C}_{|\textbf{\emph{T}}|}^2$ pairs of labels connecting with each other and the actual operation for each pair is to traverse the pairwise dependence. Thus, the time complexity of Phase 2 is $O(2^p|\textbf{\emph{U}}||\textbf{\emph{T}}|^2)$. Let the scale of child set of labels be $c$ and the largest scale of $\textbf{\emph{Z}}$ in Phase 3 (Line 13) be $z$, then the computational cost is $O(2^p|\textbf{\emph{U}}|^z(|\textbf{\emph{T}}|+c))$. Normally, if only the pairwise dependencies are considered, $z$ is set to 1, as followed by existing causality-based methods. The extra processes in CLCD-FS possess lower time complexity. Let $m=max\{|\textbf{\emph{T}}|p,|\textbf{\emph{T}}|^2,|\textbf{\emph{T}}|+c\}$, then the time complexity of CLCD and CLCD-FS is $O(2^p|\textbf{\emph{U}}|m)$. For better performance, $z$ could be set higher so that multivariate dependence could be considered. Under these circumstances, increase in running time is actually not obvious. The main reason is that, the test results of large-scale $\textbf{\emph{Z}}$ and small-scale $\textbf{\emph{Z}}$ could be used to derive each other. For example, if $\textbf{\emph{Z}}\perp X$, then any subsets $\textbf{\emph{Z}}'\subset \textbf{\emph{Z}}$ satisfy $\textbf{\emph{Z}}'\perp X$, and the converse proposition could also simplify the computational process.

\section{Experiments}
\label{exsec}

We first verify the effectiveness of CLCD on synthetic data sets with foregone causality in Section \ref{ex1}, by comparing precision, recall and time efficiency. Afterwards, the multi-label feature selection experiments are conducted on real-world data sets in Section \ref{ex2}, to demonstrate the superiority of the proposed CLCD-FS against traditional algorithms and MB-MCF. We further present the relationships between labels and selected features on Emotions data set in Section \ref{ex3}, to demonstrate the interpretability of CLCD-FS.

\subsection{Learn Common and Label-specific Causal Variables: Precision, Recall, and Time Efficiency}
\label{ex1}

In this section, we present an evaluation of CLCD for identification of common and label-specific variables in simulated data. The data sets are sampled from synthetic Bayesian networks with simulation method presented in \cite{statnikov2010tied}. Simulated data allow us to evaluate methods in a controlled setting where the underlying causal process and all causal variables of each label are exactly known. Detailed experiment settings are presented below.

\begin{table}[t]
		\caption{Experiment Parameters}%
		\centering
		\resizebox {2.5in }{!}{\setlength{\tabcolsep}{0.035in}
			\centering
			\begin{tabular}{cc}
				\hline
				\ \ \ \ Parameters \ \ \ \ & \ \ Settings \ \  \\ \hline
				\rule{0pt}{0.28cm}The number of labels & 50   \\ \hline
				\rule{0pt}{0.28cm}\ \ \ \ The number of non-label variables \ \ \ \ & 1000   \\ \hline
				\rule{0pt}{0.28cm}The number of training samples & 5000 \\ \hline
				\rule{0pt}{0.28cm}The number of MBs of each label & $\in$[1,15] \\ \hline
                \rule{0pt}{0.28cm}The size of an MB of each label & $\in$[5,15] \\ \hline
				\end{tabular}
		}
		\label{app_tb1}
	\end{table}

\begin{table*}[t]
		\caption{Average precision and recall of searched common and label-specific variables with respect to the percentage of the labels that have multiple MBs.}%
		\centering
		\resizebox {5.5in }{!}{\setlength{\tabcolsep}{0.1in}
	\begin{tabular}{ccccccccc}
		\hline
		Metric & $p_c$ & $p_m$ & $\cap$IAMB & $\cap$HITON-MB & $\cap$CCMB & $\cap$KIAMB & $\cap$TIE* & CLCD \\
		\hline
\hline
		\multirow{3}*{Precision} & \multirow{3}*{$p_c = 0.5$} & $p_m = 0$ & 0.745 & 0.919 & 0.792  & 0.415 & 0.746 & \textbf{0.915}\\
		\cline{3-9}
		%~ & ~ & $p_m = 0.2$ & ~ & ~ & ~  & ~ & ~ & ~\\
		%\cline{3-9}
		~ & ~ & $p_m = 0.5$ & 0.413 & 0.579 & 0.567  & 0.612 & 0.759 & \textbf{0.909}\\
		\cline{3-9}
		~ & ~ & $p_m = 1$ & 0.192 & 0.315 & 0.287  & 0.697 & 0.787 & \textbf{0.906}\\
		\cline{2-9}
        \hline
		\multirow{3}*{Recall} & \multirow{3}*{$p_c = 0.5$} & $p_m = 0$ & 0.659 & 0.958 & 0.979  & 0.625 & 0.912 & \textbf{0.979}\\
		\cline{3-9}
		%~ & ~ & $p_m = 0.2$ & ~ & ~ & ~  & ~ & ~ & ~\\
		%\cline{3-9}
		~ & ~ & $p_m = 0.5$ & 0.216 & 0.305 & 0.312  & 0.679 & 0.915 & \textbf{0.973}\\
		\cline{3-9}
		~ & ~ & $p_m = 1$ & 0.113 & 0.152 & 0.198  & 0.713 & 0.903 & \textbf{0.981}\\
		\cline{2-9}
		%~ &\multicolumn{2}{|c|}{Average Time}& ~ & ~ & ~  & ~ & ~ & ~\\
		\hline
\hline
		\multicolumn{3}{c}{Average Time ($lg(Time)$)}& \textbf{0.473} & 2.295 & 2.874  & 1.629 & 5.672 & 2.871\\
		\hline
\end{tabular}
}
    \label{app_tb_group1}
\end{table*}

\begin{table*}[t]
		\caption{Average precision and recall of searched common and label-specific variables with respect to the percentage of the labels that have direct causal relationships with each other.}%
		\centering
		\resizebox {5.50in }{!}{\setlength{\tabcolsep}{0.1in}
	\begin{tabular}{ccccccccc}
		\hline
		Metric & $p_c$ & $p_m$ & $\cap$IAMB & $\cap$HITON-MB & $\cap$CCMB & $\cap$KIAMB & $\cap$TIE* & CLCD \\
		\hline
\hline
		
		\multirow{3}*{Precision} & $p_c = 0$  & \multirow{3}*{$p_m = 0.5$} & 0.452 & 0.583 & 0.581  & 0.672 & 0.771 & \textbf{0.915}\\
		\cline{4-9}\cline{2-2}
		~ & $p_c = 0.5$ & ~ & 0.413 & 0.579 & 0.567  & 0.612 & 0.759 & \textbf{0.909}\\
		\cline{4-9}\cline{2-2}
		~ & $p_c = 1$ & ~ & 0.394 & 0.560 & 0.551 & 0.654  & 0.715 & \textbf{0.910} \\
		\cline{1-9}
        \multirow{3}*{Recall} & $p_c = 0$  & \multirow{3}*{$p_m = 0.5$} & 0.237 & 0.325 & 0.346  & 0.631 & 0.923 & \textbf{0.970}\\
		\cline{4-9}\cline{2-2}
		~ & $p_c = 0.5$ & ~ & 0.216 & 0.305 & 0.312  & 0.679 & 0.915 & \textbf{0.973}\\
		\cline{4-9}\cline{2-2}
		~ & $p_c = 1$ & ~ & 0.191 & 0.286 & 0.307  & 0.677 & 0.877 & \textbf{0.965}\\
		\cline{1-9}
\hline
\hline
		\multicolumn{3}{c}{Average Time ($lg(Time)$)}& \textbf{0.462} & 2.131 & 2.559  & 1.503 & 5.379 & 2.812\\
		\hline
	\end{tabular}
}
    \label{app_tb_group2}
\end{table*}

\textbf{Experiment parameters on synthetic data}: To validate CLCD and corresponding theory in this paper, each data set is setup with different controlled parameters: (1) percentage of the labels that have direct causal relationships with each other ($p_c$); (2) percentage of the labels that have multiple MBs ($p_m$). The remaining settings to simulate a Bayesian network, are same in all experiment groups, which are given in the Table \ref{app_tb1}. For each label, we randomly choose 5-10 non-label variables and labels (their proportions are determined by $p_c$) as the MB. Among these labels, $p_m$ of them have 5-10 equivalent MBs, which are induced by the probability distribution with equivalent information. Specifically, if variables $X$ and $Y$ contain equivalent information about $T$, then (a) for each combination of values of $X$ and $T$ such that $P(T = t |X = x) = p$, there exists a value $y$ of variable $Y$ such that $P(T = t | Y = y) = p$, and (b) for every combination of values of $Y$ and $T$ such that $P(T = t | Y = y) = p$, there exists a value $x$ of variable $X$ such that $P(T = t | X = x) = p$.

\textbf{Comparing algorithms}\footnote{Codes are collected in: http://home.ustc.edu.cn/$\sim$xingyuwu/Cau salFS.zip}: Since there are no algorithms for identification of common and label-specific causal variables, we deploy existing MB discovery algorithms to search the causal variables for different labels first and then take the intersection of MB sets of different labels as the common causal variables, and the remaining variables as the label-specific causal variables. Among extensive causal variable learning algorithms, we choose several representative algorithms from each type, including three single MB discovery algorithms (a simultaneous MB learning algorithm IAMB \cite{iamb}, two divide-and-conquer MB learning algorithms HITON-MB \cite{hiton} and CCMB \cite{ccmb}) and two multiple MB discovery algorithms (KIAMB \cite{pcmb} and TIE* \cite{statnikov2013algorithms}). The characteristics of these types of algorithms are detailed in Section \ref{relatedwork}. The value of $k$ in KIAMB is set to 10 (the average number of MBs). The MB discovery algorithm in CLCD is HITON-MB \cite{hiton} and the parameter in its $G^2$-test \cite{pearl1988} is set to 0.05.

\textbf{Metrics for Evaluation}: The frequently used metrics $Precision$ and $Recall$ are adapted to measure the accuracy of the searched common and label-specific causal variables. $Precision$ is the fraction of retrieved true positives over the total amount of retrieved variables, and $Recall$ is the fraction of retrieved true positives over the total amount of true positives. Mathematically,
\begin{equation}
\label{precision}
Precision=\frac{TP}{TP+FP}.
\end{equation}
\begin{equation}
\label{recall}
Recall=\frac{TP}{TP+FN}.
\end{equation}
where $TP$, $FP$ and $FN$ denote the number of true positives, false positives and false negatives, respectively.
The two metrics are calculated on each type of causal variables, and the average results of the two types are taken as the performance. The results are shown in the Table \ref{app_tb_group1} and Table \ref{app_tb_group2}. Furthermore, the logarithmic CPU time is recorded to compute the time efficiency.

\textbf{Performance Comparison}: Table \ref{app_tb_group1} and Table \ref{app_tb_group2} provide the average precision and recall of identification of common and label-specific causal variables. Each group keeps one of the $p_c$ and $p_m$ invariant and changes the other, to show the performance of CLCD and other comparing algorithms in different cases. We conclude from these results that CLCD constantly performs better than others in the cases satisfying ($p_c=0$) and violating ($p_c\neq0$) the Label-causality Hypothesis. Specifically for different comparing algorithms: (1) Single MB discovery algorithms IAMB, HITON-MB and CCMB achieve lower recall but relatively higher precision with $p_m>0$, which means that they fail to identify the two types of variables in the data sets due to inability to solve the case with multiple MBs for some labels. (2) KIAMB uses randomized strategy to discover multiple MBs, and thus it has unstable performance on both precision and recall, and can only capture part of these two types of variables although it is efficient. (3) TIE* has better precision and recall compared with other algorithms. However it is computationally expensive. Like CLCD, TIE* also considers the common causal variables with equivalent information, whereas it tries to retrieve all MBs of each label at first step, resulting in high computation time and statically low reliability. Therefore, CLCD can be considered as the first algorithm targeting to distinguish between common and label specific variables with reasonable time complexity.

\subsection{Multi-label Feature Selection: Accuracy}
\label{ex2}

To demonstrate the performance of the extended CLCD-FS for multi-label feature selection problem, in this subsection, five state-of-the-art multi-label feature selection algorithms are compared with four frequently-used metrics. Details of these experiments are given as follows.

\textbf{Multi-label Data sets}\footnote{Data Source: http://mulan.sourceforge.net/datasets-mlc.html}: The six data sets are taken from diverse application domains. The domains and standard statistics are provided in Table \ref{datasets}. $Cardinality$ denotes the average number of labels for per instance, and $density$ normalizes the label cardinality by the number of labels.

\textbf{Comparing algorithms}\footnote{Codes are collected in: http://home.ustc.edu.cn/$\sim$xingyuwu/Trad itional-Multi-label-FS.zip}: To validate the performance of CLCD-FS, six state-of-the-art multi-label feature selection algorithms are compared, including SFUS \cite{sfus}, CSFS \cite{csfs}, MIFS \cite{mifs}, CMFS \cite{cmfs}, MCLS \cite{mcls}, and the previously proposed MB-MCF \cite{wu2020multi}. These recently proposed algorithms reflect the effectiveness of multi-label feature selection from different perspectives (or metrics). To evaluate the effectiveness of proposed methods, we use the binary classifier SVM cooperating with the multi-label classification model BR \cite{zhang2007ml} to decompose a multi-label problem into several independent binary problems first and compute their classification accuracies archived by selected features. The main consideration is that BR does not involve the label correlations, which could more clearly demonstrate the strengths of these compared algorithms in terms of addressing complex relationships in multi-label data. Additionally, since CLCD-FS and MB-MCF measure the importance of features through uncovering the causal mechanisms rather than calculating the correlations, \textbf{CLCD-FS and MB-MCF do not need to predetermine the number of selected features}, as shown in the Fig. \ref{expic1}.

\begin{figure*}
\centering%
\subfigure[]{ \centering
    \label{Birds_HL}
    \includegraphics[height=1.3in,width=1.7in]{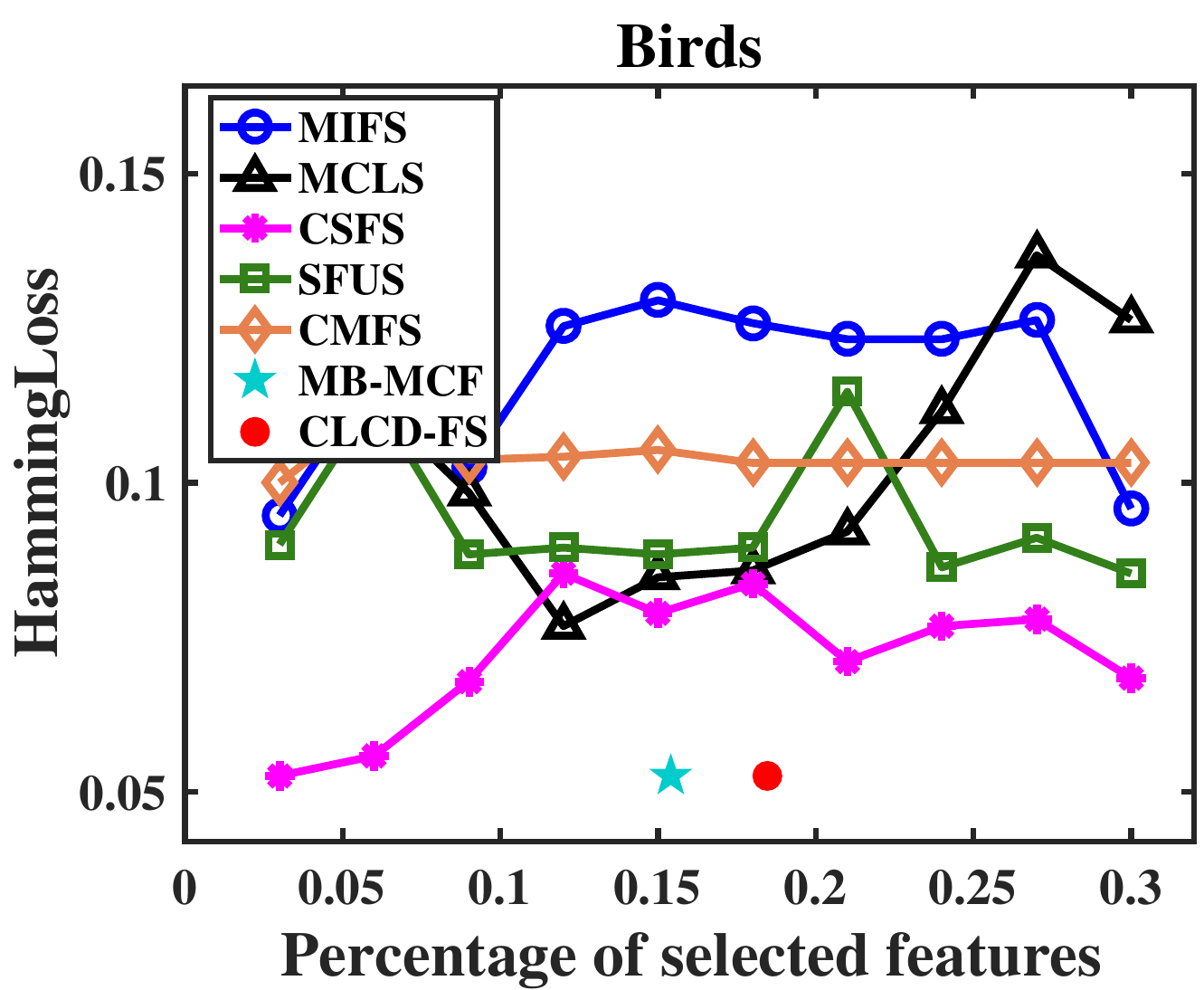}}
\subfigure[]{ \centering
    \label{Birds_RL}
    \includegraphics[height=1.3in,width=1.7in]{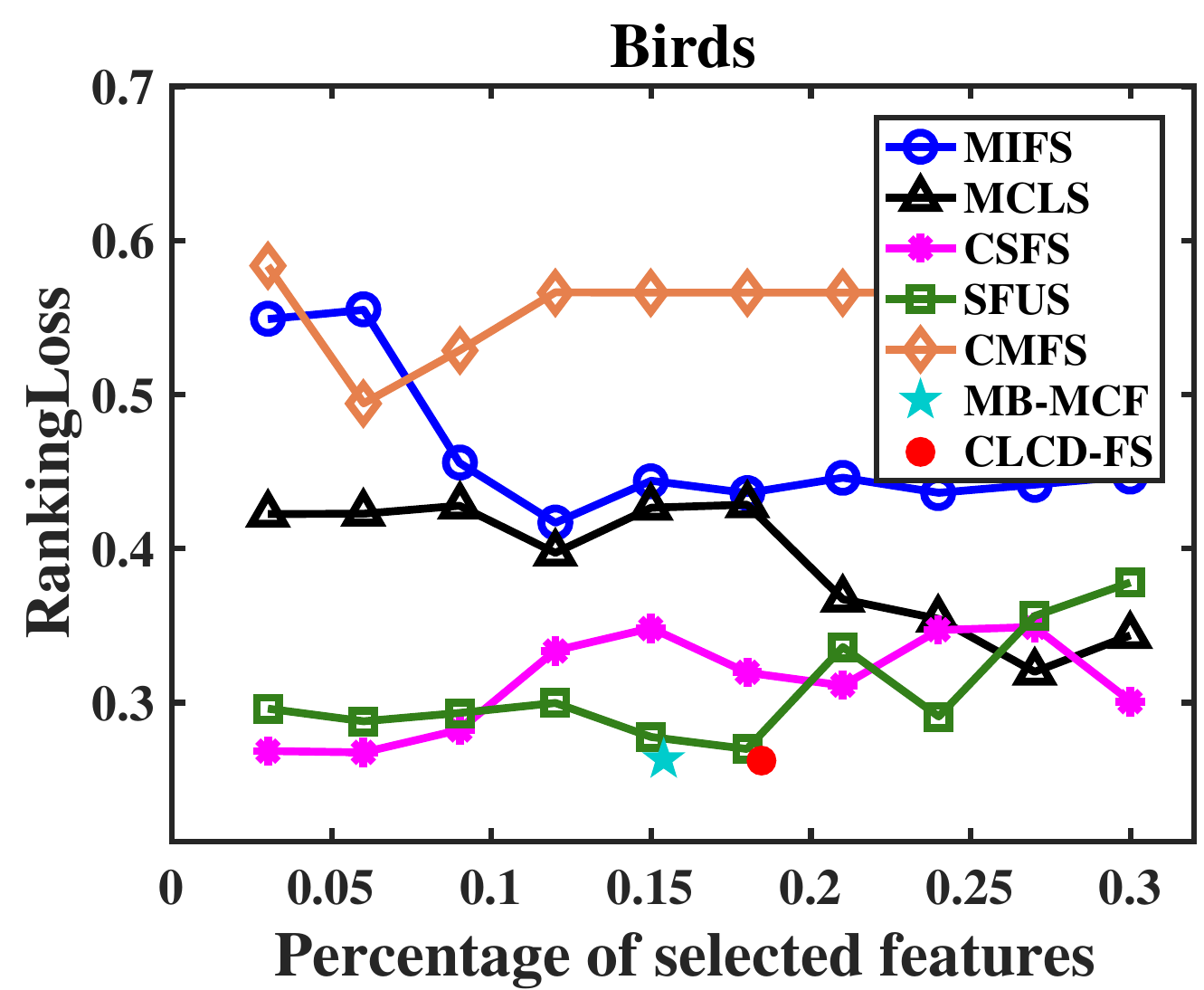}}
\subfigure[]{ \centering
    \label{Birds_Ma}
    \includegraphics[height=1.3in,width=1.7in]{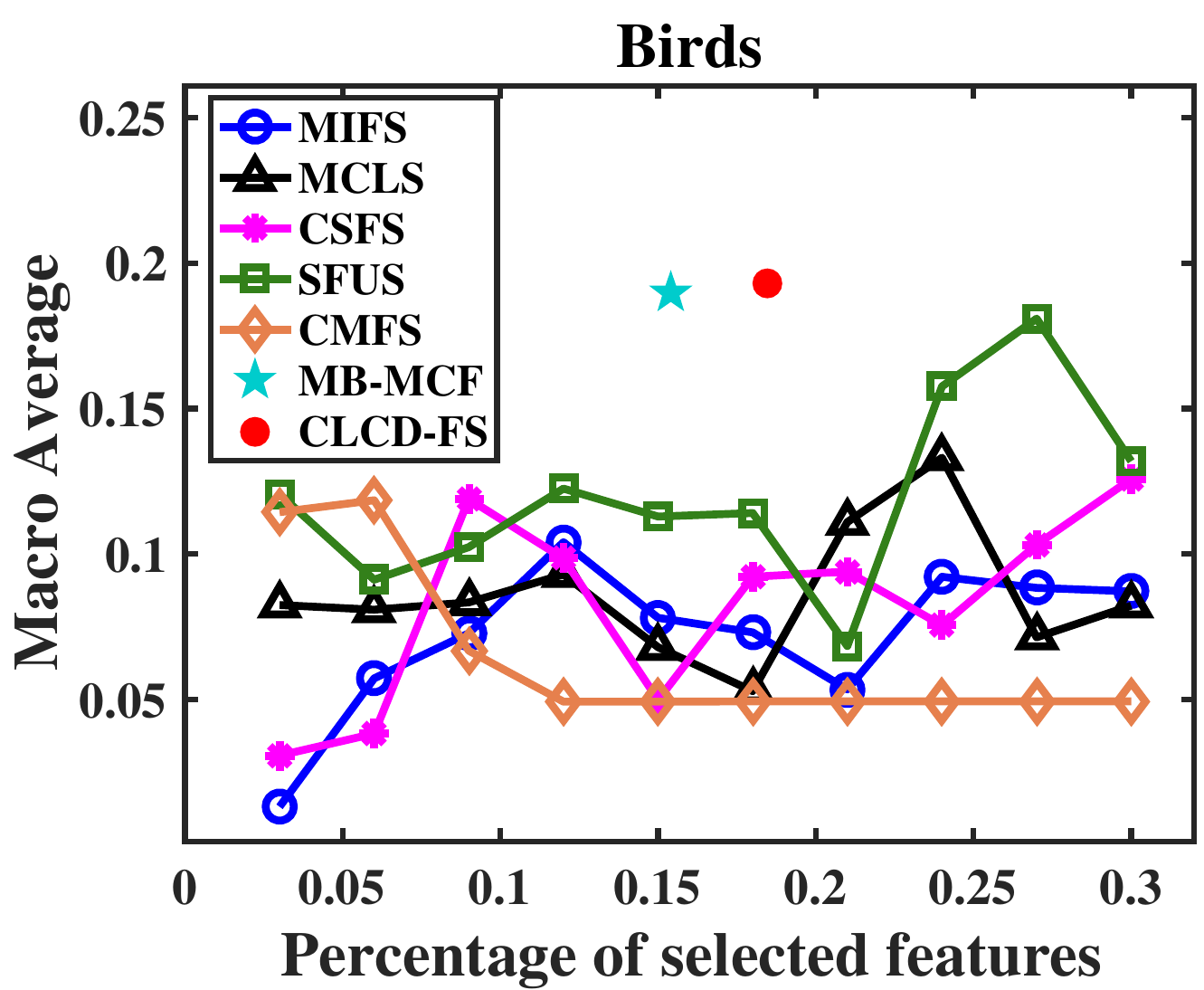}}
\subfigure[]{ \centering
    \label{Birds_Mi}
    \includegraphics[height=1.3in,width=1.7in]{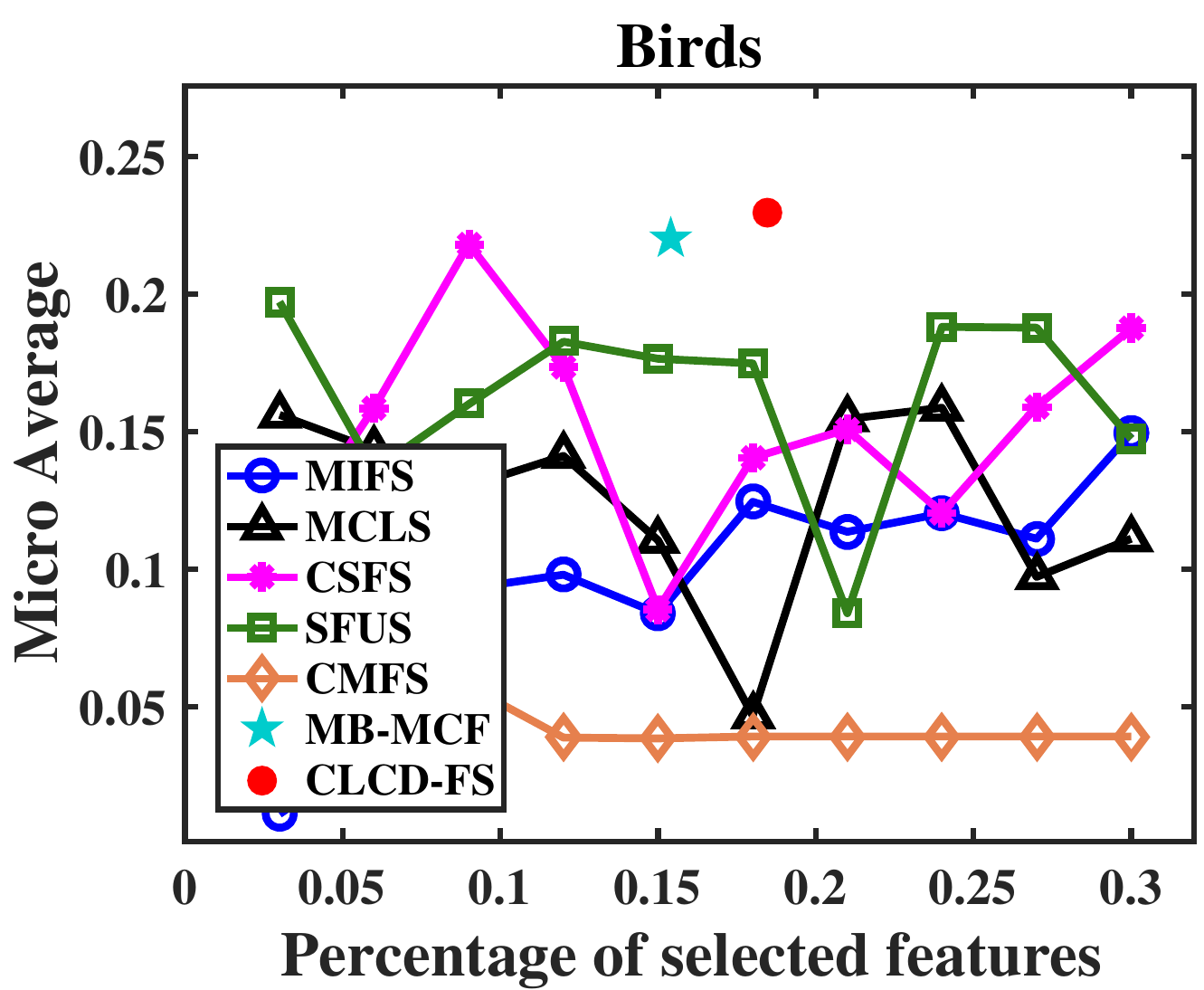}}
\subfigure[]{ \centering
    \label{CAL500_HL}
    \includegraphics[height=1.3in,width=1.7in]{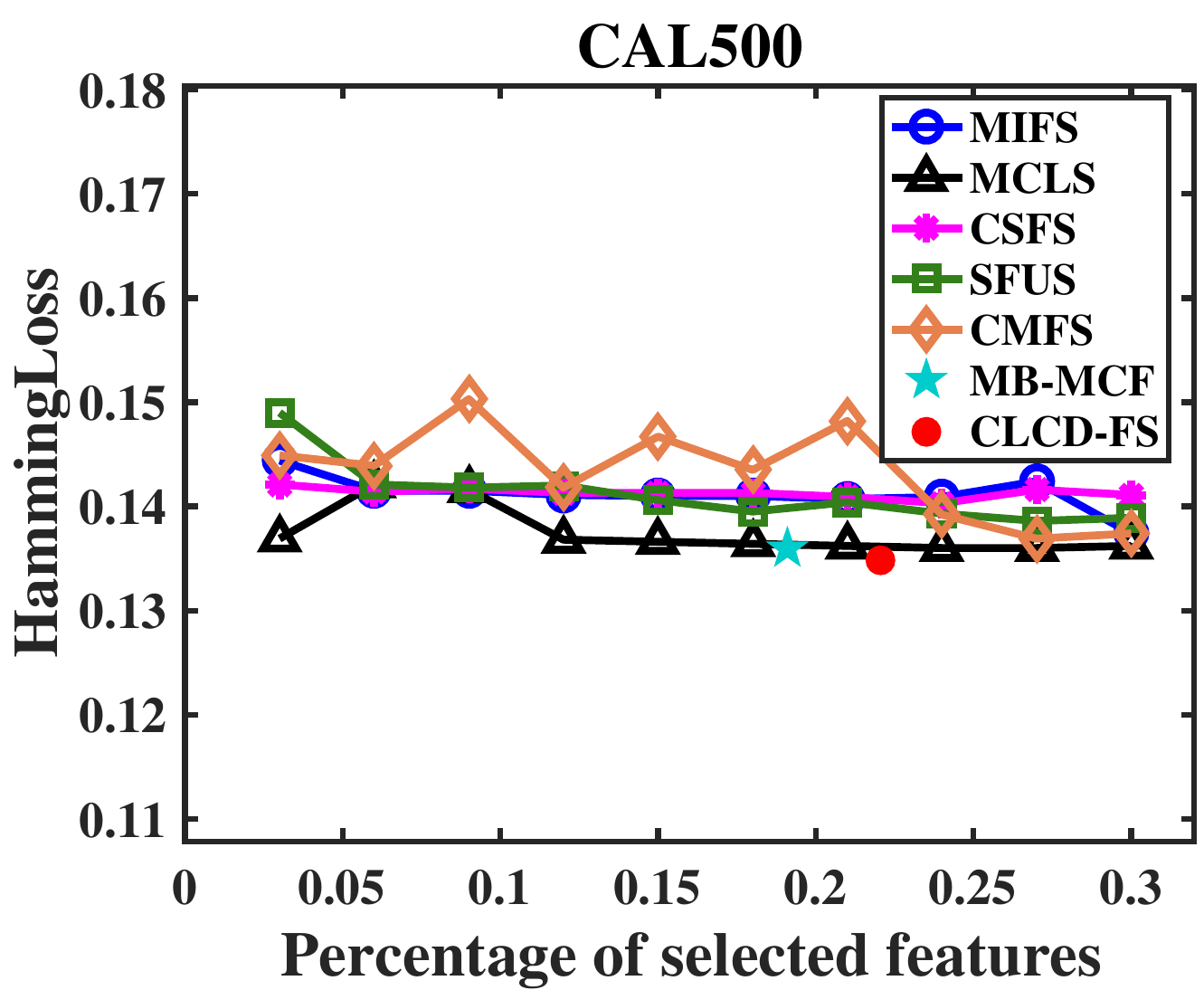}}
\subfigure[]{ \centering
    \label{CAL500_RL}
    \includegraphics[height=1.3in,width=1.7in]{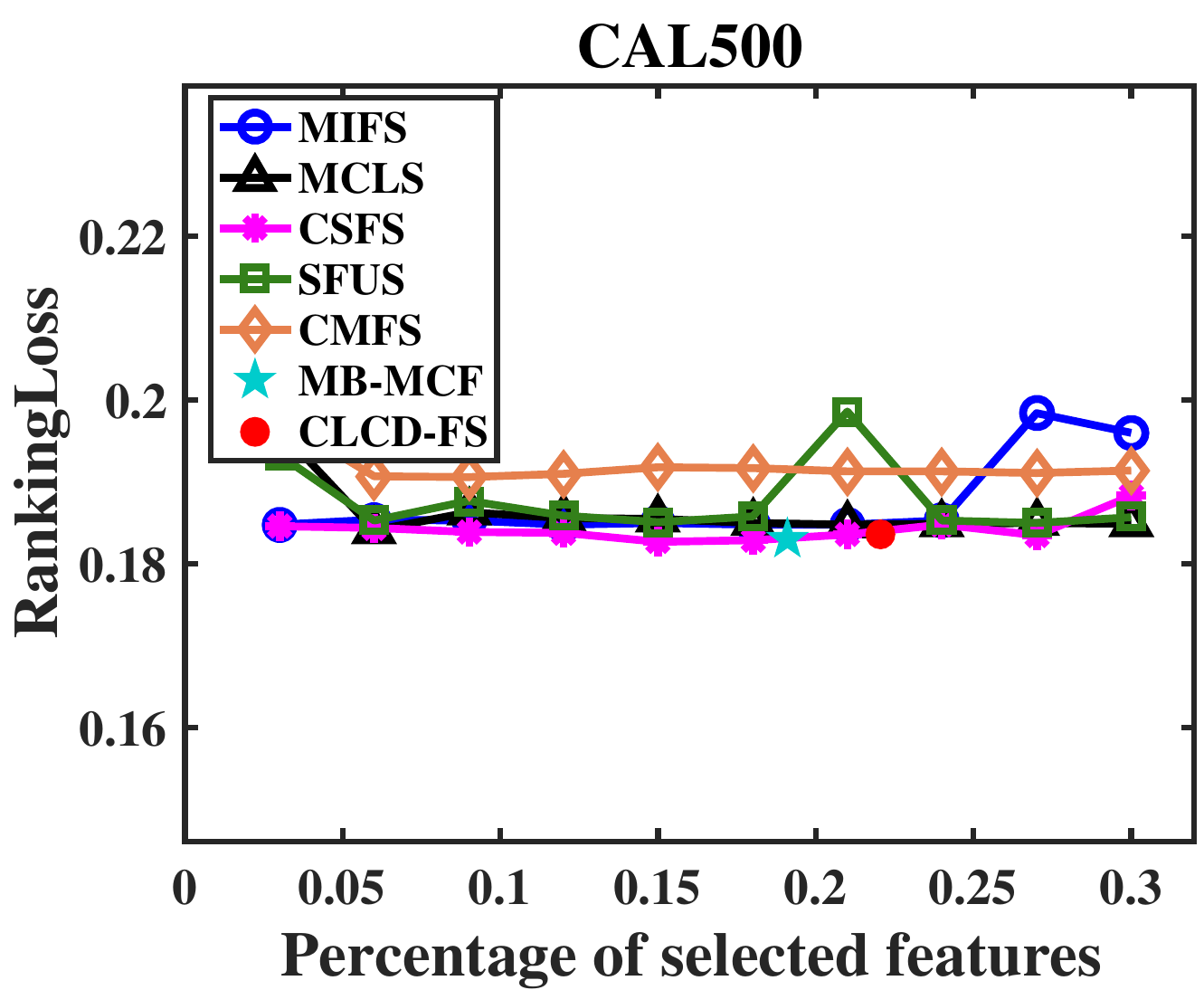}}
\subfigure[]{ \centering
    \label{CAL500_Ma}
    \includegraphics[height=1.3in,width=1.7in]{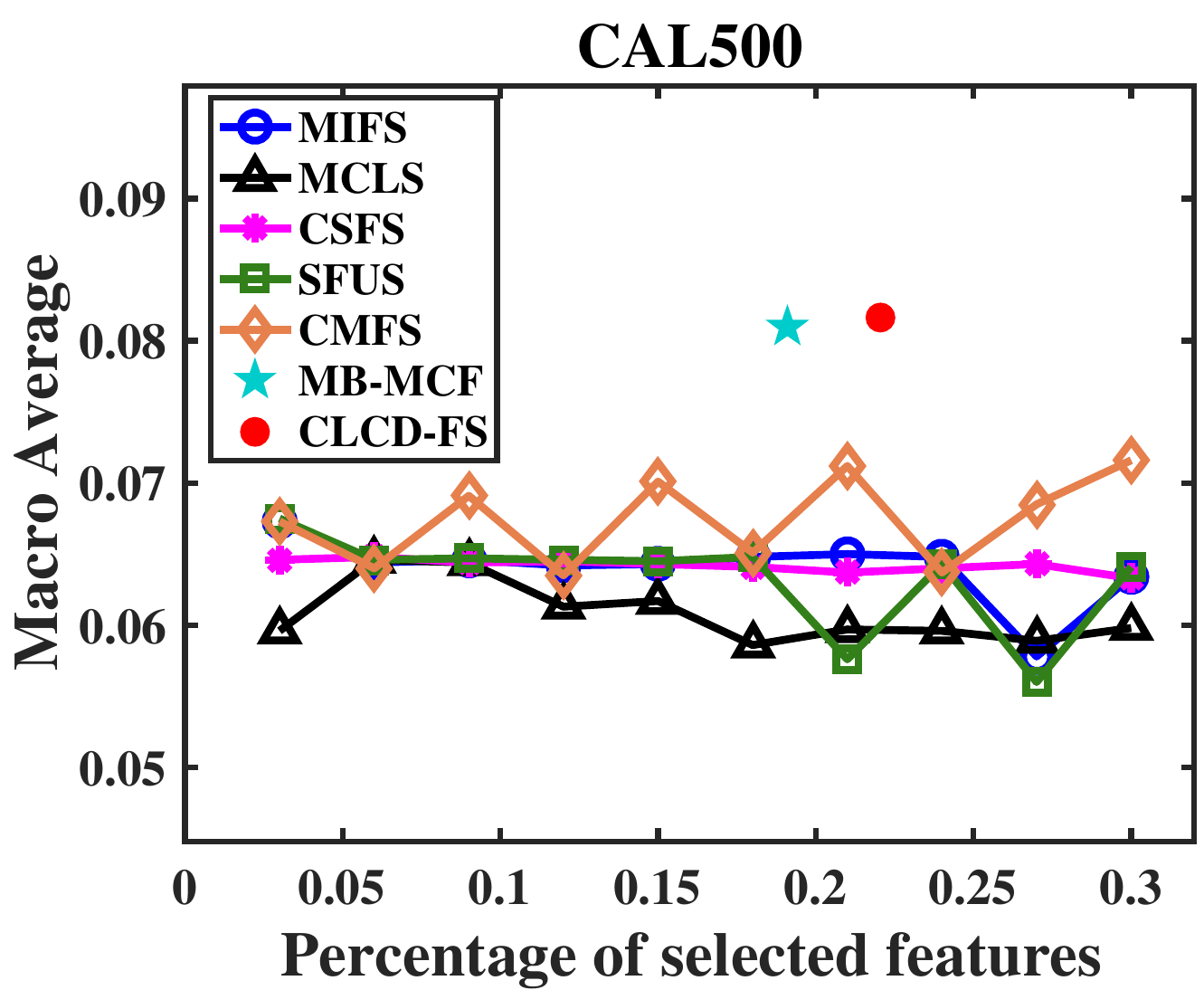}}
\subfigure[]{ \centering
    \label{CAL500_Mi}
    \includegraphics[height=1.3in,width=1.7in]{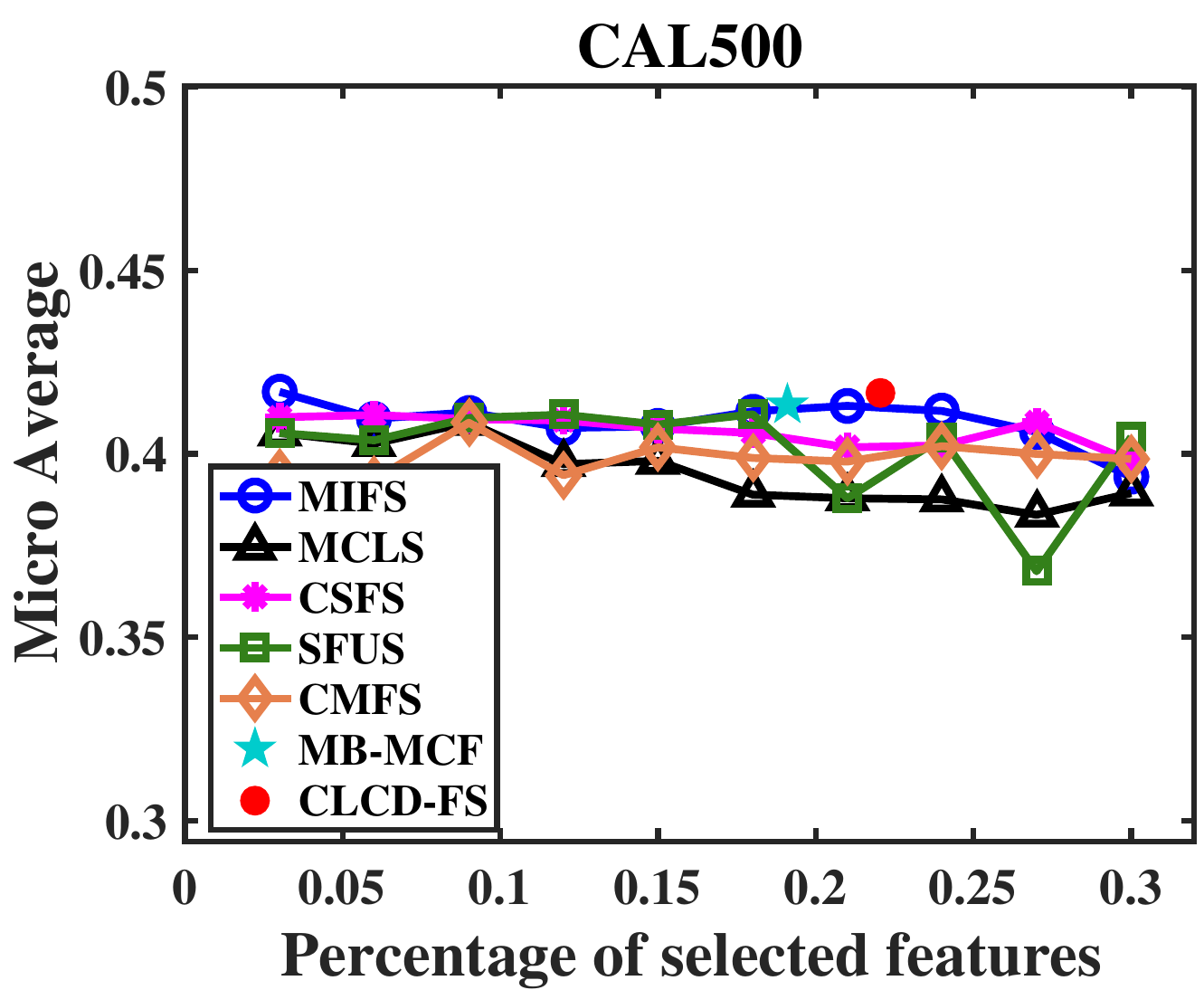}}
\subfigure[]{ \centering
    \label{Emotions_HL}
    \includegraphics[height=1.3in,width=1.7in]{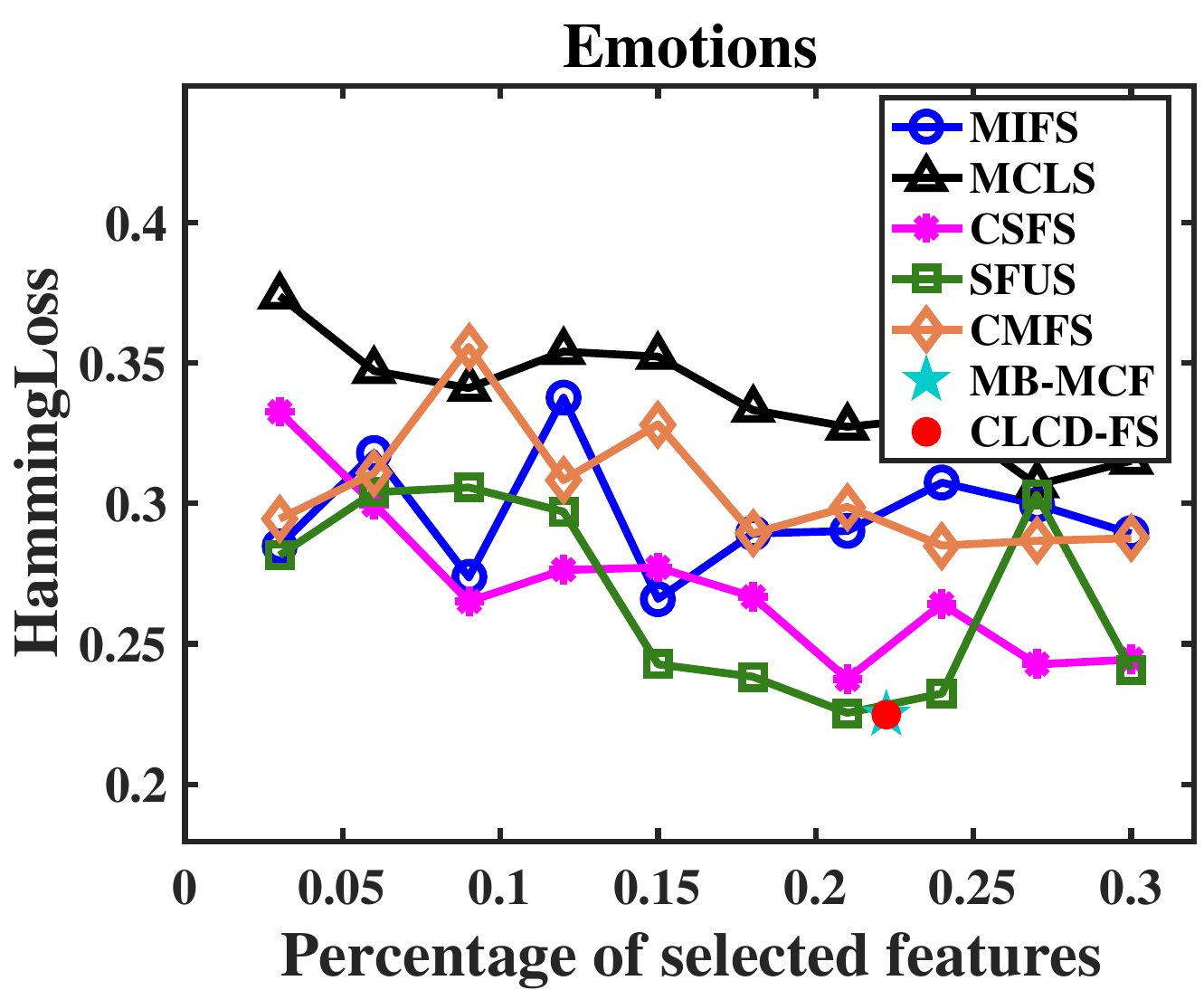}}
\subfigure[]{ \centering
    \label{Emotions_RL}
    \includegraphics[height=1.3in,width=1.7in]{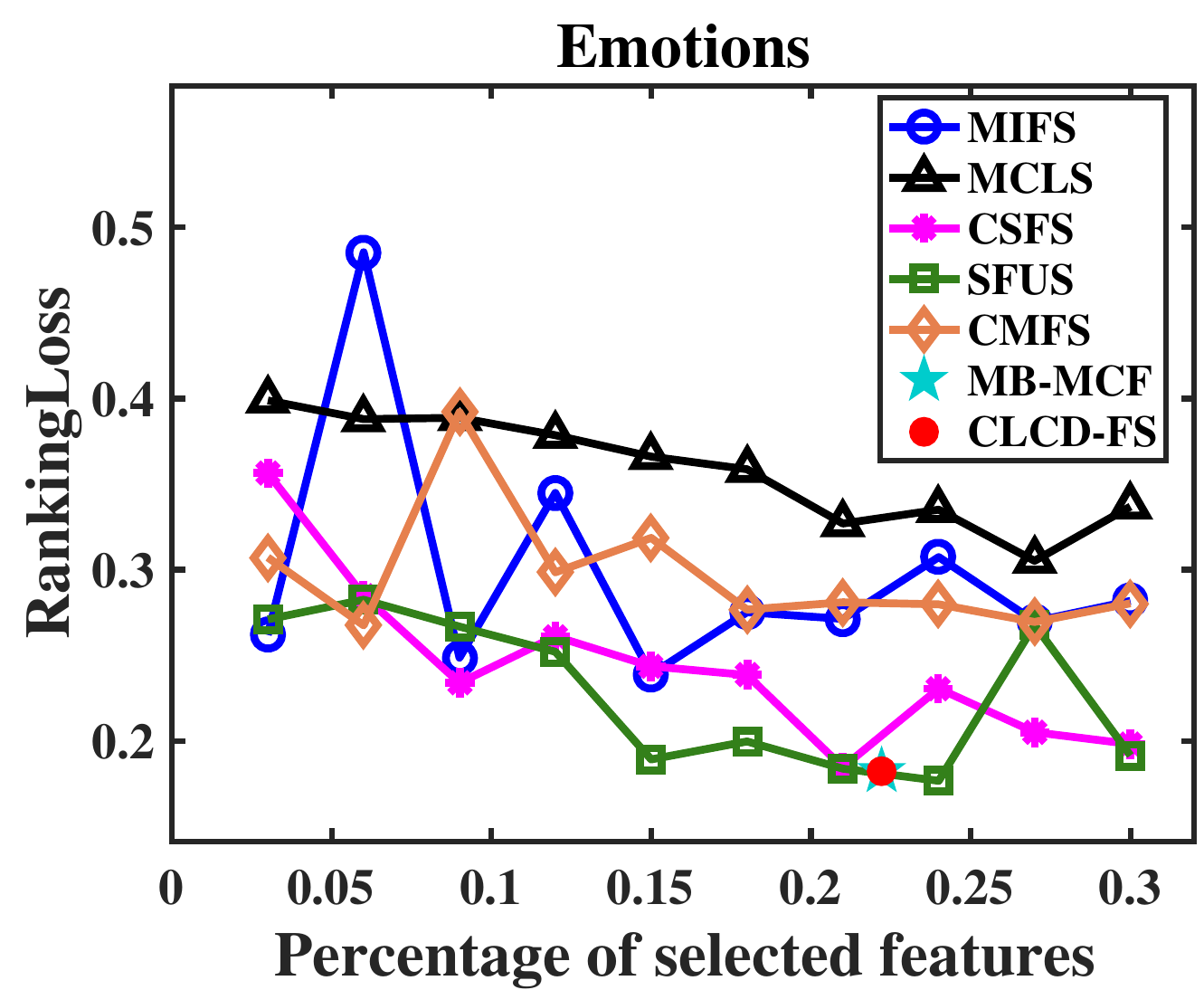}}
\subfigure[]{ \centering
    \label{Emotions_Ma}
    \includegraphics[height=1.3in,width=1.7in]{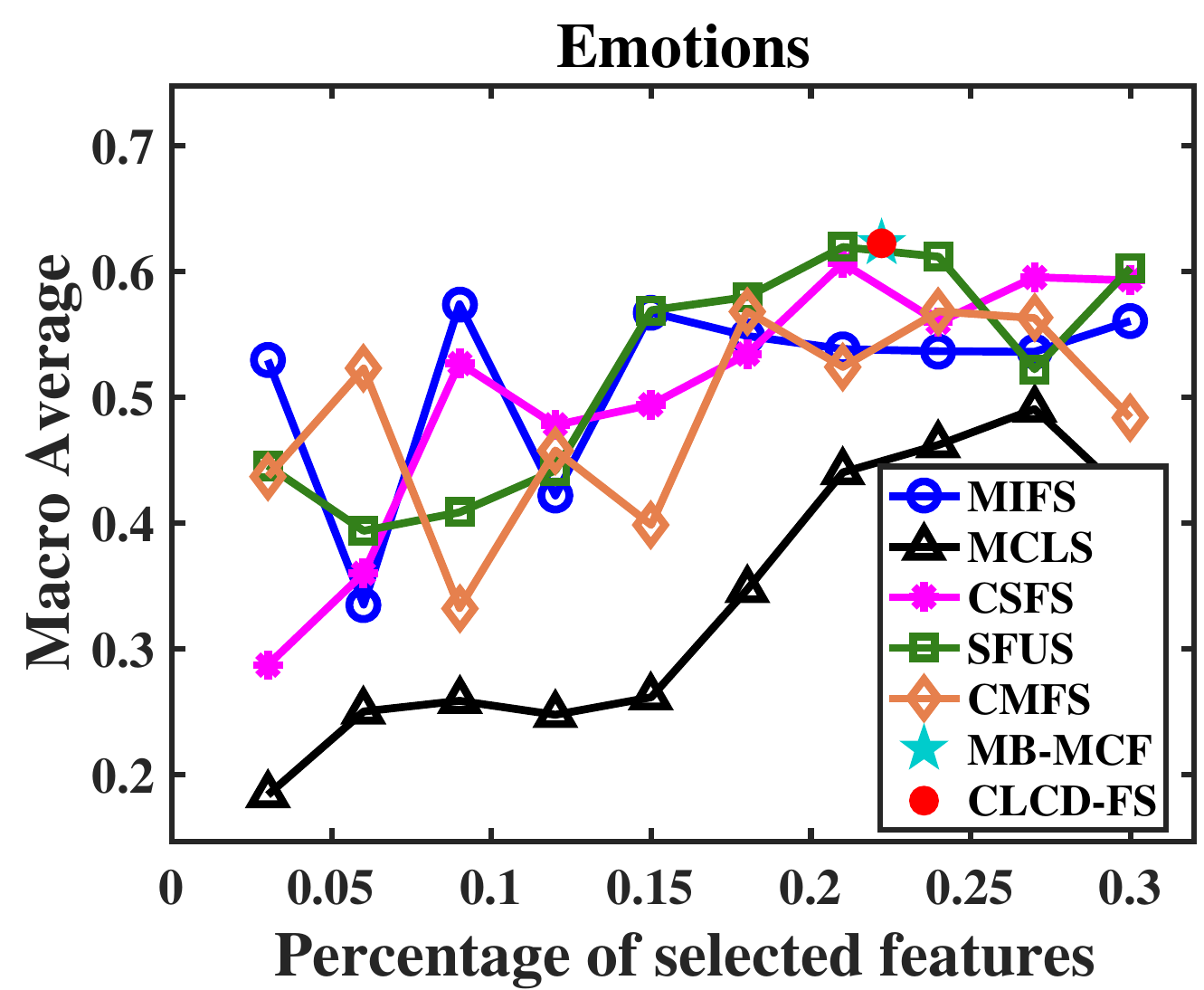}}
\subfigure[]{ \centering
    \label{Emotions_Mi}
    \includegraphics[height=1.3in,width=1.7in]{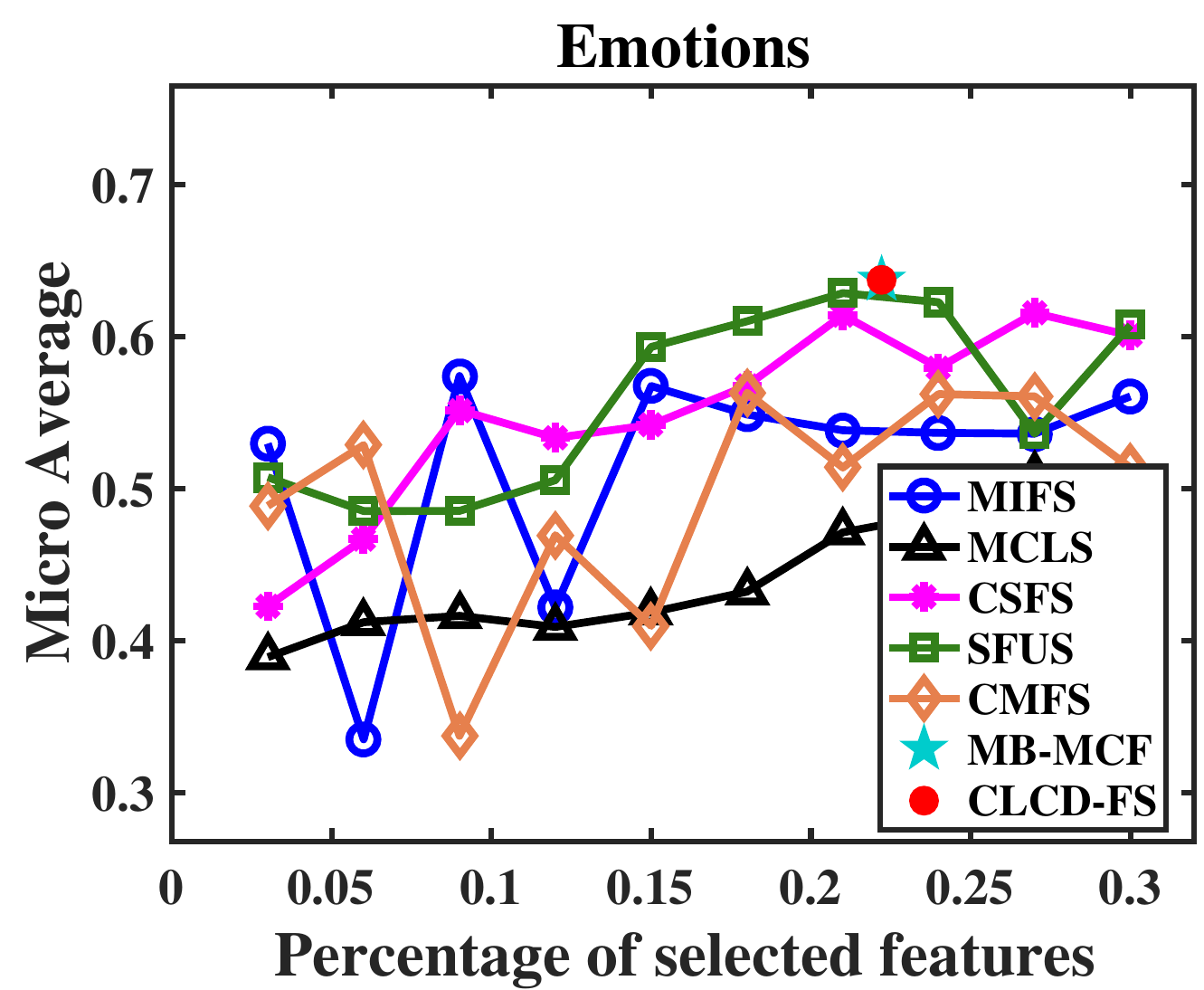}}
\subfigure[]{ \centering
    \label{EUR_HL}
    \includegraphics[height=1.3in,width=1.7in]{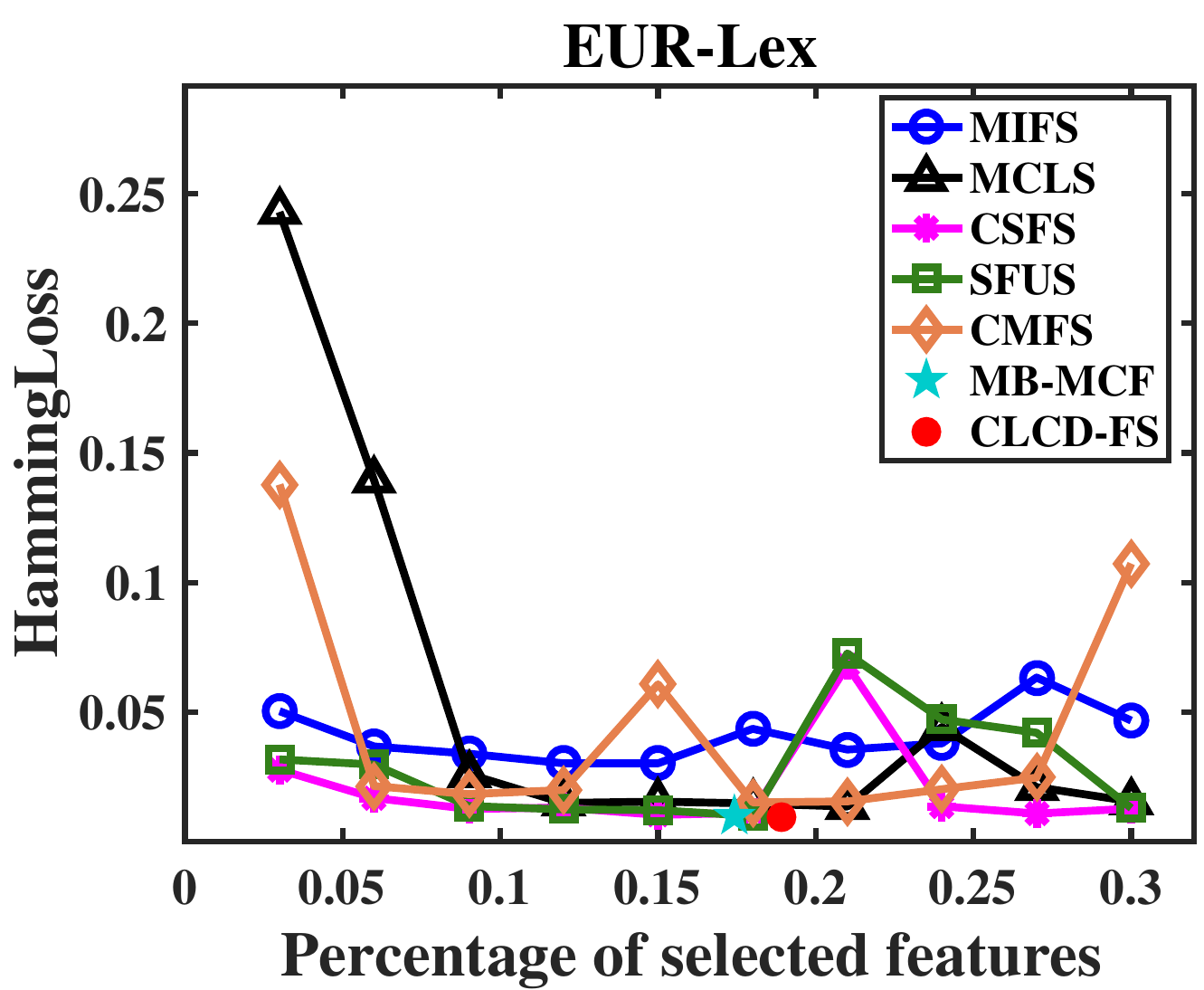}}
\subfigure[]{ \centering
    \label{EUR_RL}
    \includegraphics[height=1.3in,width=1.7in]{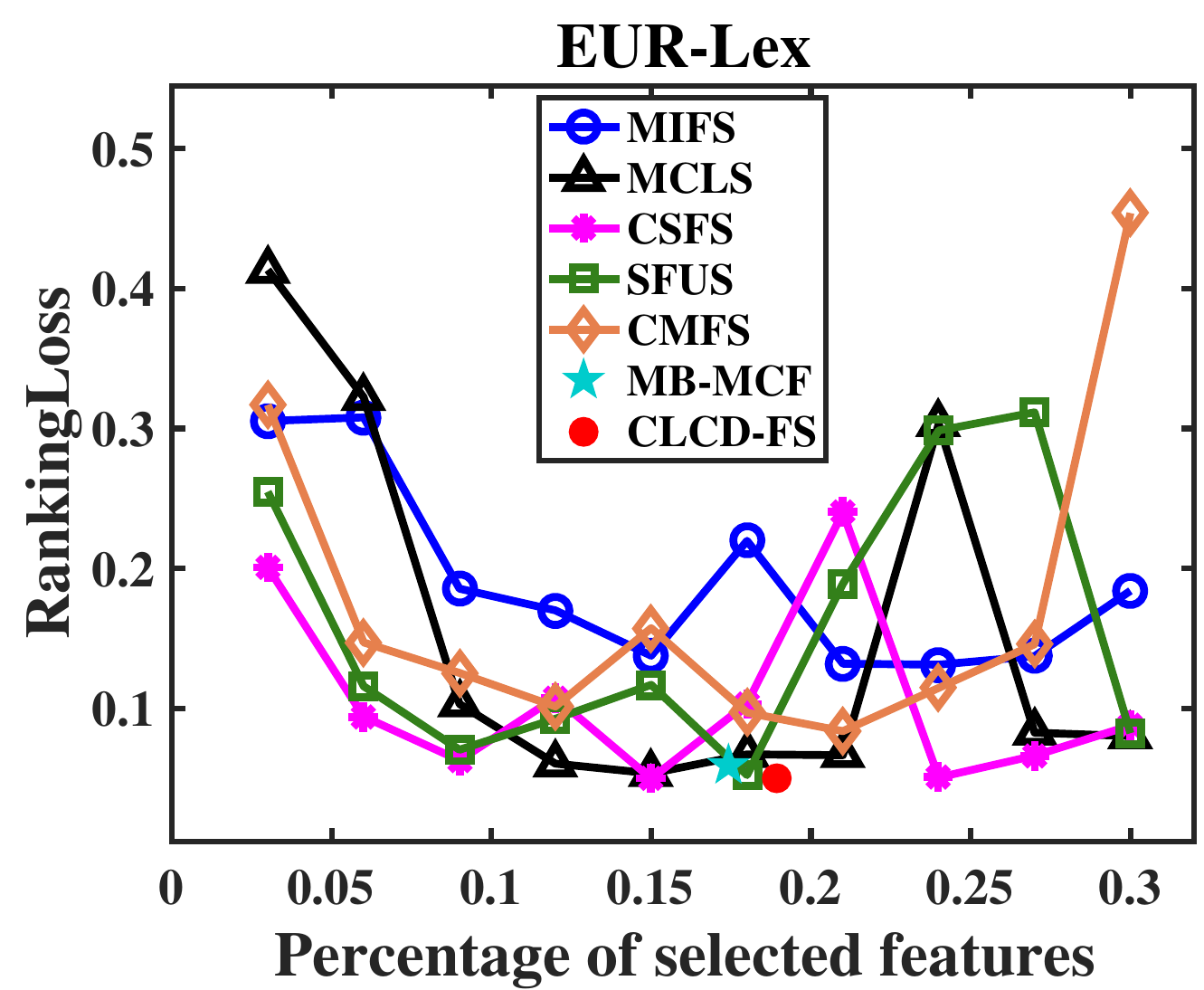}}
\subfigure[]{ \centering
    \label{EUR_Ma}
    \includegraphics[height=1.3in,width=1.7in]{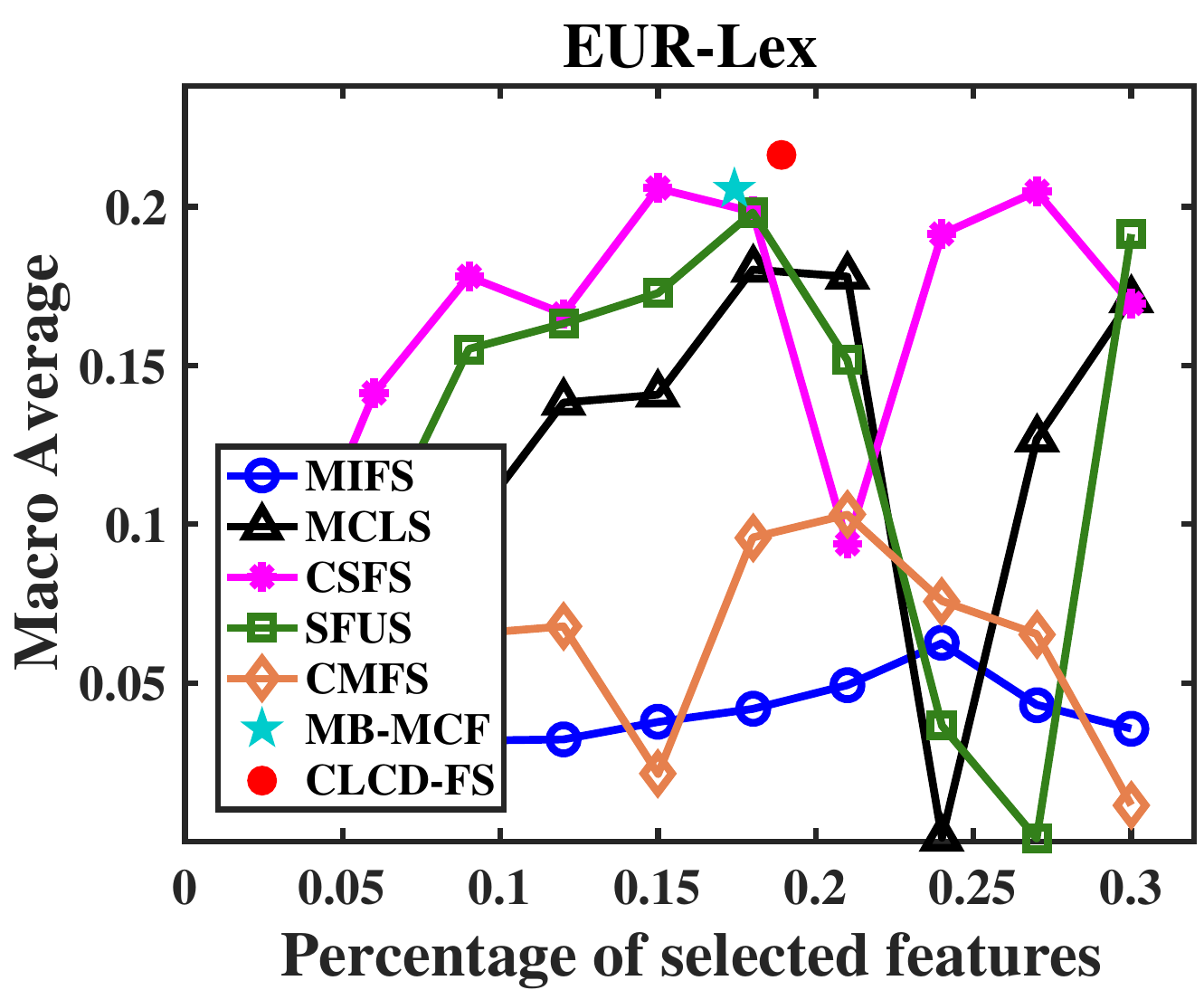}}
\subfigure[]{ \centering
    \label{EUR_Mi}
    \includegraphics[height=1.3in,width=1.7in]{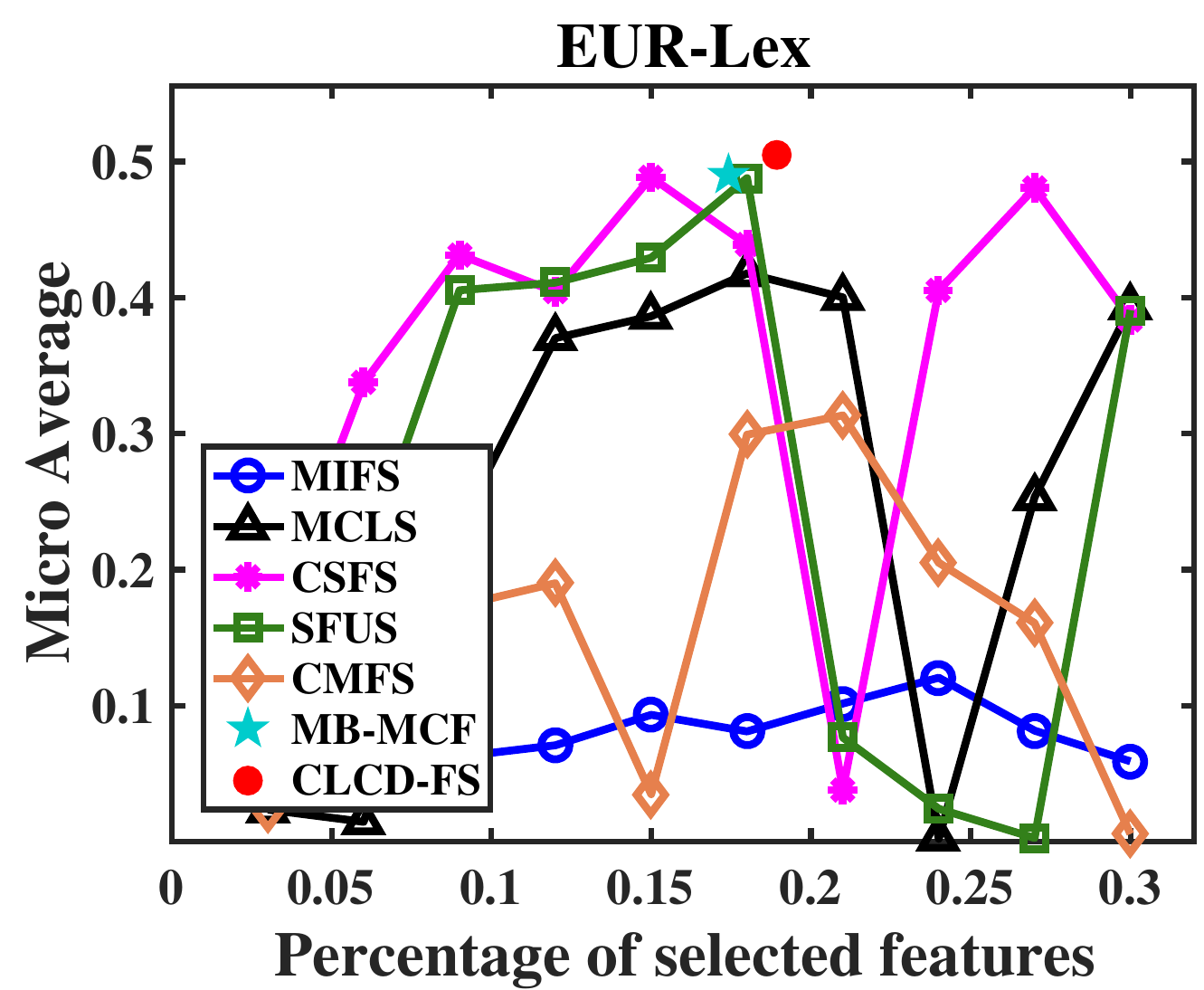}}
\subfigure[]{ \centering
    \label{Mediamill_HL}
    \includegraphics[height=1.3in,width=1.7in]{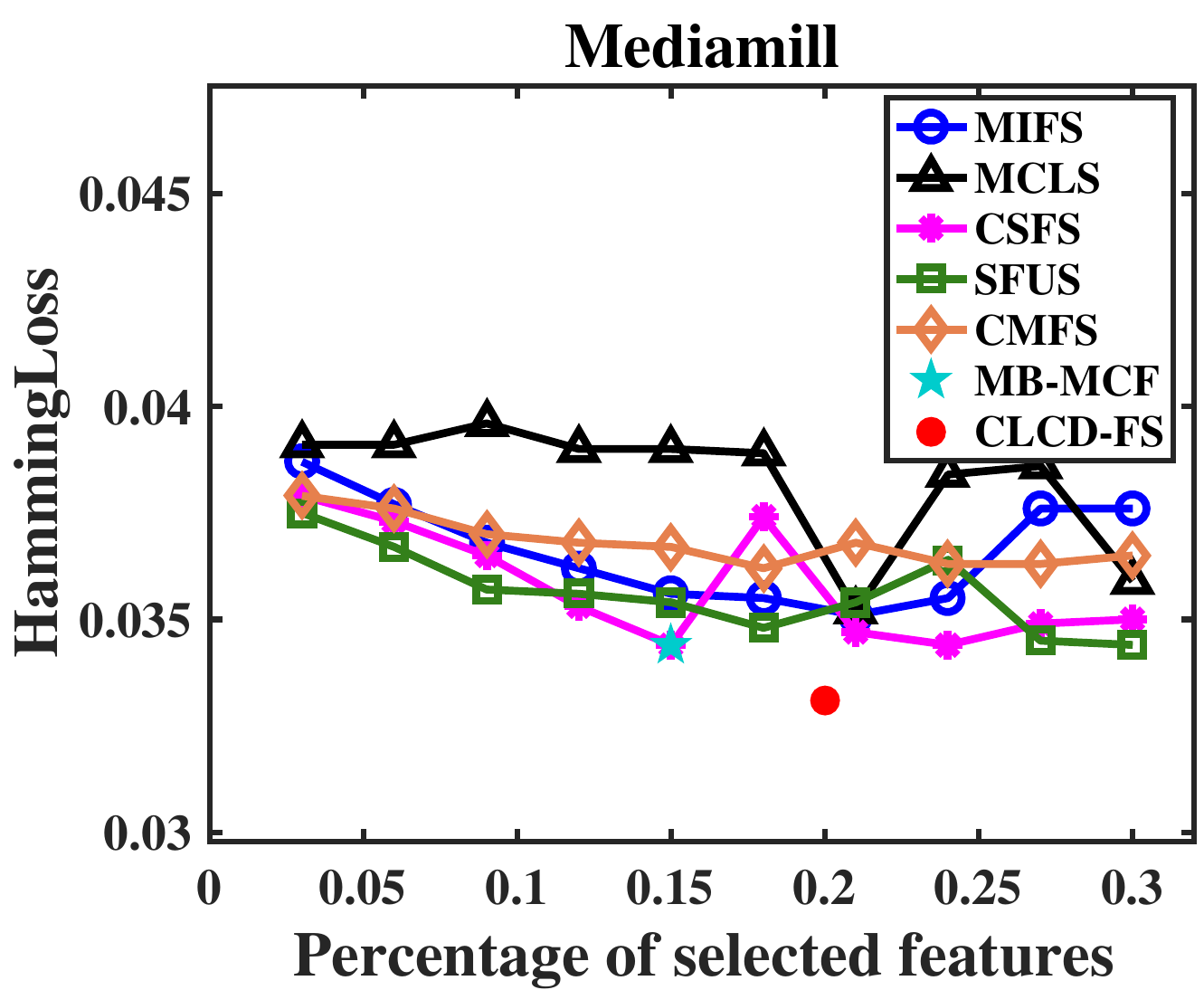}}
\subfigure[]{ \centering
    \label{Mediamill_RL}
    \includegraphics[height=1.3in,width=1.7in]{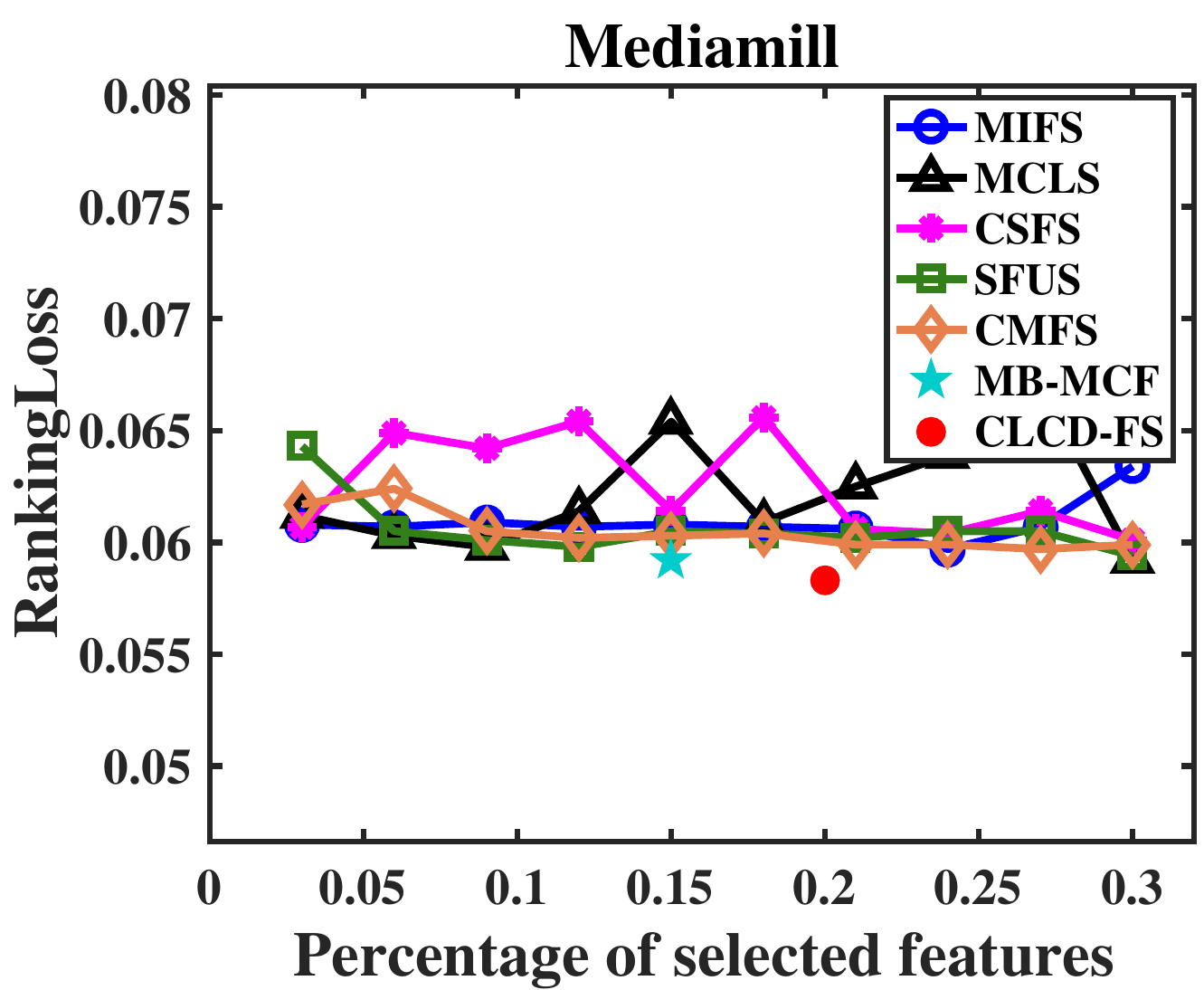}}
\subfigure[]{ \centering
    \label{Mediamill_Ma}
    \includegraphics[height=1.3in,width=1.7in]{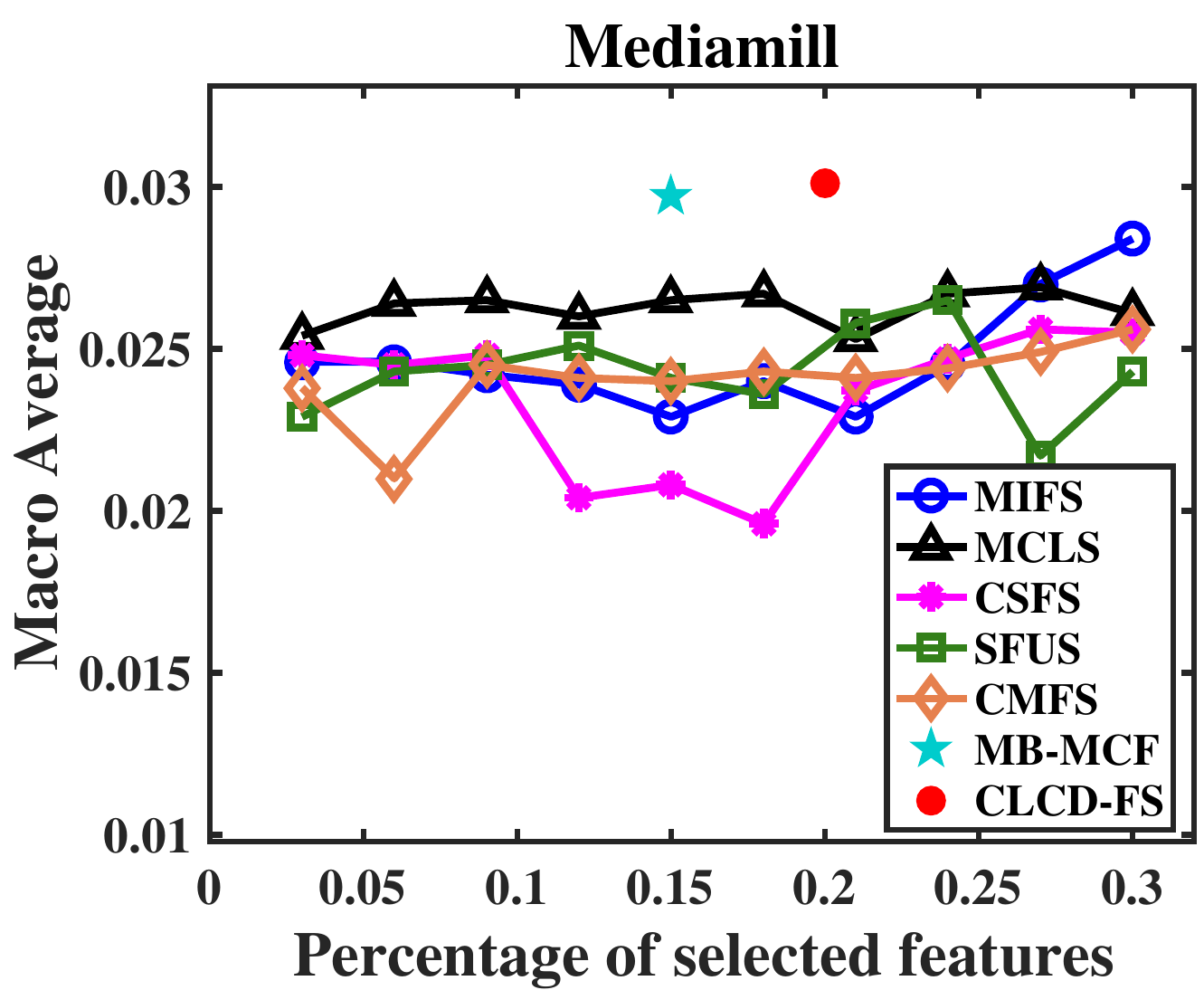}}
\subfigure[]{ \centering
    \label{Mediamill_Mi}
    \includegraphics[height=1.3in,width=1.7in]{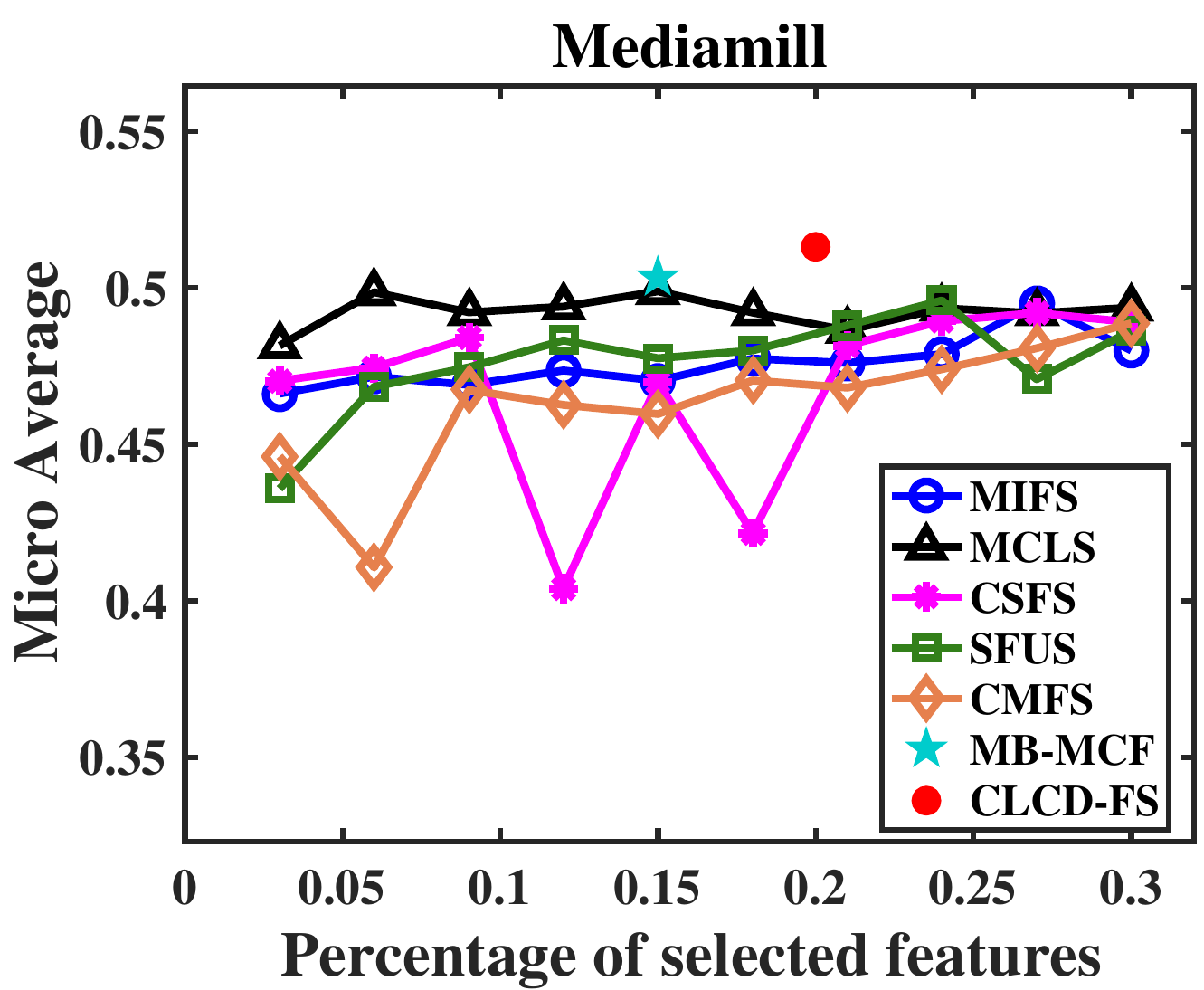}}
\subfigure[]{ \centering
    \label{NUS_HL}
    \includegraphics[height=1.3in,width=1.7in]{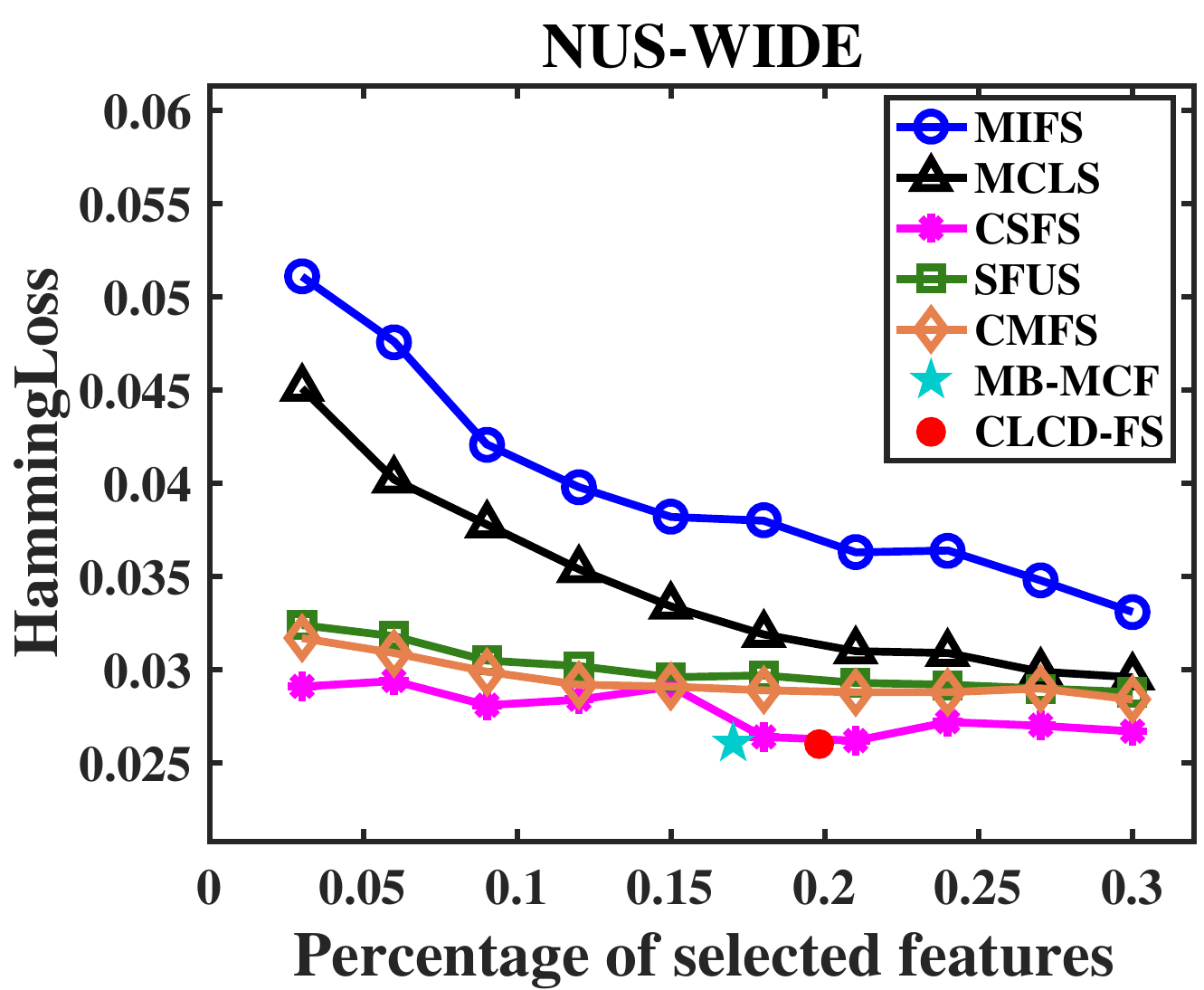}}
\subfigure[]{ \centering
    \label{NUS_RL}
    \includegraphics[height=1.3in,width=1.7in]{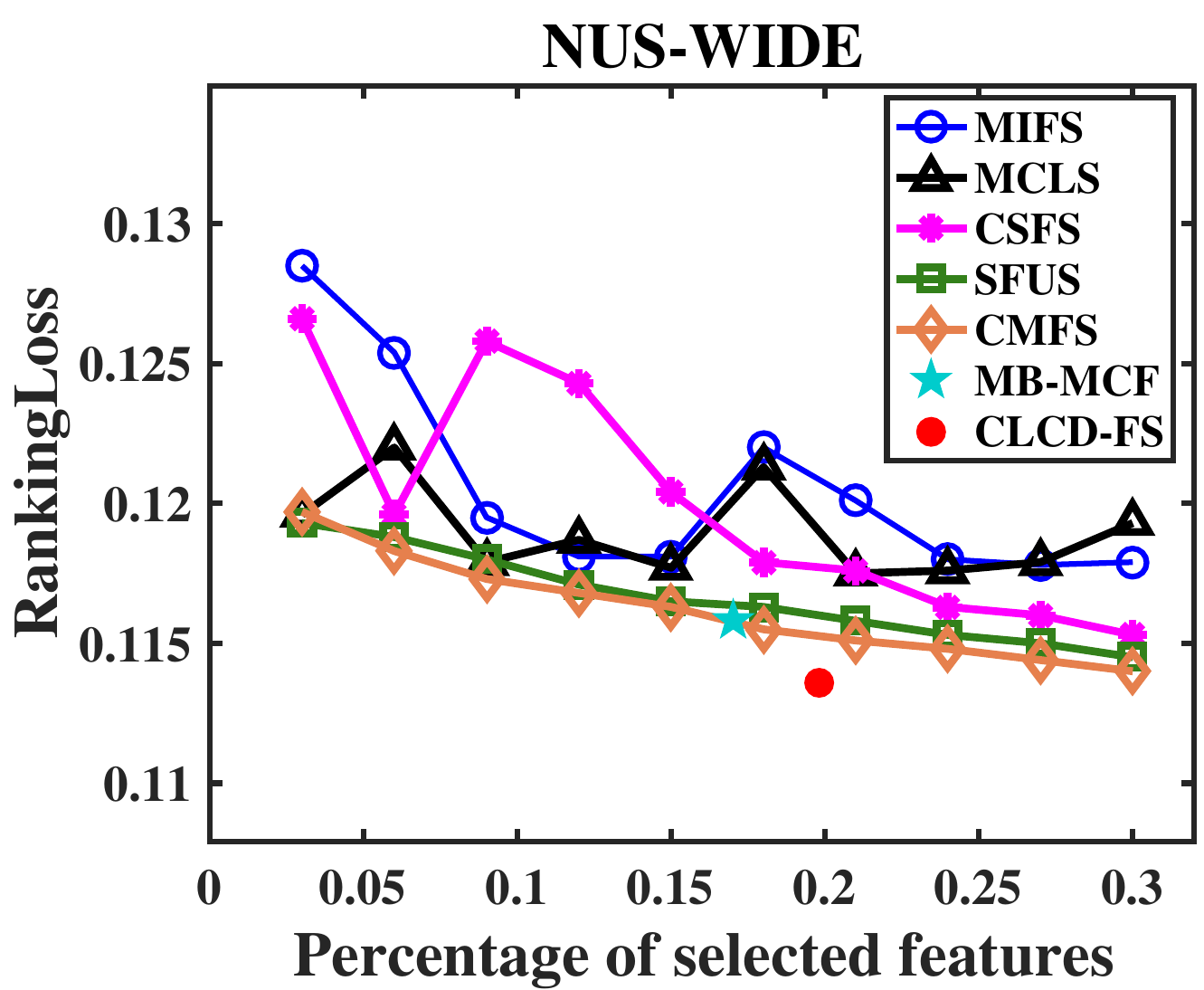}}
\subfigure[]{ \centering
    \label{NUS_Ma}
    \includegraphics[height=1.3in,width=1.7in]{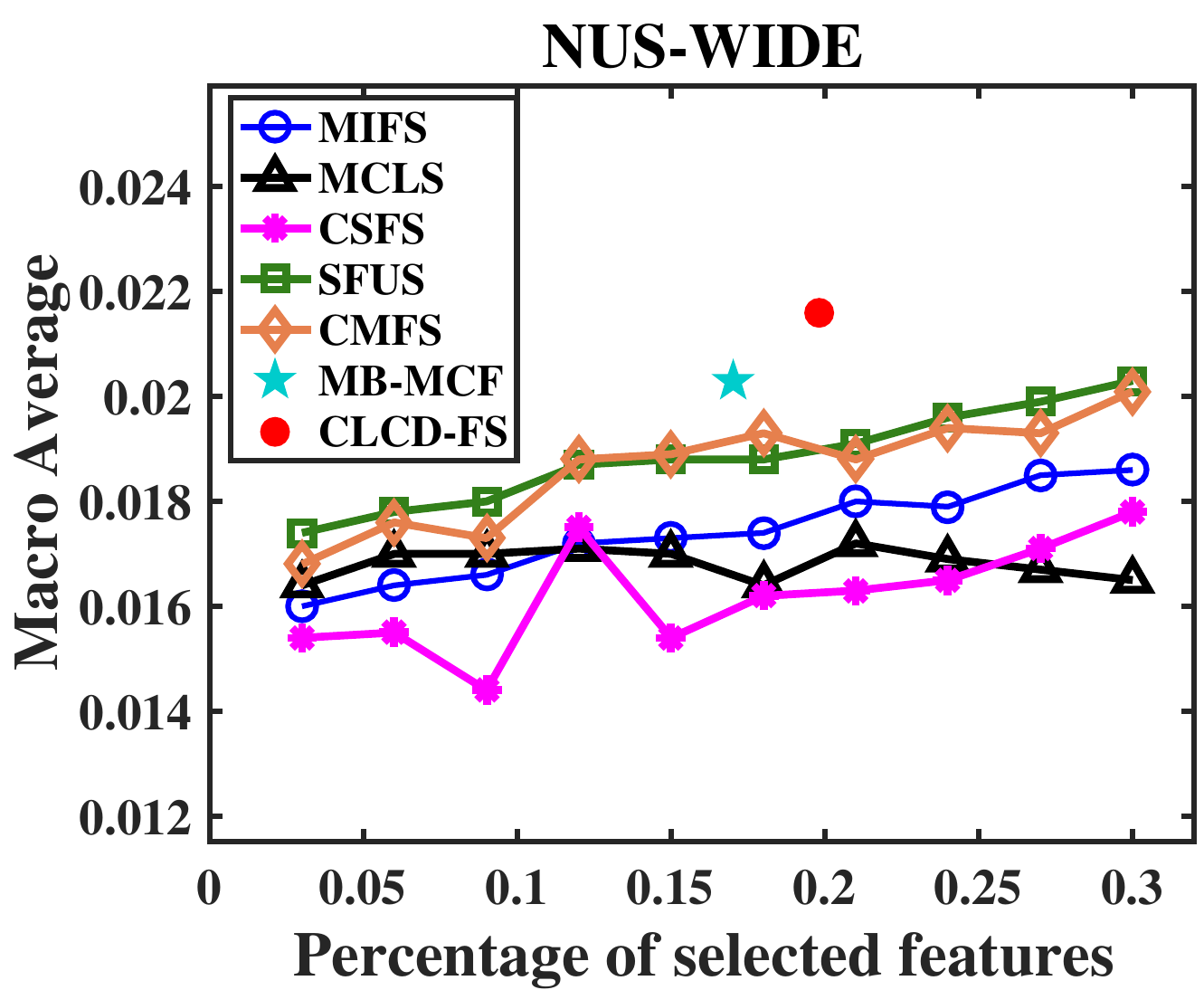}}
\subfigure[]{ \centering
    \label{NUS_Mi}
    \includegraphics[height=1.3in,width=1.7in]{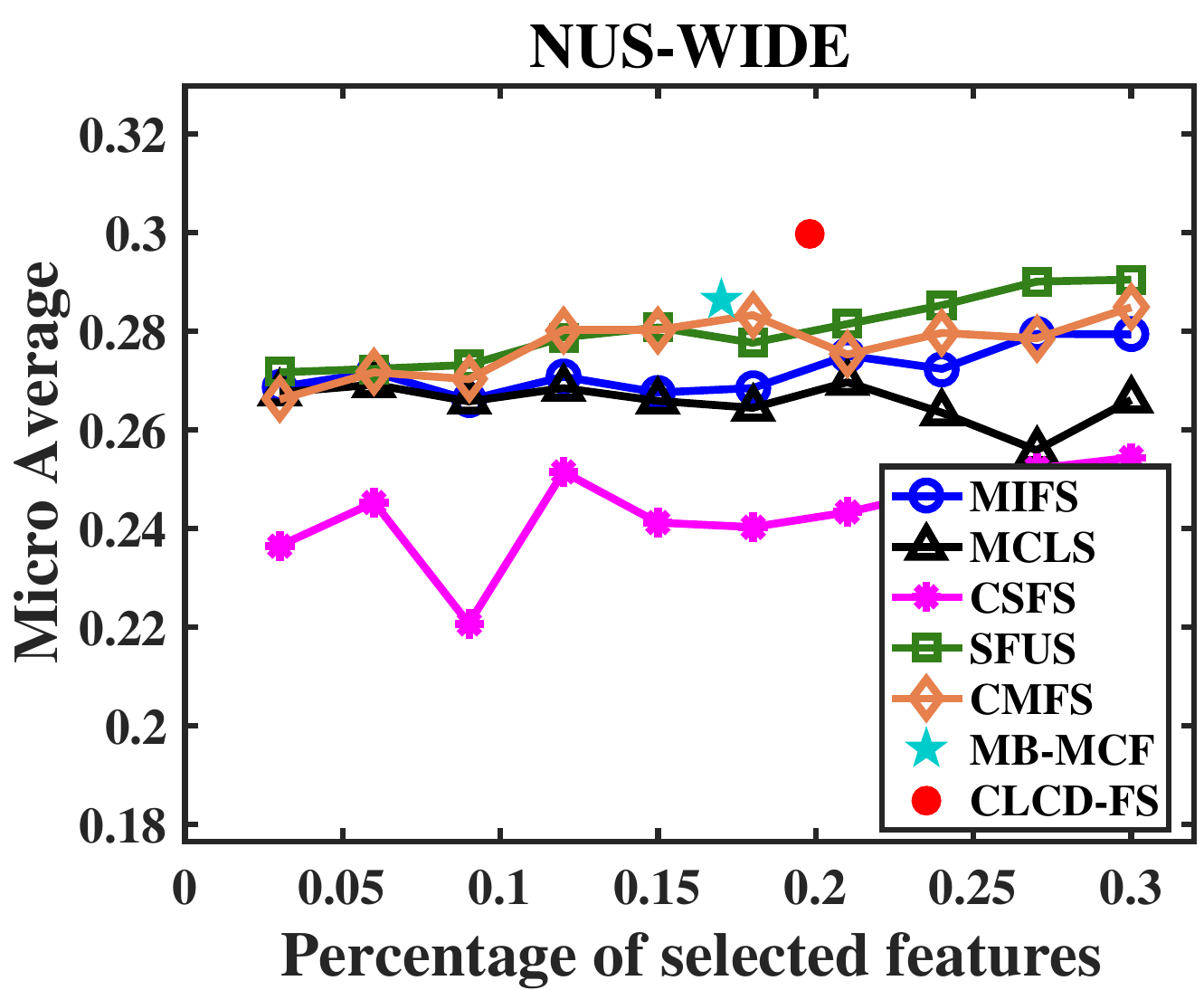}}
\caption{The $Hamming Loss$, $Ranking Loss$, $F_{Macro}$ and $F_{Micro}$ of CLCD-FS and other state-of-the-art multi-label feature selection algorithms on six real-world data sets.}
\label{expic1}
\end{figure*}

\textbf{Metrics for Evaluation}: For fair comparison, we choose two example-based metrics $Hamming Loss$ and $Ranking Loss$, and two label-based metrics $F_{Macro}$ and $F_{Micro}$ (macro-averaging and micro-averaging of F1-measure) \cite{zhang2013review}, to measure the performances of multi-label classification results with selected features of each comparing algorithm. $Hamming Loss$ evaluates the ratio of false outputting labels, including the missed relevant labels and the predicted irrelevant labels:
\begin{equation}
\label{hammingloss}
Hamming Loss=\frac{1}{m}\sum_{i=1}^m \frac{1}{n}|\hat{Y}_i\Delta Y_i|,
\end{equation}
where $m$ is the number of samples and $n$ is the number of labels. $\hat{Y}_i$ and $Y_i$ represents the predicted and real label set of the $i$-th sample, respectively. $\Delta$ denotes the symmetric difference between them.

$Ranking Loss$ evaluates the ratio of reversely ordered label pairs, i.e. an irrelevant label is ranked higher than a relevant label, which is calculated by:
\begin{equation}
\label{rankingloss}
\begin{split}
Ranking Loss=&\frac{1}{m}\sum_{i=1}^m \frac{1}{|\hat{Y}_i||Y_i|} |\{(y_1,y_2)|f(x_i,y_1)\\
&\leq f(x_i,y_2),y_1\in Y_i,y_2\in\hat{Y}_i\}|,
\end{split}
\end{equation}
where $f$ denotes the intermediate real-valued function.

$F_{Micro}$ is the weighted average arithmetic average of $F_1$-score (harmonic mean of $Precision$ and $Recall$) over all $m$ samples, whereas $F_{Macro}$ is an arithmetic average $F_1$-score of all $n$ labels. Mathematically,
\begin{equation}
\label{micro}
F_{Micro}=\frac{1}{n}\sum_{i=1}^n \frac{2TP_i}{2TP_i+FP_i+FN_i}.
\end{equation}
\begin{equation}
\label{macro}
F_{Macro}=\frac{\sum_{i=1}^n 2TP_i}{\sum_{i=1}^n (2TP_i+FP_i+FN_i)}.
\end{equation}
where $TP_i$, $FP_i$ and $FN_i$ denote the number of true positives, false positives and false negatives in the $i$-th label, respectively.

\begin{table}[t]
\caption{Details of the multi-label data sets.}
\centering
\resizebox {3.5in }{!}{
\setlength{\tabcolsep}{0.1in}
\centering
    \begin{tabular}{cccccccc}
    \hline
    Data set & domain & \#Features & \#Labels & $cardinality$ & $density$ \\
    \hline
    Birds & audio  & 260   & 19   & 1.014   & 0.053 \\
    \hline
    CAL500 & music  & 68   & 174   & 26.044   & 0.150 \\
    \hline
    Emotions & music  & 72   & 6   & 1.869   & 0.311 \\
    \hline
    EUR-Lex & text  & 5000   & 201   & 2.213   & 0.011 \\
    \hline
    Mediamill & video  & 120   & 101   & 4.376   & 0.043 \\
    \hline
    NUS-WIDE & images  & 500   & 81   & 1.869   & 0.023 \\
    \hline
    \end{tabular}}
  \label{datasets}%
\end{table}

\textbf{Performance Comparison}: We employ CLCD-FS and other algorithms in comparison to select features first and then train the BR-SVM with these selected features. Each experiment is repeated 10 times with different training and test data, and we report the average performances, i.e., $Hamming Loss$, $Ranking Loss$, $F_{Macro}$ and $F_{Micro}$. As previously mentioned, traditional feature selection algorithms need to predetermine the number of features while CLCD-FS and MB-MCF do not need to, therefore the percentage of the selected features is gradually turned in $\{0.03,0.06,\dots,0.27,0.3\}$ for these traditional algorithms. Similarly, the regularization parameters for all algorithms are searched from $\{0.01,0.1,0.3,\dots,$ $0.9,1\}$ by grid search. The MB discovery algorithm in CLCD-FS and MB-MCF is HITON-MB \cite{hiton} and the parameter in its $G^2$-test \cite{pearl1988} is set as 0.05. Fig. \ref{expic1} shows the average $Hamming Loss$, $Ranking Loss$, $F_{Macro}$ and $F_{Micro}$ variation curves of different multi-label feature selection algorithms with respect to the percentage of selected features.

\begin{figure*}
\centering%
\subfigure[$Hamming\ Loss$]{ \centering
    \label{Spider_HammingLoss}
    \includegraphics[height=1.3in,width=1.7in]{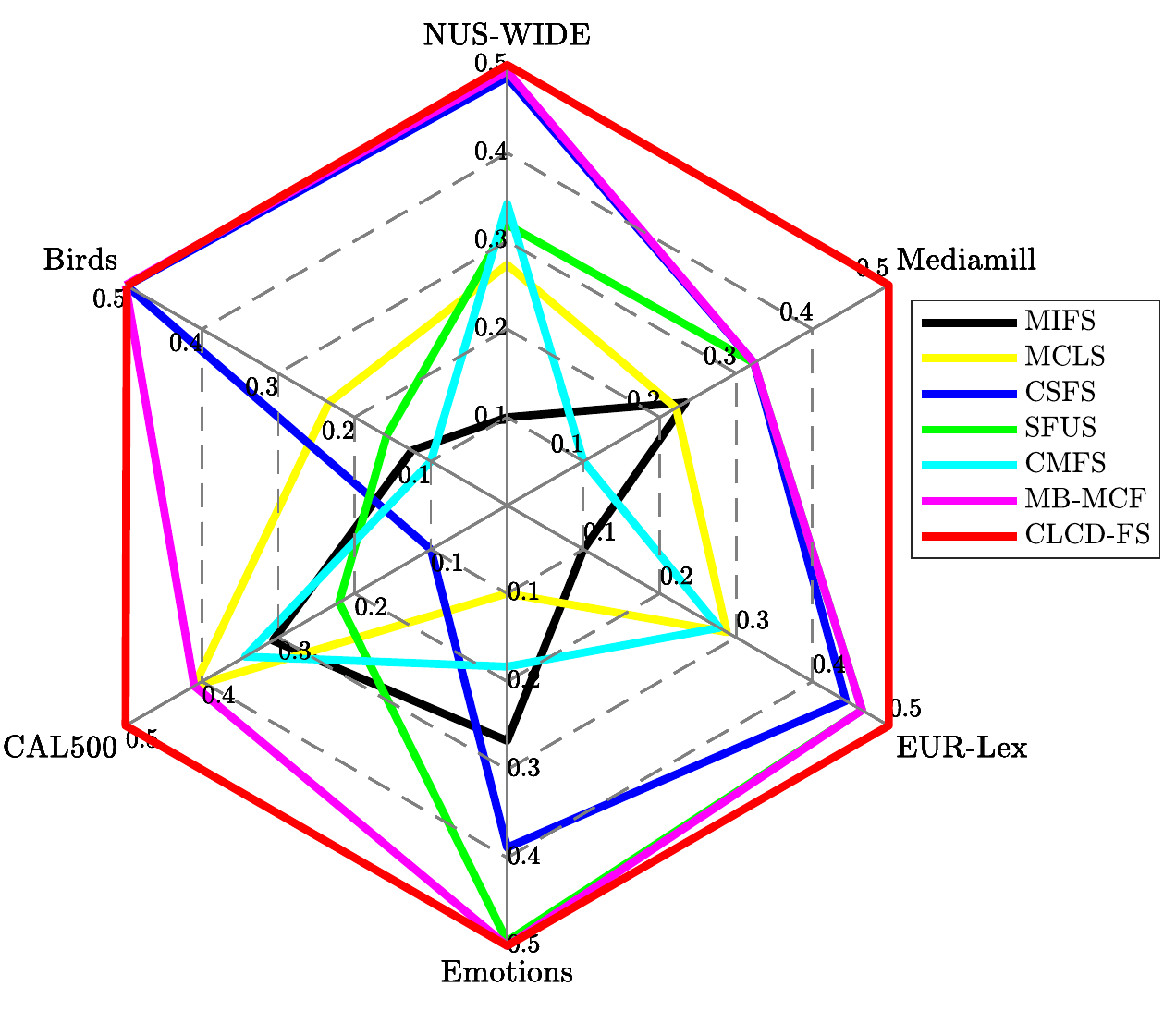}}
\subfigure[$Ranking\ Loss$]{ \centering
    \label{Spider_RankingLoss}
    \includegraphics[height=1.3in,width=1.7in]{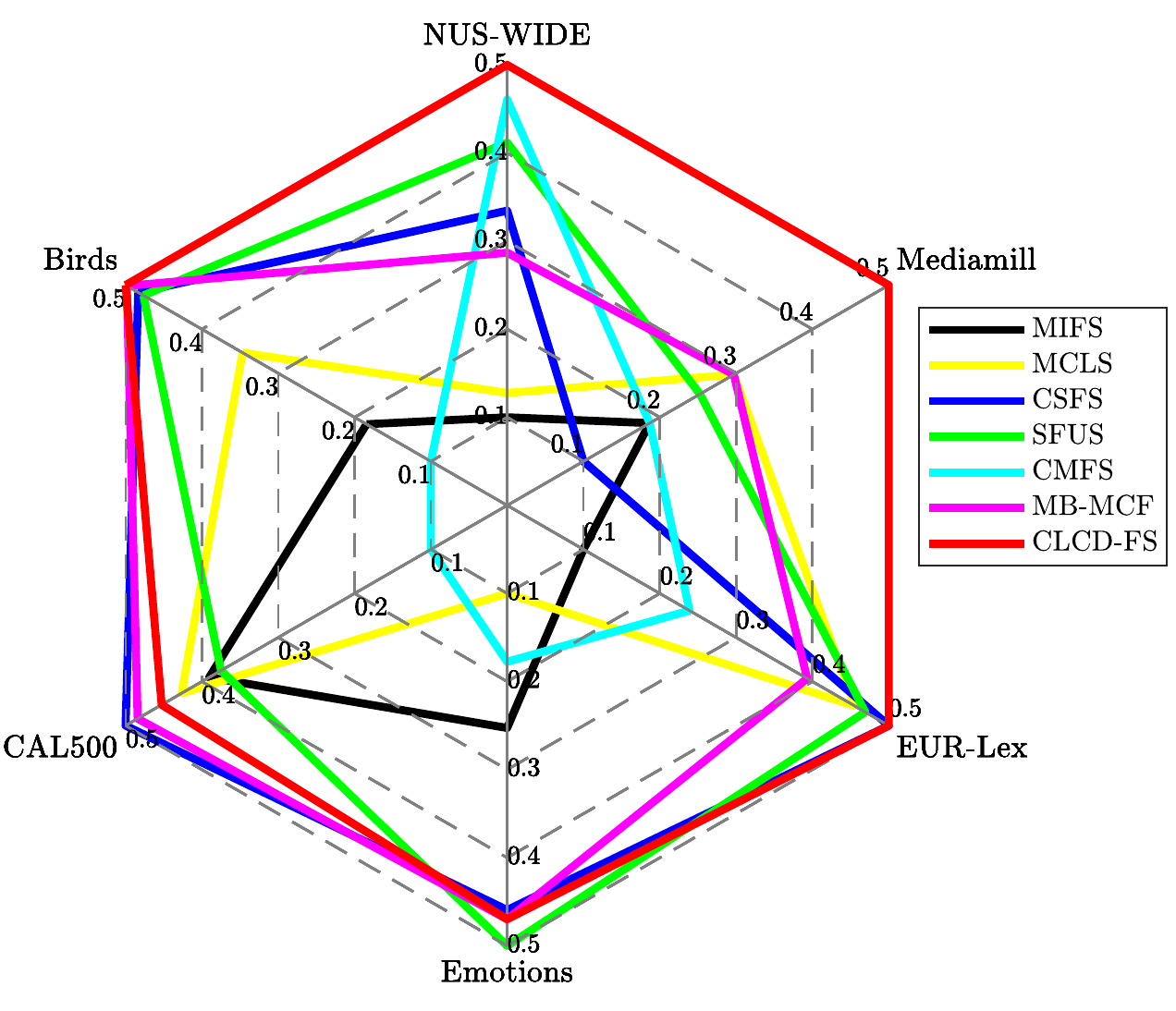}}
\subfigure[$F_{Macro}$]{ \centering
    \label{Spider_MacroAverage}
    \includegraphics[height=1.3in,width=1.7in]{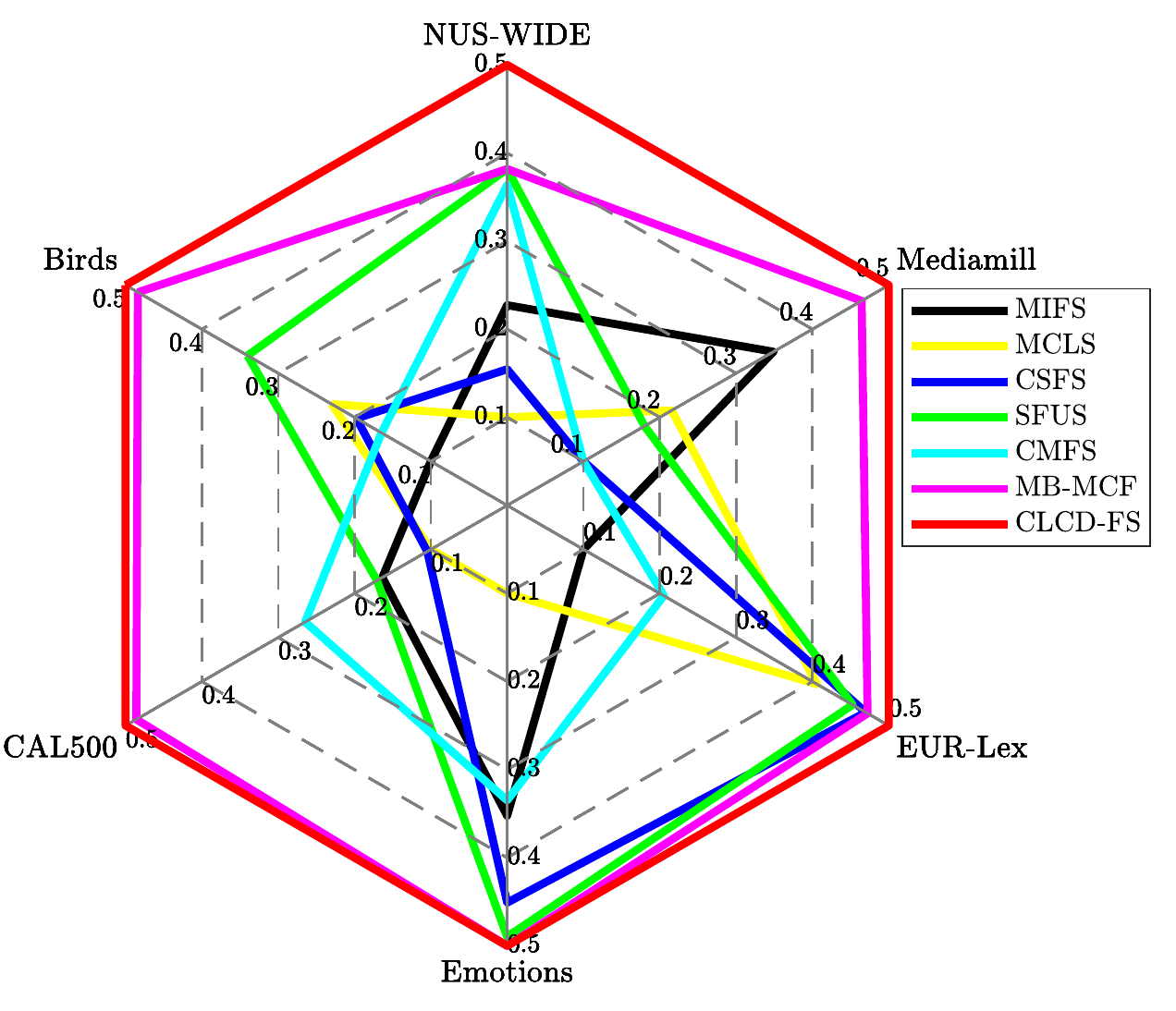}}
\subfigure[$F_{Micro}$]{ \centering
    \label{Spider_MicroAverage}
    \includegraphics[height=1.3in,width=1.7in]{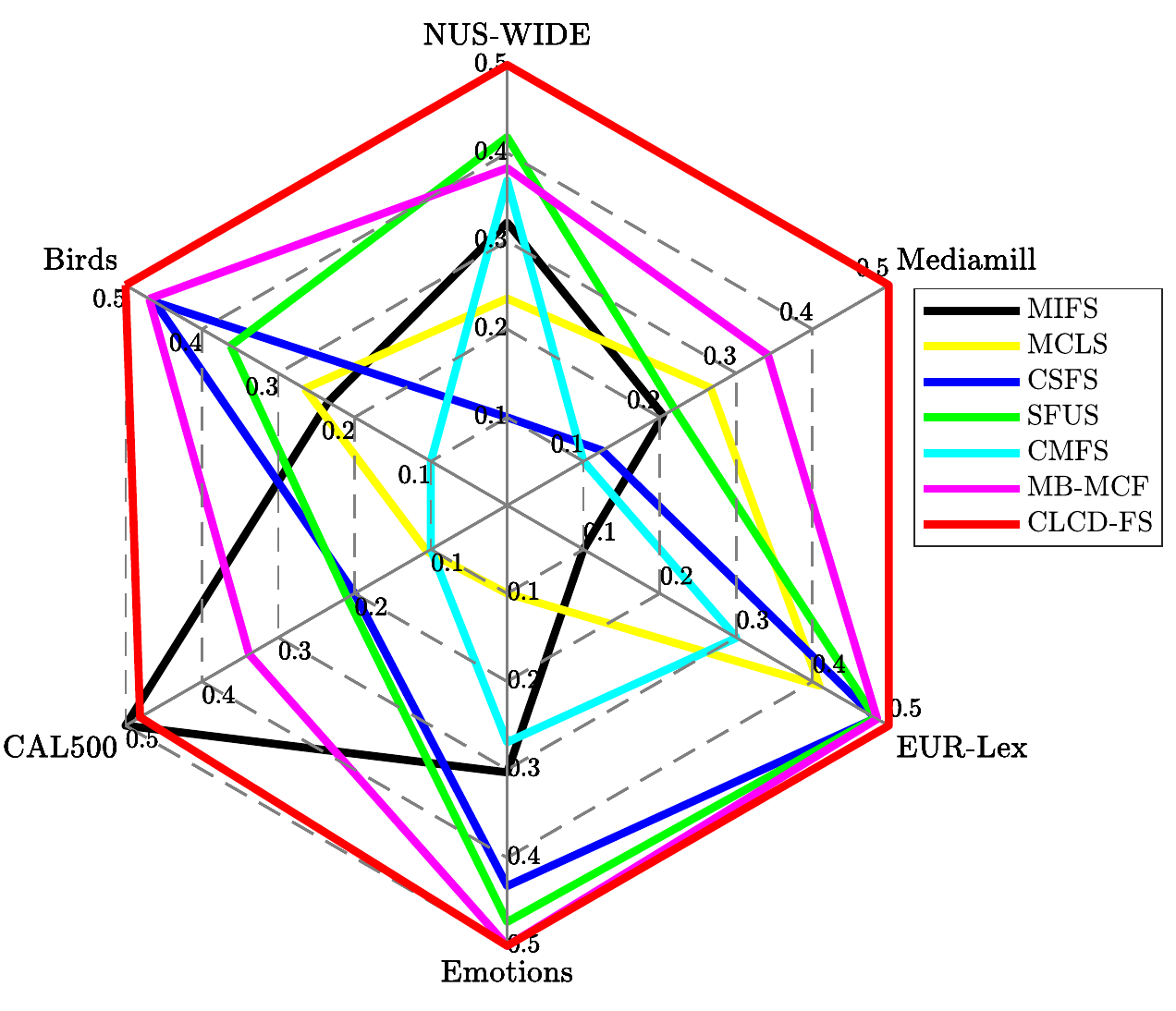}}
\caption{Spider web diagrams showing the stability obtained on six multi-label data sets with $Hamming Loss$, $Ranking Loss$, $F_{Macro}$ and $F_{Micro}$ of CLCD-FS and other state-of-the-art multi-label feature selection algorithms.}
%we do not specifically provide the the results of PM-CCMB.
\label{expic_spider}
\end{figure*}

%\begin{table}[h]
%\caption{The number of features selected by MB-MCF.}
%\centering
%\resizebox {3.2in }{!}{
%\setlength{\tabcolsep}{0.025in}
%\centering
%    \begin{tabular}{|c|c|c|c|c|c|}
%    \hline
%    Data set & Birds & CAL500 & EUR-Lex & Mediamill & NUS-WIDE \\
%    \hline
%    \#Features & 40  & 13 & 871  & 18   &  85 \\
%    \hline
%    \end{tabular}}
%  \label{features}%
%\end{table}

(1) \emph{Comparison with traditional algorithms:} As mentioned previously, CLCD-FS could automatically determine the number of selected features, and thus its performance trend is a red dot, instead of a curve. Based on the experimental results in Fig. \ref{expic1}, we make the following observations:

(i) Under the same ratio of selected features, CLCD-FS achieves the best performance in terms of four metrics compared with the traditional feature selection algorithms as shown in Fig. \ref{expic1}. (ii) For $Hamming Loss$, $F_{Macro}$, and $F_{Micro}$, CLCD-FS consistently outperforms the best performance of these traditional algorithms. Especially on large-scale data set (EUR-Lex, Mediamill, and NUS-WIDE), CLCD-FS achieves significantly higher performance compared with traditional methods, which validates the practicability in the real-world large-scale problem. (iii) For $Ranking Loss$, CLCD-FS is slightly worse than CSFS algorithm on Birds data set, and SFUS algorithms on Emotions data set. It is reasonable since few algorithms could obtain the best performance on all these metrics simultaneously. Nevertheless, the performance of CLCD-FS is also competitive compared with other algorithms. (iv) CLCD-FS could automatically find the optimal number of features. The size of feature set selected by CLCD-FS exactly falls nearby the optimal point, and accordingly, CLCD-FS achieves the best performance.

Overall, we can observe from Fig. \ref{expic1} that, the superiority compared with traditional algorithms reflects on two sides: (i) CLCD-FS could automatically determine the relatively optimal number of selected features; and (ii) CLCD-FS achieves the best or very competitive performances, especially on large-scale data sets.

(2) \emph{Comparison with MB-MCF:} On most of data sets, CLCD-FS selects more features than MB-MCF, and accordingly, CLCD-FS achieves better performance than MB-MCF, which indicates that CLCD-FS could better determine the relative optimal number of selected features. The additional features selected by CLCD-FS generally include two types. Some of these features are the spouse features of labels, which could enhance the predictive ability of the direct effects (child features) of labels. Others are the relevant features shielded by label causality in the feature selection process, which are retrieved by CLCD-FS to make up the predictive information loss contained by some labels. These improvements suggest that CLCD-FS achieves better performance compared with MB-MCF.

We also note that MB-MCF and CLCD-FS select the same features on Emotions data set. Nonetheless, the causal relationship between a label and a feature mined in these algorithms are not exactly the same. In Section \ref{ex3}, we demonstrate the causal relationships retrieved by CLCD-FS on Emotions, which is slightly different as compared with Figure 4 in \cite{wu2020multi}.

(3) \emph{Stability Analysis:} To verify the stability of different methods, we draw four spider web diagrams in terms of each evaluation metric, i.e., $Hamming Loss$, $Ranking Loss$, $F_{Macro}$, and $F_{Micro}$. The best performance of each algorithm is normalized to a universal standard $[0,0.5]$ so that the differences between the classification performances on different data sets do not influence the demonstration. Then, we present the stability index according to the value after normalization. The stability with different metrics is shown in Fig. \ref{expic_spider}, where the red line denotes the stability value of the proposed CLCD-FS algorithm. We conclude from Fig. \ref{expic_spider} that: (i) For $Hamming Loss$, $F_{Macro}$, and $F_{Micro}$, the shapes of CLCD-FS are regular hexagons, which means that CLCD-FS obtains the most stable solution on each metric. For $Ranking Loss$, CLCD-FS is close to a regular hexagon. Nonetheless, CLCD-FS more comes into contact with the regular hexagon than other algorithms. (ii) Compared with MB-MCF, the performance and stability of CLCD-FS are significantly improved due to the theoretical guarantee proposed in this paper.

(4) \emph{Experiment Time:} We recorded the CPU time for each algorithm on each data set in the above experiments. Table \ref{dataset_time} provides the average running time under the same number of selected features with CLCD-FS (except MB-MCF since it could determine the selected feature size).

\begin{table}[t]
\caption{Experiment time ($lg(Time)$) of each algorithms.}
\centering
\resizebox {3.5in }{!}{
\setlength{\tabcolsep}{0.1in}
\centering
    \begin{tabular}{cccccccccc}
    \hline
    Algorithm & MIFS & MCLS & CSFS & SFUS & CMFS & MB-MCF & CLCD-FS\\
    \hline
    Birds & 1.924 & 1.646 & \textbf{1.409} & 1.546 & 1.763 & 1.795 & 1.857  \\
    \hline
    CAL500 & 4.793 & 4.788 & 4.802 & \textbf{4.782} & 4.792 & 4.841  & 4.880 \\
    \hline
    Emotions & 0.983 & 1.002 & 0.999 & 0.986 & 1.048 & \textbf{0.957} & 1.015  \\
    \hline
    EUR-Lex & 4.674 & 4.655 & 4.715 & \textbf{4.649} & 4.666 & 4.760 & 4.766  \\
    \hline
    Mediamill & 4.318 & 4.334 & \textbf{4.296} & 4.342 & 4.329 & 4.375 & 4.442  \\
    \hline
    NUS-WIDE & 4.166 & 4.176 & \textbf{4.006} & 4.181 & 4.184 & 4.233 & 4.274 \\
    \hline
    \end{tabular}}
  \label{dataset_time}%
\end{table}

We conclude from the Table \ref{dataset_time} that, the CPU time of CLCD-FS is similar to MB-MCF but slightly higher than the traditional multi-label feature selection methods. Note that, causality-based feature selection methods usually have higher time complexity than traditional methods since they possess interpretability and theoretical guarantee \cite{yu2019tpami}. Hence, the loss of time efficiency is unsurprising. However, these traditional methods need to execute many times to determine the optimal number of selected features, and the process using classifiers to obtain the predictability of a feature subset is far more time-consuming than the feature selection process, whereas CLCD-FS could predetermine the number of selected feature. Therefore, the cost of time is reasonable coupled with many benefits of CLCD-FS.

\subsection{An Example of the Interpretability of CLCD-FS}
\label{ex3}

\begin{figure}[t]%t
  \centering
  \includegraphics[height=1.50in, width=3.00in]{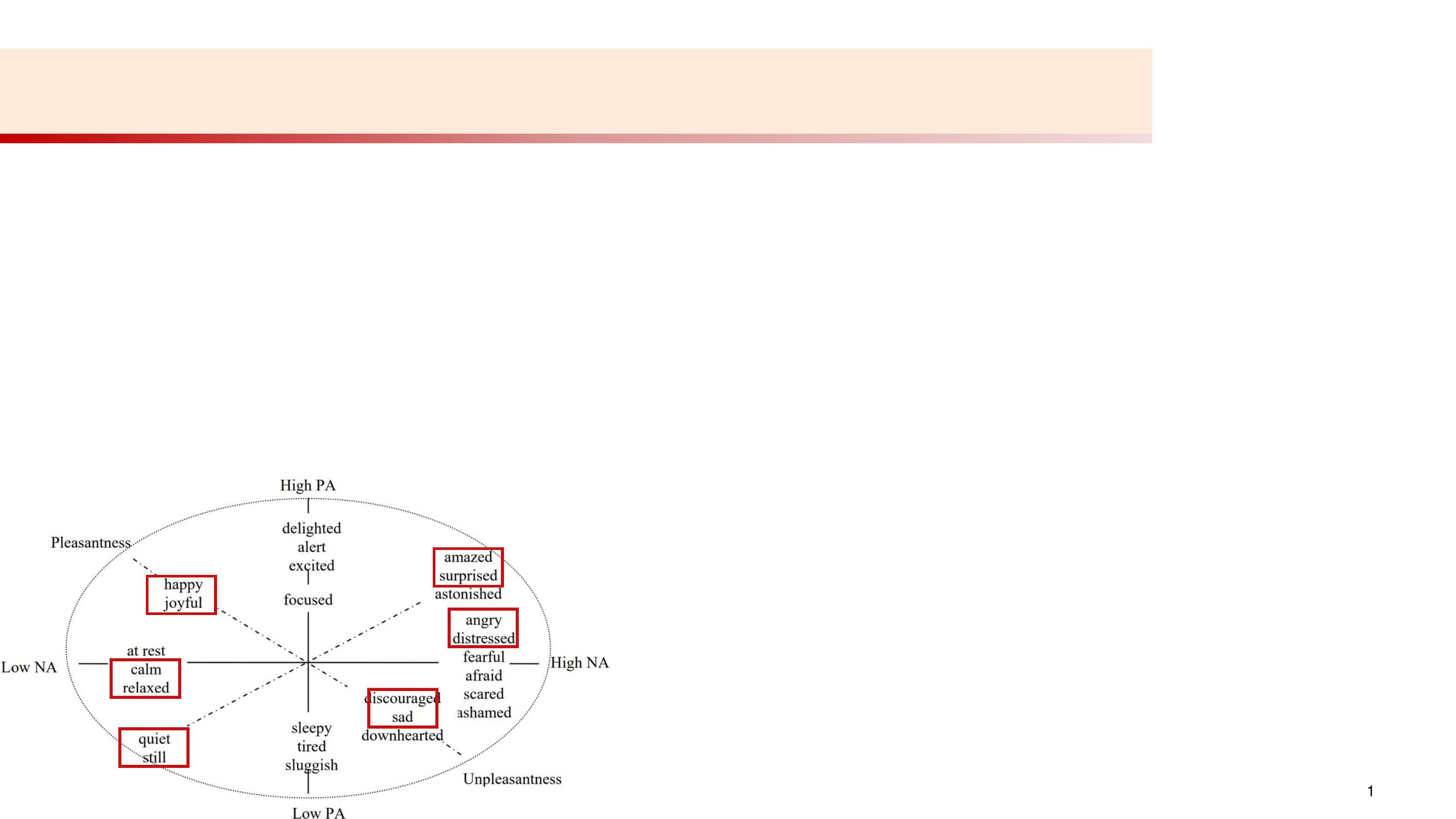}
  \caption{The graphical representation of Tellegen-Watson-Clark model \cite{tellegen1999dimensional}, where amazed-surprised, happy-pleased, relaxing-calm, quiet-still, sad-lonely, and angry-aggressive are labels in Emotions data set.}\label{Emotions1}
\end{figure}

\begin{figure}[t]%t
  \centering
  \includegraphics[height=1.20in, width=3.00in]{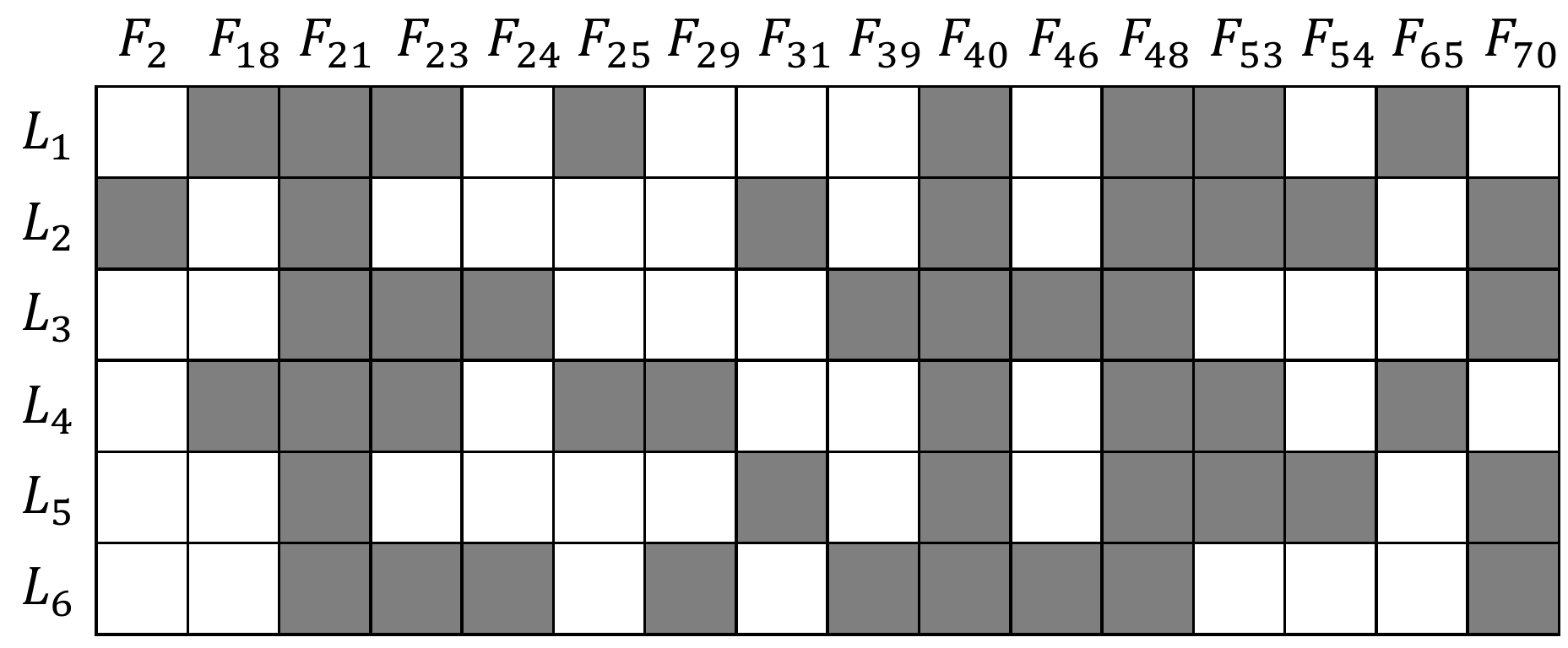}
  \caption{The Identified relationship between each selected feature and each label in the Emotions data set. In the grid, each column corresponds to a feature selected by CLCD-FS, and each row corresponds to a label, where $L_1$: amazed-surprised, $L_2$: happy-pleased, $L_3$: relaxing-calm, $L_4$: quiet-still, $L_5$: sad-lonely, $L_6$: angry-aggressive. The shaded cell indicates that the corresponding feature has an effect on the corresponding label.}\label{Emotions2}
\end{figure}

Inherited from CLCD, CLCD-FS naturally possesses the capacity of distinguishing the common features and label-specific features, and thus it can explicitly interpret which labels a feature causally influences. In this subsection, we further choose Emotions data set as an example to demonstrate the interpretability of CLCD-FS, because: (i) Emotions data set is derived from psychology and some known conclusion in the psychological study \cite{tellegen1999dimensional} can be used to validate our experimental results; (ii) Emotions data set contains 6 labels and 72 features, convenient to demonstrate in a figure.

These 72 features are extracted from musical context. And the 6 labels are derived from the Tellegen-Watson-Clark model of mood \cite{tellegen1999dimensional} as shown in Fig. \ref{Emotions1}, including amazed-surprised ($L_1$), happy-pleased ($L_2$), relaxing-calm ($L_3$), quiet-still ($L_4$), sad-lonely ($L_5$), and angry-aggressive ($L_6$). From Fig. \ref{Emotions1}, we can discover that label pairs $(L_1,L_4)$, $(L_2,L_5)$ and $(L_3,L_6)$ are opposite emotions, which means that the labels in each pair should be influenced by similar features.

The identified relationship between each selected feature and each label are provided in Fig. \ref{Emotions2}, where a dyed cell indicates that the corresponding feature has an effect on the corresponding label. We can observe from Fig. \ref{Emotions2} that common features $F_{21}$, $F_{40}$, and $F_{48}$ carry the information about all labels, and $F_2$ is a label-specific feature of label $L_2$. From the distribution of shaded cells, we can conclude that the labels in label pairs $(L_1,L_4)$, $(L_2,L_5)$ and $(L_3,L_6)$ share similar common features, which is consistent with the Tellegen-Watson-Clark model in previous study \cite{tellegen1999dimensional}.

Note that, although CLCD-FS selects the same features as MB-MCF, the mined causal relationships between labels and features are not exactly the same, which could be observed in the slight difference between Fig. \ref{Emotions2} in this paper and Figure 4 in \cite{wu2020multi}. Therefore, there exist causal relationships unidentified when using MB-MCF, whereas the causal relationships between the feature and other labels are discovered. Under these circumstances, the two algorithms achieve similar performance.

\section{Conclusion and Future Work}
\label{futurework}

The identification of common causal variables and label-specific causal variables is an interesting topic for researchers due to their imperative role in the causal mechanism. In this paper, we investigate the theoretical property of common causal variables of multiple labels and find that the common causal variables are determined by equivalent information following different mechanisms with or without existence of label causality. Based on extensive analyses, the discovery and distinguishing algorithm CLCD is proposed to identify these two types of variables without mining all of the multiple MBs. Furthermore, we apply CLCD to multi-label feature selection problem to improve the accuracy and interpretability. Experiments on synthetic and real-world data demonstrate the efficacy of the these proposed methods. To our knowledge, it is the first study focusing on the causal variable discovery in multi-label data.

The proposed concept of common causal variables, is frequently used and considered, however it has not been formally discussed before in literatures. We believe that this study can provide several advantages in causal and non-causal problems. Two examples of causal and non-causal learning problems are presented below to prompt the possible future work based on this research.
\begin{itemize}
\item \textbf{Common causes and effects mining.} Some of the real-world applications need to mine the common causes and common effects of multiple targets from the data so that the underlying knowledge and information can emerge. For example, the pandemic viral pneumonia COVID-19 has similar pathogenesis \cite{li2020molecular} (causes) and clinical features \cite{huang2020clinical} (effects) with earlier SARS and MERS in 21st century, which could uncover the underlying relationships of these pneumonia and facilitate the sharing of the experience in therapies. Since both of the common causes and effects are included by the common causal variable set, CLCD in this paper can be the preliminary process of common cause and effect discovery, and the subsequent step is to employ a structure learning method \cite{gao2017local} to orient the causal relationships between each common causal variables and targets.
\item \textbf{Causal variable-based learning task and inference task.} Previous researches on multi-label learning \cite{lift,xu2016multi} have pointed that label-specific features can facilitate the prediction of its corresponding label. However, these learning methods exploit extra steps to identify these features. With the CLCD-FS, the predictor or classifier modeled on the causal variables (causal features) can distinguish the common and label-specific causal features before training. Therefore, any causal variable-based tasks might benefit from this research. Similarly, treating these two types of causal variables differently could improve the performance of multi-target causal inference.
\end{itemize}

\ifCLASSOPTIONcaptionsoff
  \newpage
\fi



%\bibliographystyle{IEEEtran}
\bibliographystyle{IEEEtran}
\bibliography{bare_jrnl}



%\newpage



%\vfill


\end{document}